\documentclass{article}

\usepackage{microtype}
\usepackage{graphicx}
\usepackage{booktabs} %

\usepackage[dvipsnames]{xcolor}
\usepackage{hyperref}
\hypersetup{
    colorlinks=true,
    linkcolor=NavyBlue,
    filecolor=magenta,
    urlcolor=cyan,
    citecolor=MidnightBlue
}

\usepackage[nonatbib,final]{neurips_2024}

\usepackage[square,numbers,sort]{natbib}

\usepackage{multirow}
\usepackage{threeparttable}
\usepackage{diagbox}
\usepackage[utf8]{inputenc} %
\usepackage[T1]{fontenc}    %
\usepackage{url}            %
\usepackage{booktabs}       %
\usepackage{amsfonts}       %
\usepackage{nicefrac}       %
\usepackage{microtype}      %
\usepackage{xspace}
\usepackage{caption,subcaption}
\usepackage{algorithm,algorithmicx,algpseudocode}
\usepackage{comment}
\usepackage{prettyref}
\usepackage{amsfonts}
\usepackage{makecell}
\usepackage{adjustbox}
\usepackage{multicol}

\usepackage{amsmath}
\usepackage{amssymb}
\usepackage{mathtools}
\usepackage{amsthm}

\usepackage[textsize=footnotesize]{todonotes}
\usepackage{xspace}

\usepackage{etoolbox}

\usepackage{enumitem}

\usepackage{titlesec}
\titlespacing\section{0pt}{4pt}{0pt}
\titlespacing\subsection{0pt}{2pt}{0pt}

\newcommand*\boldell{\pmb{\ell}}
\newcommand\loss{\ell}
\newcommand\lossB{\ensuremath{\boldell}}

\newcommand\Loss{\ensuremath{L}}
\newcommand\LossB{\ensuremath{\mathbi{L}}}

\newcommand\Policies{\ensuremath{\Pi}}
\newcommand\policy{\ensuremath{\pi}}
\newcommand\PoliciesDist{q}
\newcommand\PoliciesDistB{\mathbi{q}}
\newcommand\pw{h}
\newcommand\pwB{\mathbi{h}}

\newcommand\QueryIndicator{\textit{U}}
\newcommand\queryProb{z}

\def\mathbi#1{\textbf{\em #1}}

\newcommand\PoliciesNum{\ensuremath{n}}
\newcommand\policyIndex{i}
\newcommand\modelIndex{j}
\newcommand\ModelsDist{w}%
\newcommand\ModelsDistB{\mathbi{w}}%

\newcommand\pd{\ensuremath{\hat{\clabel}}}
\newcommand\pdB{\ensuremath{\hat{{\mathbi{\clabel}}}}}

\newcommand\qlb{\ensuremath{{\frac{1}{\sqrt{t}}}}}

\newcommand\tconst{\ensuremath{\tau_\text{const}}}

\newcommand\Models{\ensuremath{\mathcal{F}}}
\newcommand\ModelsNum{\ensuremath{k}}
\newcommand\model{\ensuremath{f}}
\newcommand\instance{\ensuremath{\mathbi{x}}}
\newcommand\domInstance{\ensuremath{\mathcal{X}}}
\newcommand\domClabel{\ensuremath{\mathcal{Y}}}
\newcommand\numClabel{\ensuremath{c}}

\newcommand\clabel{y} %
\newcommand\Learner{\ensuremath{\mathcal{A}}\xspace}

\newcommand\adviceMatrix{\ensuremath{\mathbi{E}}}

\newcommand\expectation{\ensuremath{\mathbb{E}}}
\newcommand\budget{\ensuremath{b}}

\newcommand\entropyF{\mathfrak{E}}

\newcommand\minpgap[1]{\ensuremath{\rho_{#1}}}

\newcommand\maxind[1]{\ensuremath{\text{maxind}(#1)}}

\renewcommand{\Pr}[1]{\ensuremath{\mathbb{P}\left[#1\right] }}

\newcommand{\expct}[1]{\mathbb{E}\left[#1\right]}
\newcommand{\expctover}[2]{\mathbb{E}_{#1}\!\left[#2\right]}

\def \argmax {\mathop{\rm arg\,max}}

\newcommand{\algname}{\textsc{CAMS}\xspace}
\newcommand{\algRecommend}{\textsc{Recommend}\xspace}

\newif\iffinal
\finaltrue

\newcommand{\fixremoved}[1]{{\color{black} #1}}

\iffinal
    \newcommand{\fix}[1]{{#1}}
    \newcommand{\YC}[1]{}
    \newcommand{\YCinline}[1]{}
    \newcommand{\note}[1]{}
    \newcommand{\pref}[1]{}
    \newcommand{\rebuttal}[2]{{\color{black} #2}}
\else
    \newcommand{\fix}[1]{{\color{red} #1}}
    \newcommand{\rebuttal}[2]{{\color{red} #2}}
    \newcommand{\YC}[1]{\todo[fancyline,color=Maroon!40]{YC: #1}\xspace}
    \newcommand{\YCinline}[1]{\textcolor{Maroon}{[YC: #1]}}

    \newcommand{\note}[1]{{\color{purple}[XL: #1]}}
    
    \newcommand{\pref}[1]{{\color{blue}(\ref{#1})}}
\fi

\newcommand{\reals}{\ensuremath{\mathbb{R}}}

\newcommand{\tabref}[1]{Table~\ref{#1}}
\newcommand{\figref}[1]{Fig.~\ref{#1}}
\newcommand{\eqnref}[1]{\text{Eq.}~(\ref{#1})}

\newcommand{\appref}[1]{Appendix~\ref{#1}}
\newcommand{\thmref}[1]{Theorem~\ref{#1}}

\newcommand{\lemref}[1]{Lemma~\ref{#1}}

\newcommand{\algoref}[1]{Algorithm~\ref{#1}}

\newcommand{\paren} [1] {\ensuremath{ \left( {#1} \right) }}

\newcommand{\bracket}[1]{\left[#1\right]}
\newcommand{\tuple}[1]{\ensuremath{\left\langle #1 \right\rangle}}
\newcommand{\curlybracket}[1]{\ensuremath{\left\{#1\right\}}}

\theoremstyle{plain}
\newtheorem{theorem}{Theorem}%

\newtheorem{lemma}[theorem]{Lemma}

\theoremstyle{definition}

\theoremstyle{remark}
\newtheorem{remark}[theorem]{Remark}

\newcommand{\stochastic}{\textcolor{ForestGreen}{\textsc{stochastic}}}
\newcommand{\adversarial}{\textcolor{Maroon}{\textsc{adversarial}}}
\newcommand{\queryCost}{C}
\newcommand{\numTotalPolicies}{m}

\makeatletter
\newcommand*{\inlineequation}[2][]{%
  \begingroup
    \refstepcounter{equation}%
    \ifx\\#1\\%
    \else
      \label{#1}%
    \fi
    \relpenalty=10000 %
    \binoppenalty=10000 %
    \ensuremath{%
      #2%
    }%
    ~\@eqnnum
  \endgroup
}
\makeatother

\newcommand{\denselist}{\itemsep 0pt\topsep-10pt\partopsep-6pt}

\newtoggle{longversion}
\settoggle{longversion}{true}

\iftoggle{longversion}{
    \renewcommand{\appref}[1]{App.~\ref{#1}}
    }
    {
    \renewcommand{\appref}[1]{the supplemental materials}
    }

\title{
Contextual Active Model Selection
}

\author{Xuefeng Liu\textsuperscript{1}\thanks{Correspondence to: Xuefeng Liu <\href{mailto:xuefeng@uchicago.edu}{xuefeng@uchicago.edu}>.} ,~\textbf{Fangfang Xia\textsuperscript{2}},~\textbf{Rick L. Stevens\textsuperscript{1,2}},~\textbf{Yuxin Chen\textsuperscript{1}} \\
\textsuperscript{1}Department of Computer Science, University of Chicago\\
\textsuperscript{2}Argonne National Laboratory 
}

\begin{document}

\maketitle

\begin{abstract}

While training models and labeling data are resource-intensive, a wealth of pre-trained models and unlabeled data exists. %
To effectively utilize these resources, we present an approach to actively select pre-trained models while minimizing labeling costs. We frame this as an \textit{online contextual active model selection} problem: At each round, the learner receives an unlabeled data point as a context. The objective is to adaptively select the best model to make a prediction while limiting label requests. %
To tackle this problem, we propose \algname, a contextual active model selection algorithm that relies on two novel components: (1) a contextual model selection mechanism, which leverages context information to make informed decisions about which model is likely to perform best for a given context, and (2) an active query component, which strategically chooses when to request labels for data points, minimizing the overall labeling cost. We provide rigorous theoretical analysis for the regret and query complexity under both adversarial and stochastic settings. Furthermore, we demonstrate the effectiveness of our algorithm on a diverse collection of benchmark classification tasks. %
Notably, \algname requires substantially less labeling effort (less than 10\%) compared to existing methods on CIFAR10 and DRIFT benchmarks, while achieving similar or better accuracy. Our code is publicly available at: \url{https://github.com/xuefeng-cs/Contextual-Active-Model-Selection}.

\end{abstract}

\section{Introduction}\label{sec:intro}

{As pre-trained models %
become increasingly prevalent in a variety of real-world machine learning applications \citep{brown2020language,ouyang2022training,achiam2023gpt,liuentropy},}
there is a growing demand for label-efficient approaches for model selection, especially when facing varying data distributions and contexts at run time. Oftentimes, no single pre-trained model achieves the best performance for every context, and a proper approach is to construct a policy for adaptively selecting models for specific contexts \citep{luo2020metaselector}. %
For instance, in medical diagnosis and drug discovery, accurate predictions are of paramount importance. The diagnosis of diseases through {pathologist}
or the determination of {compound chemical properties}
through lab testing can be costly and time-consuming.
Different models may excel in analyzing different types of {pathological images~\citep{abd2018two,du2016overview,aggarwal2021diagnostic}}
or chemical compounds~\citep{hansen2015machine,clyde2023ai,liu2014multi}. %
Furthermore, 
in many real-world applications, %
the collection of labels for model evaluation can be expensive {and}
data instances may arrive as a stream rather than all at once. 
This scenario necessitates \textit{cost-effective} and \textit{robust} online algorithms capable of determining the most efficient model selection policy even when faced with a limited supply of labels, a scenario not fully addressed by previous works that typically assume access to all labels \cite{freund1997decision, oza2001online, auer2002nonstochastic, auer2002finite}.

{Recently, the problem of online model selection with the consideration of label acquisition costs was studied in a \textit{context-free} setting by \citet{karimi2021online}.}
However, this approach doesn't fully capture the dynamics of data contexts that are essential in many applications. %
Recognizing this gap, in this paper, we consider a more general problem setting that incorporates context information for adaptive model selection.
We introduce \algname, an {algorithm}
for active model selection that dynamically adapts to the data context to choose the most suitable 
models for an arbitrary data stream. As highlighted in \tabref{tab:alg:characteristics}, \algname aims to address the need for %
adaptive and effective model selection, 
by bridging the gap between contextual bandits, online learning, and active learning. 

Our key contributions are summarized as follows:
\vspace{-3mm}
\begin{itemize}[leftmargin=*] %
    \item 
We investigate a novel problem which we refer to as \textit{contextual active model selection}, and introduce a novel principled %
{algorithm}
that features two key technical components: (1) a \textit{contextual online model selection} procedure, designed to handle both stochastic and adversarial settings, and (2) an \textit{active query} strategy. 
The proposed algorithm is designed to be robust to heterogeneous data streams, accommodating both stochastic and adversarial online data streaming scenarios. %

\item We provide rigorous theoretical analysis on the \textit{regret} and \textit{query complexity} of the proposed algorithms. %
We establish regret upper bounds %
for both adversarial and stochastic data streams under limited label costs. Our regret upper bounds are within constant factors of the existing lower bounds for online learning problems with expert advice under the full information setting. %

\item Empirically, we demonstrate the effectiveness and robustness of our approach on a variety of online model selection tasks spanning different application domains (from generic ML benchmarks such as CIFAR10 to domain-specific tasks in biomedical analysis), data scales (ranging from 80 to 10K), data modalities (i.e., tabular, image, and graph-based data), and label types (binary or multiclass labels). For the tasks evaluated, (1) \algname outperforms all competing baselines by a significant margin. (2) Asymptotically, \algname performs no worse than the best single model. %
(3) \algname is not only robust to adversarial data streams but also can efficiently recover from ``malicious experts'' (i.e. inferior pre-trained models). %
\end{itemize}

\begin{table*}[t]
\centering
\scalebox{0.65}{
\begin{threeparttable}
\begin{tabular}{l | l l l l l l l}
\toprule
\multirow{ 3}{*}{\backslashbox{\textbf{Setup}}{\textbf{Algorithm}}}
& \textbf{\makecell[l]{{Online bagging}\tnote{*}}} \makecell[l]{}
& \textbf{Hedge}
& \textbf{EXP3}
& \textbf{EXP4}
& \textbf{Query by Committee}
& \textbf{ModelPicker}
& \textbf{\algname}
\\
& {\citep{oza2001online}}
& {\citep{freund1997decision}}
& {\citep{auer2002nonstochastic}} 
& {\citep{auer2002nonstochastic}}
& {\citep{seung1992query}}
& {\citep{karimi2021online}}
& {(ours)}
\\
& {\textrm{{bagging}}}
& {\textrm{{online learning}}}
& {\textrm{{bandit}}} 
& {\textrm{{contextual bandits}}}
& {\textrm{{active learning}}}
& {\textrm{{model selection}}}
& {(ours)}
\\
\hline
model %
selection\tnote{\textdagger}
& $\times$
& $\checkmark$
& $\checkmark$
& $\checkmark$
& $\times$
& $\checkmark$
& $\checkmark$ \\
\hline
full-information
& $\checkmark$
& $\checkmark$
& $\times$ 
& $\times$ 
& $\checkmark$ 
& $\checkmark$
& $\checkmark$ \\
\hline
active queries
& $\times$  
& $\times$  
& $\times$  
& $\times$ 
& $\checkmark$ 
& $\checkmark$
& $\checkmark$ \\
\hline
context-aware 
& $\times$ 
& $\times$  
& $\times$  
& $\checkmark$ 
& $\times$  
& $\times$  
& $\checkmark$ \\
\bottomrule
\end{tabular}
\begin{tablenotes}\footnotesize
\item[\textdagger] We regard ``arms'' as ``models'' when comparing \algname against bandit algorithms, such as EXP3/EXP4. 
\item[*] Online ensemble learning aims to build a composite model by aggregating multiple models %
rather than selecting the best model (for a given context).
\end{tablenotes}
\end{threeparttable}
}
\caption{%
Comparing \algname against related work in terms of problem setup. %
}\label{tab:alg:characteristics}
\end{table*}

\section{Related Work}\label{sec:related}

\paragraph{Contextual bandits.} Classical bandit algorithms %
(e.g., \citep{auer2002nonstochastic,auer2002finite}) %
aim to find the best arm(s) %
through a sequence of actions. When side information (e.g., user profile for recommender systems or environmental context for experimental design) is available, %
many bandit algorithms can be lifted to the contextual setting: For example, EXP4 \citep{auer2002nonstochastic,beygelzimer2011contextual,neu2015explore} considers the bandit setting with expert advice: At each round, experts announce their predictions of which actions are the most promising for the given context, and the goal is to construct a expect selection policy that competes with the best expert from hindsight. In bandit problems, the learner only gets to observe the reward for each action taken. In contrast, for the online model selection problem considered in this work---where an action corresponds to choosing a model to make prediction on an incoming data point---we get to see the loss/reward of \textit{all} models on the labeled data point. {By utilizing the information from unchosen arms, it could significantly reduce the cumulative regret. %
In this regard, this work aligns more closely with online learning with \textit{full information} setting, where the learner has access to the loss of all the arms at each round (e.g. as considered in the Hedge algorithm \citep{freund1997decision,burtini2015survey,cesa2006prediction,hoi2021online}).

\paragraph{Online learning with full information.} A clear distinction between our work and online learning is that %
we assume the labels of the online data stream are not readily available but can be acquired at each round with a cost. In addition, the learner only observes the loss incurred by all models on a data point when it decides to query its label. In contrast, in the canonical online learning setting, labels arrive with the data and one gets to observe the loss of all candidate models at each round. Similar setting also applies to other online learning problems, such as online boosting or bagging. %
A related work to ours is online learning with label-efficient prediction \citep{cesa2005minimizing}, which proposes an online learning algorithm with matching upper and lower bounds on the regret. However, they consider a fixed query probability that leads to a linear query complexity. Our algorithm, inspired by uncertainty sampling in active learning, achieves an improved query complexity with the adaptive query strategy while maintaining a comparable regret.

\paragraph{Stream-based Active learning.}%
Active learning aims to achieve a target learning performance with fewer training examples \citep{settles2009active}. 
The active learning framework closest to our setting is query-by-committee (QBC) \citep{seung1992query}, in particular under the stream-based setting %
\citep{loy2012stream,ho2008query}. QBC maintains a committee of hypotheses; each committee member votes on the label of an instance, and the instances with the maximal disagreement among the committee are considered the most informative labels. Note that existing stream-based QBC algorithms are designed and analyzed assuming i.i.d. data streams. In comparison, our work uses a different query strategy as well as a novel model recommendation strategy, which also applies to the adversarial setting.

\paragraph{Active model selection.} Active model selection captures a broad class of problems where model evaluations are expensive, either  due to (1) the cost of evaluating (or ``probing'') a model, or (2) the cost of annotating a training example. 
Existing works under the former setting  \citep{madani2012active,santana2020contextual} {and online setting \citep{cutkosky2021dynamic,shukla2019adaptive}} often ignore context information and data annotation cost, and only consider \emph{partial} feedback on the models being evaluated/ probed on i.i.d. data. The goal is to identify the best model with as few model probes as possible. %
This is quite different from our problem setting which considers the full information setting as well as non-negligible data annotation cost. 
{
\citep{van2018almost} proposes that the optimal model choice is influenced by the sample size rather than any individual sample feature.
\citep{li2020finding} addresses the active model selection problem, however both works do not adopt a stream-based approach.}
For the later, apart from \citet{karimi2021online}, {an online contextual-free model selection work}, as shown in \tabref{tab:alg:characteristics}, most existing works assume a pool-based setting where the learner can choose among the pool of unlabeled data \citep{sugiyama2008batch,madani2012active,sawade2012,sawade2010active,ali2014active,gardner2015bayesian,zhang2014beyond,leite2010active}, and the goal is to identify the best model with a minimal set of labels.

\section{Problem Statement}\label{sec:problem_statement}

\paragraph{Notations.} Let $\domInstance$ be the input domain %
and $\domClabel := \curlybracket{0, \dots, \numClabel-1}$ be the set of $\numClabel$  possible class labels for each input instance.  %
Let $\Models = \curlybracket{\model_1, \dots, \model_k}$ be a set of $\ModelsNum$ pre-trained classifiers over $\domInstance\times \domClabel$. %
A model selection policy $\policy: \domInstance \rightarrow \Delta^{\ModelsNum-1}$ maps any input instance $\instance \in \domInstance$ to a distribution over the pre-trained classifiers $\Models$, specifying the probability $\policy\paren{\instance}$ of selecting each classifier under input $\instance$. Here, $\Delta^{\ModelsNum-1}$ denotes the $\ModelsNum$-dimensional probability simplex $\curlybracket{\ModelsDistB \in \reals^\ModelsNum: |\ModelsDistB|=1, \ModelsDistB \geq 0}$. %
One can interpret a policy $\policy$ as an ``expert'' that suggests which model to select for a given \textit{context} $\instance$.

Let $\Policies$ be a collection of model selection policies. In this paper, we propose an \textit{extended policy set}
$    \Policies^* := \Policies \cup \curlybracket{\policy^{\text{const}}_1, \dots, \policy^{\text{const}}_k} \label{eq:extpolset}$
which also includes constant policies that always suggest a fixed model. %
Here, $\policy^{\text{const}}_j (\cdot) := \mathbi{e}_j$, and $\mathbi{e}_j \in \Delta^{\ModelsNum-1}$ denotes the canonical basis vector with $e_j=1$. 
Unless otherwise specified, we assume $\Policies$ is finite %
with $|\Policies| = \PoliciesNum$, and $|\Policies^*| \leq \PoliciesNum+\ModelsNum$. As a special case, when $\Policies = \emptyset$, our problem reduces to the contextual-free setting.

\looseness -1 \noindent \textbf{The contextual active model selection protocol.} Assume that the learner knows the set of classifiers $\Models$ as well as the set of model selection policies $\Policies$. At round $t$, the learner receives a data instance $\instance_t \in \domInstance$ as the context for the current round, and computes the predicted label $\pd_{t,j} = \model_j\paren{\instance_t}$ for each pre-trained classifier indexed by $j \in [\ModelsNum]$. Denote the vector of predicted labels by all $\ModelsNum$ models by $\pdB_t := [\pd_{t,1}, \dots, \pd_{t,\ModelsNum}]^\top$. %
Based on previous observations, the learner identifies a model/classifier $f_{j_t}$ and makes a prediction $\pd_{t,j_t}$ for the instance $\instance_t$.
Meanwhile, the learner can obtain the true label $\clabel_t$ \textit{only if} it decides to query $\instance_t$. Upon observing $\clabel_t$, the learner incurs a \textit{query cost}, and receives a (full) loss vector %
$\lossB_{t}=\mathbb{I}_{\curlybracket{\pdB_t\neq{\clabel_t}}}$, 
where the $j$th entry %
$\loss_{t,j} := \mathbb{I}_{\curlybracket{\pd_{t,j}\neq{\clabel_t}}}$
corresponds to the 0-1 loss for model $j\in [k]$ at round $t$. The learner can then use the queried labels to adjust its model selection criterion for future rounds. 

\paragraph{Performance metric.}
If $\instance_t$ is misclassified by the model $j_t$ selected by learner at round $t$, i.e. $\pd_{t,j_t} \neq \clabel_t$, it will be counted towards the \textit{cumulative loss} of the learner, regardless of the learner making a query. Otherwise, no loss will be incurred for that round. For a learning algorithm \Learner, its cumulative loss over $T$ rounds is defined as ${L}^\Learner_T := \sum_{t=1}^T{\loss}_{t,j_t}$. %

In practice, the choice of model $j_t$ at round $t$ by the learner \Learner could be random: For {\emph{stochastic}} data streams where $\paren{\instance, \clabel}$ {arrives i.i.d.}, the learner may choose different models for different random realizations of $\paren{\instance_t, \clabel_t}$. For the {\emph{adversarial}} setting where the data stream $\curlybracket{\paren{\instance_t,\clabel_t}}_{t\geq{1}}$ is {chosen by an adversary before each round}, the learner may randomize its choice of model to avoid a constant loss at each round \citep{hazan2019introduction}. Therefore, due to the randomness of ${L}^\Learner_T$, we consider the \textit{expected} cumulative loss $\expectation[{L}^\Learner_T]$ as a key performance measure of the learner $\Learner$. To characterize the progress of $\Learner$, we consider the \textit{regret}---formally defined as follows--- as the difference between the cumulative loss received by the learner and the loss if the learner selects the ``best policy'' $\policy^*\in \Policies^*$ in hindsight. 

For stochastic data streams, we assume that each policy $i$ recommends the \emph{most probable} model
w.r.t. $\policy_i(\instance_t)$ for context $\instance_t$. We use $\maxind{\ModelsDistB} := \argmax_{j: \ModelsDist_{j}\in \ModelsDistB } \ModelsDist_{j}$ to denote the index of the maximal-value entry\footnote{Assume ties are broken randomly.} of $\ModelsDistB$. Since $\paren{\instance, \clabel}$ are drawn i.i.d., we define %
$\mu_i = \frac1T\sum_{t=1}^T\expctover{\instance_t,\clabel_t}{{\loss}_{t,\maxind{\policy_i(\instance_t)}}}$.
This leads to the pseudo-regret for the stochastic setting over $T$ rounds, defined as
\begin{equation}\label{eq:expected_stochastic_regret}
\overline{\mathcal{R}}_T\paren{\Learner}=\expectation[{\Loss}^{\Learner}_T]-T\min_{i\in [|\Policies^*|]}\mu_{i}. 
\end{equation}
In an adversarial setting, since the data stream (and hence the loss vector) is determined by an adversary, we consider the reference best policy to be the one that minimizes the loss on the adversarial data stream, and the expected regret %
\begin{equation}\label{eq:expected_adversarial_regret}
   \mathcal{R}_T\paren{\Learner} = \expectation{[{\Loss}^{\Learner}_T]} - \min_{i\in [|\Policies^*|]} \sum_{t=1}^T \tilde{\loss}_{t,i},
\end{equation}
where $\tilde{\loss}_{t,i} := \langle \policy_i\paren{\instance_t}, {\lossB}_{t} \rangle$ denotes the expected loss if the learner commits to policy $\policy_i$, randomizes and selects $j_t \sim \policy_i\paren{\instance_t}$ (and receives loss ${\loss}_{t,j_t}$) at round $t$. %
Our goal is to devise a principled online active model selection strategy %
to minimize the regret as defined in \eqref{eq:expected_stochastic_regret} or \eqref{eq:expected_adversarial_regret}, while maintaining a low total query cost.
For convenience, we refer the readers to \appref{app:notations} for a summary of the notations used in this paper.

\section{Contextual Active Model Selection}\label{sec:cams}
In this section, we introduce our main algorithm for both stochastic and adversarial data streams. %
\begin{figure*}[h]
  \centering
  \scalebox{0.76}{
  \fbox{
  \begin{subfigure}[t]{0.6\textwidth}
\begin{algorithmic}[1]
  \State {\bfseries Input:} Models $\Models$, policies $\Policies$, \#rounds $T$, budget $\budget$
  \State Initialize loss $\tilde{\LossB}_{0}\leftarrow 0$; query cost $\queryCost_0 \leftarrow 0$ %
  \State Set $\Policies^* \leftarrow \Policies \cup \curlybracket{\policy^{\text{const}}_1, \dots, \policy^{\text{const}}_k}$ according to \eqnref{eq:extpolset} \label{alg:cams:line:extpolset} %
  \For{$t=1,2,...,T$}
  \State Receive $\instance_t$ \label{alg:cams:line:context}
  \State $\eta_t \leftarrow \textsc{SetRate}(t,\instance_t, \left|\Policies^*\right|)$
    \State Set $\PoliciesDist_{t,i} \propto \exp{\paren{-\eta_t\tilde{\Loss}_{t-1,i}}}~\forall \policyIndex \in |\Policies^*|$ 
  \State $j_t \leftarrow \algRecommend(\instance_t,\PoliciesDistB_t) $ %
  \State Output $\pd_{t,j_{t}} \sim f_{t,j_{t}}$ as the prediction for $\instance_t$ \label{alg:cams:line:recommend}
  \State Compute $\queryProb_t$ in \eqnref{eq:query-probability} \label{alg:cams:line:queryprob}
  \State Sample $\mathbb{\QueryIndicator}_t \sim \textrm{Ber}\paren{\queryProb_t}$ %
  \If{$\mathbb{\QueryIndicator}_t = 1$ and $\queryCost_t\leq {\budget}$} 
  \State Query the label $\clabel_t$
  \State $\queryCost_t \leftarrow \queryCost_{t-1} + 1$ \label{alg:cams:line:observe}
  \State Compute ${\lossB}_t$: $\loss_{t,j}=\mathbb{I}\curlybracket{\pd_{t,j}\neq \clabel_t}, \forall {j\in{[\left|\Models\right|]}}$ \label{alg:cams:line:update}
    \State Estimate model loss:  $\hat{\loss}_{t,j}=\frac{\loss_{t,j}}{\queryProb_t}, \forall j \in{[\left|\Models\right|]}$
    \State Update $\tilde{\lossB}_{t}$: %
    $\tilde{\loss}_{t,i} \leftarrow \langle \policy_i(\instance_t),  \hat{\loss}_{t,j} \rangle, \forall i\in \left[|\Policies^*|\right]$
    \State $\tilde{\LossB}_{t} = \tilde{\LossB}_{t-1} + \tilde{\lossB}_{t}$
    \Else
    \State $\tilde{\LossB}_{t} = \tilde{\LossB}_{t-1}$
    \State $\queryCost_t \leftarrow \queryCost_{t-1}$ \label{alg:cams:line:endofprocedure}
    \EndIf
  \EndFor
\end{algorithmic}
    \end{subfigure}
  \begin{subfigure}[t]{0.50\textwidth}
  \begin{subfigure}[t]{1\textwidth}
\hfill
\begin{algorithmic}[1]
\setcounter{ALG@line}{20}
    \Procedure{SetRate}{$t,\instance_t,\numTotalPolicies$}
        \If{\stochastic}
        \State     $\eta_t=\sqrt{\frac{\ln{{\numTotalPolicies}}}{t}}$ \label{alg:cames:line:setrate:stochastic}
        \EndIf
        \If{\adversarial}
        \State %
        Set $\minpgap{t}$ as in adversarial setting section
        \label{alg:cames:line:setrate:mingap}
        \State  $\eta_t=\sqrt{\qlb + \frac{ \minpgap{t}}{c^2\ln c }}\cdot \sqrt{\frac{\ln{{\numTotalPolicies}}}{T}}$ \label{alg:cames:line:setrate:rate}
        \EndIf
        \State \Return $\eta_t$
    \EndProcedure
\end{algorithmic}
    \end{subfigure}
    \\~
    \\~
 \begin{subfigure}[t]{1\textwidth}
\hfill
\begin{algorithmic}[1]
\setcounter{ALG@line}{28}
    \Procedure{\algRecommend}{$\instance_t,\PoliciesDistB_t$}
        \If{\stochastic}
        \State $\ModelsDistB_t = \sum_{i \in |\Policies^*|} \PoliciesDist_{t,i} \policy_i(\instance_t)$ \label{alg:cams:line:stochastic:weightmodel}
        \State $j_t \leftarrow \maxind{\ModelsDistB_t}$ \label{alg:cams:line:stochastic:selectmodel}
        \EndIf
        \If{\adversarial}
        \State  $i_t \sim \PoliciesDistB_t$ \label{alg:cams:line:adversarial:selectpolicy}
        \State  $j_t \sim \policy_{i_t}\paren{\instance_t}$ \label{alg:cams:line:adversarial:selectmodel}
        \EndIf
        \State \Return $j_t$
    \EndProcedure
    \end{algorithmic}
        \end{subfigure} 
    \\~
    \end{subfigure}
  }}
  \vspace{-1mm}
  \caption{The \underline{C}ontextual \underline{A}ctive \underline{M}odel \underline{S}election (\algname) algorithm}
    \label{alg:CAMS}
    \vspace{-3mm}
\end{figure*}

\noindent \textbf{Contextual model selection.}
Our key insight underlying the contextual model selection strategy extends from the \textit{online learning with expert advice} framework \citep{freund1997decision,burtini2015survey}. 
We start by appending the constant policies that always pick single pre-trained \emph{models} to form the extended policy set $\Policies^*$ (Line~\ref{alg:cams:line:extpolset}, in \figref{alg:CAMS}). This allows \algname to be at least as competitive as the best model. Then, at each round, \algname maintains a probability distribution over the (extended) policy set $\Policies^*$, and updates those according to the observed loss for each policy. %
We use $\PoliciesDistB_t := (\PoliciesDist_{t,i})_{i\in |\Policies^*|}$ to denote the probability distribution over $\Policies^*$ at $t$.  Specifically, the probability $\PoliciesDist_{t,i}$ is computed based on the exponentially weighted cumulative loss, i.e. $\PoliciesDist_{t,i} \propto \exp{\paren{-\eta_t\tilde{\Loss}_{t-1,i}}}$ where $\tilde{\Loss}_{t,i} := \sum_{\tau=1}^{t} \tilde{\loss}_{\tau,i}$ denotes the cumulative loss of policy $i$. 

For adversarial data streams, it is natural for both the online learner and the model selection policies to randomize their actions to avoid linear regret \citep{hazan2019introduction}. Following this insight, %
\algname randomly samples a policy $i_t \sim \PoliciesDistB_t$, and---based on the current context $\instance_t$---samples a classifier $j_t \sim \policy_{i_t}\paren{\instance_t}$ to recommend at round $t$.

Under the stochastic setting, \algname adopts a 
weighted majority strategy \citep{littlestone1994weighted} when selecting models. The vector of each model's weighted votes from the policies, $\ModelsDistB_t = \sum_{i \in |\Policies^*|} \PoliciesDist_{t,i} \policy_i(\instance_t)$, is interpreted as a distribution induced by the weighted policy. 
 The model $j_t = \maxind{\ModelsDistB_t}$ which receives the highest probability becomes the recommended model at round $t$. This deterministic model selection strategy is commonly used in stochastic online optimization \citep{hazan2019introduction}. 
 An alternative strategy is to take a randomized approach as in the adversarial setting, or take a Follow-the-Leader approach \citep{lattimore2020bandit} and go with the most probable model recommended by the most probable policy (i.e. use $\ModelsDistB_t = \policy_{\maxind{\PoliciesDistB_t}}(\instance_t)$).As shown in experimental results section
 and further discussed in Appendix (outperformance over the best policy/expert section),
 \algname outperforms these policies in a wide range of practical applications. The model selection steps are detailed in Line~\ref{alg:cams:line:context}-\ref{alg:cams:line:recommend} in \figref{alg:CAMS}.

\noindent \textbf{Active queries.} Under a limited budget, %
we intend to query the labels of those instances that exhibit significant disagreement among the pre-trained models $\Models$. To achieve this goal, we design an adaptive query strategy with query probability $\queryProb_t$. Concretely, given context $\instance_t$, model predictions $\pdB_t$ and model distribution $\ModelsDistB_t$, we denote by $\bar\loss_t^{\clabel} := \langle \ModelsDistB_t, \mathbb{I}\curlybracket{\pdB_t \neq \clabel} \rangle$ as the expected loss if the true label is $\clabel$. We characterize the model disagreement as
\begin{equation}\label{eq:instance-info}
\entropyF\paren{\pdB_t,\ModelsDistB_t} :=\frac{1}{\numClabel} \sum_{\clabel \in \domClabel, \bar\loss_t^{\clabel} \in (0,1)} \bar\loss_t^{\clabel} \log_{\numClabel}{\frac{1}{\bar\loss_t^{\clabel}}}.
\end{equation}
Intuitively, when $\bar\loss_t^{\clabel}$ is close to $0$ or $1$, there is little disagreement among the models in labeling $\instance_t$ as $\clabel$, otherwise %
there is significant disagreement. %
We capture this insight with function $h(x)=-x\log x$. %
Since the label $\clabel_t$ is unknown upfront when receiving $\instance_t$, we iterate through all the possible labels $\clabel \in \domClabel$ and take the average value as in \eqnref{eq:instance-info}. Note that $\entropyF$ takes a similar algebraic form to the entropy function, although it does not inherit the information-theoretic interpretation.

With the model disagreement term defined above, we consider an adaptive query probability\footnote{For convenience of discussion, we assume that those rounds where all policies in $\Policies^*$ select the same models or all models $\Models$ make the same predictions are removed as a precondition.}
\begin{equation}\label{eq:query-probability}
    \queryProb_t= \max{\curlybracket{\delta_0^t,\entropyF\paren{\pdB_t,\ModelsDistB_t}}},
\end{equation}
where $\delta_0^t=\qlb \in (0,1]$ is an adaptive lower bound %
on the query probability to encourage exploration at an early stage. The query strategy is summarized in Line~\ref{alg:cams:line:queryprob}-\ref{alg:cams:line:observe} in \figref{alg:CAMS}.

\paragraph{Model updates.} Now define $\QueryIndicator_t \sim \text{Ber}\paren{\queryProb_t}$ as a binary query indicator that is sampled from a Bernoulli distribution parametrized by $\queryProb_t$. 
Upon querying the label $\clabel_t$, one can calculate the loss for each model $\model_j \in{\Models}$ as $\loss_{t,j}=\mathbb{I}\curlybracket{\pd_{t,j}\neq \clabel_t}$. Since \algname does not query all the i.i.d. examples, we introduce an unbiased loss estimator for the models, defined as 
$\hat{\loss}_{t,j}=\frac{\loss_{t,j}}{\queryProb_t}\QueryIndicator_t$. %
The unbiased loss of policy $\policy_i \in \Policies^*$ can then be computed as $\tilde{\loss}_{t,i} = \langle \policy_i(\instance_t),  \hat{\loss}_{t,j} \rangle$. 
In the end, \algname computes the (unbiased) cumulative loss of policy $\policy_i$ as $\tilde{\Loss}_{T,i}=\sum_{t=1}^T\tilde{\loss}_{t,i}$, which is used to update the policy probability distribution in next round. Pseudocode for the model update steps is summarized in Line~\ref{alg:cams:line:update}-\ref{alg:cams:line:endofprocedure} in  \figref{alg:CAMS}.%

\looseness -1
\noindent \textit{Remark. }
{\algname runs efficiently with time complexity ${{O}\paren{nk}}%
$ per round %
 and space complexity ${{O}\paren{\paren{\PoliciesNum+\ModelsNum}\cdot \ModelsNum}}$. %
 At each round, each model selection policy specifies a probability distribution over the models for the given context. When these distributions correspond to constant Dirac delta distributions (regardless of the context), the problem reduces to the context-free problem investigated by \citet{karimi2021online}. }

\section{Theoretical Analysis}\label{sec:analysis}
We now present theoretical bounds on the regret (defined in \eqnref{eq:expected_stochastic_regret} and \eqnref{eq:expected_adversarial_regret}, respectively) and the query complexity of CAMS for both the stochastic and the adversarial settings.  

\subsection{Stochastic setting} \label{sec:analysis:stochastic}
Under the stochastic setting, 
the cumulative loss of \algname over T rounds---as specified by the \algRecommend procedure---is ${\Loss}^{\algname}_T=\sum_{t=1}^T
\hat{\loss}_{t,\maxind{\ModelsDistB_t}} %
$ where recall $\ModelsDistB_t = \sum_{i \in |\Policies^*|} \PoliciesDist_{t,i} \policy_i(\instance_t)$ is the probability distribution over $\Models$ induced by the weighted policy. 

Let $i^* = \arg\min_{i\in[|\Policies^*|]}\mu_i$ be the index of the best policy ($\mu_i$ denotes the expected loss of policy $i$, as defined in problem statement section.
The cumulative expected loss of policy $i^*$ is $T \mu_{i^*}$; therefore the expected pseudo-regret (\eqnref{eq:expected_stochastic_regret}) is $\overline{\mathcal{R}}_T\paren{\algname}=\expct{\sum_{t=1}^T \hat{\loss}_{t,\maxind{\ModelsDistB_t}}}-T \mu_{i^*}.$

Define ${\Delta} :=\min_{i\neq i^*} (\mu_i - \mu_{i^*})$ as the minimal sub-optimality gap\footnote{w.l.o.g. assume there is a single best policy, and thus ${\Delta}>0$.} in terms of the expected loss against the best policy $i^*$. 
Furthermore, let $\ModelsDistB_{i^*}^t := \policy_{i^*}\paren{\instance_t}$ be probability distribution over $\Models$ induced by policy $i^*$ at round $t$. 
We define %
\fix{
$\inlineequation[eq:gamma_sto]{\gamma := \min_{\instance_t}\{\max_{\ModelsDist_j \in \ModelsDistB_{i^*}^t} \ModelsDist_j  -\max_{\ModelsDist_j \in \ModelsDistB_{i^*}^t,j\neq\maxind{\ModelsDistB_{i^*}^t}} \ModelsDist_j \}}$}
as the minimal probability gap between the most probable model and the rest (assuming no ties) induced by the best policy $i^*$. We further define \fix{$b=p_{\min}\log_{c}{(1/p_{\min})}$, where $p_{\min} = \min_{s,i} \pi(\mathbf{x}_s)$ denotes the minimal model selection probability by any policy\footnote{We assume $p_{\min} > 0$ per the policy regularization criterion in Appendix C.3. (cf. Algorithm 1 on ``Regularized policy $\bar{\policy}(\mathbf{x}_t)$)''.}.} As our first main theoretical result, 
we show that, without exhaustively querying the labels of the stochastic stream, \algname achieves constant expected regret. %

\begin{theorem}{(Regret)}\label{thm:stochasticReg}
In the stochastic environment, %
{with probability at least $1-\delta$}, \algname achieves constant expected pseudo regret 
$\overline{\mathcal{R}}_T\paren{\algname} {=} \paren{\frac{\ln{\frac{|\Policies^*|\fix{-1}}{\gamma}}+{\sqrt{\ln{|\Policies^*|}\cdot{2\fix{b^2}\ln{\frac{2}{\delta}}}}}}{\sqrt{\ln{|\Policies^*|}}\Delta}}^2$.
\end{theorem}
{Note that in the stochastic setting, a lower bound of $\Omega\paren{{\left(\log{\Policies^*}\right)/\Delta}}$ was shown in \citet{mourtada2019optimality} for online learning problems with expert advice under the full information setting (i.e. assuming labels are given for all data points in the stochastic stream).} To establish the proof of \thmref{thm:stochasticReg}, we consider a novel procedure to connect the weighted policy by \algname to the best policy $\policy_{i^*}$. Conceptually, we would like to show that, after a \textit{constant} %
number of rounds $\tconst$, with high probability, the model selected by \algname (Line~\ref{alg:cams:line:stochastic:selectmodel}) will be the same as the one selected by the best policy $i^*$. In that way, the expected pseudo regret will be dominated by the maximal cumulative loss up to $\tconst$. Toward this goal, we first bound the weight of the best policy $\ModelsDist_{t,i^*}$ as a function of $t$, by choosing a proper learning rate $\eta_t$ (\algname, Line~\ref{alg:cames:line:setrate:stochastic}). Then, we identify a constant threshold $\tconst$, beyond which \algname exhibits the same behavior as $\policy_{i^*}$ with high probability. Finally, we obtain the regret bound by inspecting the regret at the two stages separately. 
The formal statement of \thmref{thm:stochasticReg} and the detailed proof are deferred to \appref{app:stochastic_results_regret}.

Next, we provide an upper bound on the query complexity in the stochastic setting. 
\begin{theorem}{(Query Complexity)}%
\label{thm:stochastic-query-complexity}. %
For $\numClabel$-class classification problems, with probability at least $1-\delta$, the expected number of queries made by \algname over $T$ rounds is upper bounded by %
$%
{\paren{{\paren{\frac{\ln{\frac{|\Policies^*|\fix{-1}}{\gamma}}+\sqrt{\ln{|\Policies^*|}\cdot{2\fix{b^2}\ln{\frac{2}{\delta}}}}}{\sqrt{\ln{|\Policies^*|}}\Delta}}^2}+T \mu_{i^*}}\frac{\ln{T}}{\numClabel\ln{\numClabel}}}%
$. %

\end{theorem}

\thmref{thm:stochastic-query-complexity} is built upon \thmref{thm:stochasticReg}, where the 
the key idea behind the proof is to relate the number of updates to the regret. { When $T\mu_{i^*},\tilde{\Loss}_{T,*}$ are regarded as constants (given by an oracle), %
the query-complexity bound is then sub-linear \textit{w.r.t.} $T$. Note that the number of class labels $\numClabel$ affects the quality of the query complexity bound. The intuition behind this result is, with larger number of classes, \emph{each query may carry more information upon observation}. For instance, in an extreme case where only one expert always recommends the best model and others gives random recommendations of models (and predicts random labels), having more classes lowers the chance of a model making the correct guess, and therefore helps to "filter out" those suboptimal experts in fewer rounds---hence being more query efficient.
We defer the proof of \thmref{thm:stochastic-query-complexity} to \appref{app:stochastic_results_query_complexity}.

\subsection{Adversarial setting}\label{sec:analysis:adversarial}
 Now we consider the adversarial setting.  %
Let $\tilde{\Loss}_{T,*} := \min_{i \in [|\Policies^*|]} \sum_{t=1}^{T} \tilde{\loss}_{t,i}$ be the cumulative loss of the best policy. The expected regret (\eqnref{eq:expected_adversarial_regret}) for \algname equals to $\mathcal{R}_T\paren{\algname} = \expectation{\big[
   \sum_{t=1}^T \langle \PoliciesDistB_t,\tilde{\lossB}_t \rangle  \big]} - \tilde{\Loss}_{T,*}$.
We show that under the adversarial setting, \algname achieves sub-linear regret in $T$ without accessing all labels.
\begin{theorem}{(Regret)}\label{thm:adversarial-regret-bound}
Let $c$ be the number of classes and $\minpgap{t}$ %
be specified as Line~\ref{alg:cames:line:setrate:mingap}-\ref{alg:cames:line:setrate:rate} in the \textsc{SetRate} procedure.
Under the adversarial setting, the expected regret of \algname is bounded by 
$2\numClabel\sqrt{\ln \numClabel / {\max}\{\minpgap{T}, {\sqrt{1/T}}\}} \cdot \sqrt{{T\log{{|\Policies^*|}}}}$.
\end{theorem}

The proof is provided in {\appref{app:adversarial_regret_bound}}.
{Assuming $\minpgap{t}$ to be a constant, our regret upper bound in \thmref{thm:adversarial-regret-bound} matches (up to constants) the lower bound of $\Omega\paren{\sqrt{T\ln{|\Policies^*|}}}$ for online learning problems with expert advice under the full information setting \citep{cesa1997use,seldin2016lower} (i.e. assuming labels are given for all data points).}
Hereby, the decaying learning rate $\eta_t$ as specified in Line~\ref{alg:cames:line:setrate:rate} is based on two parameters, where $1/\sqrt{t}$ corresponds to the lower bound $\delta_0^t$ on the query probability, 
and %
$\inlineequation[eq:rho_sto]{\minpgap{t} \triangleq 1 - \max_{\tau\in [t{-1}]} \langle \ModelsDistB_{{\tau}}, \mathbb{I}\curlybracket{\pdB_{{\tau}} = {\clabel_{\tau}}} \rangle}$ 
is a (data-dependent) term that is chosen to reduce the impact of the randomized query strategy on the regret bound (especially when $t$ is large). %
Intuitively, $\minpgap{t}$ relates to the skewness of the policy %
where the $\max$ term corresponds to the maximal probability of most probable mispredicted label over $t$ rounds.
Note that in theory $\minpgap{t}$ can be small (e.g. \algname may choose a constant policy $\policy^{\text{const}}_i \in \Policies^*$ that mispredict the label for $\instance_t$, which leads to $\minpgap{t}=0$); {in such cases, our result still translates to a sublinear regret bound of $O(c\sqrt{\log c} \cdot T^{\frac{3}{4}}\sqrt{\log{{|\Policies^*|}}}$). Furthermore}, in practice, we consider to ``regularize'' the policies (\appref{app:regularized_policy}) to ensure that probability a policy selecting any model is bounded away from 0. %

Finally, the following theorem (proof in \appref{app:adversarial_query_complexity}) 
establishes a query complexity bound of \algname.
\begin{theorem}{(Query Complexity)}\label{thm:adversarial-query-complexity}. %
Under the adversarial setting, the expected query complexity %
over $T$ rounds is
${O}\paren{
{\ln{{T}}}\paren{\sqrt{\frac{T\log{{|\Policies^*|}}}{{\max}\{\minpgap{T},{\sqrt{1/T}}\}}} + \tilde{\Loss}_{T,*}}}.$
\end{theorem}

\section{Experiments}\label{sec:experiment}

\begin{figure*}[t!]

\begin{subfigure}{1\textwidth}
    \centering
    \includegraphics[height=0.3cm, 
    clip={0,0,0,0}]{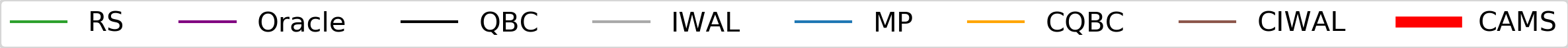}
\end{subfigure}
\rotatebox[origin=l]{90}{
\scriptsize \qquad\qquad Cumulative loss
}
\begin{subfigure}{.24\textwidth}
    \centering
    \includegraphics[%
    width=3.4cm,  clip={0,0,0,0}]{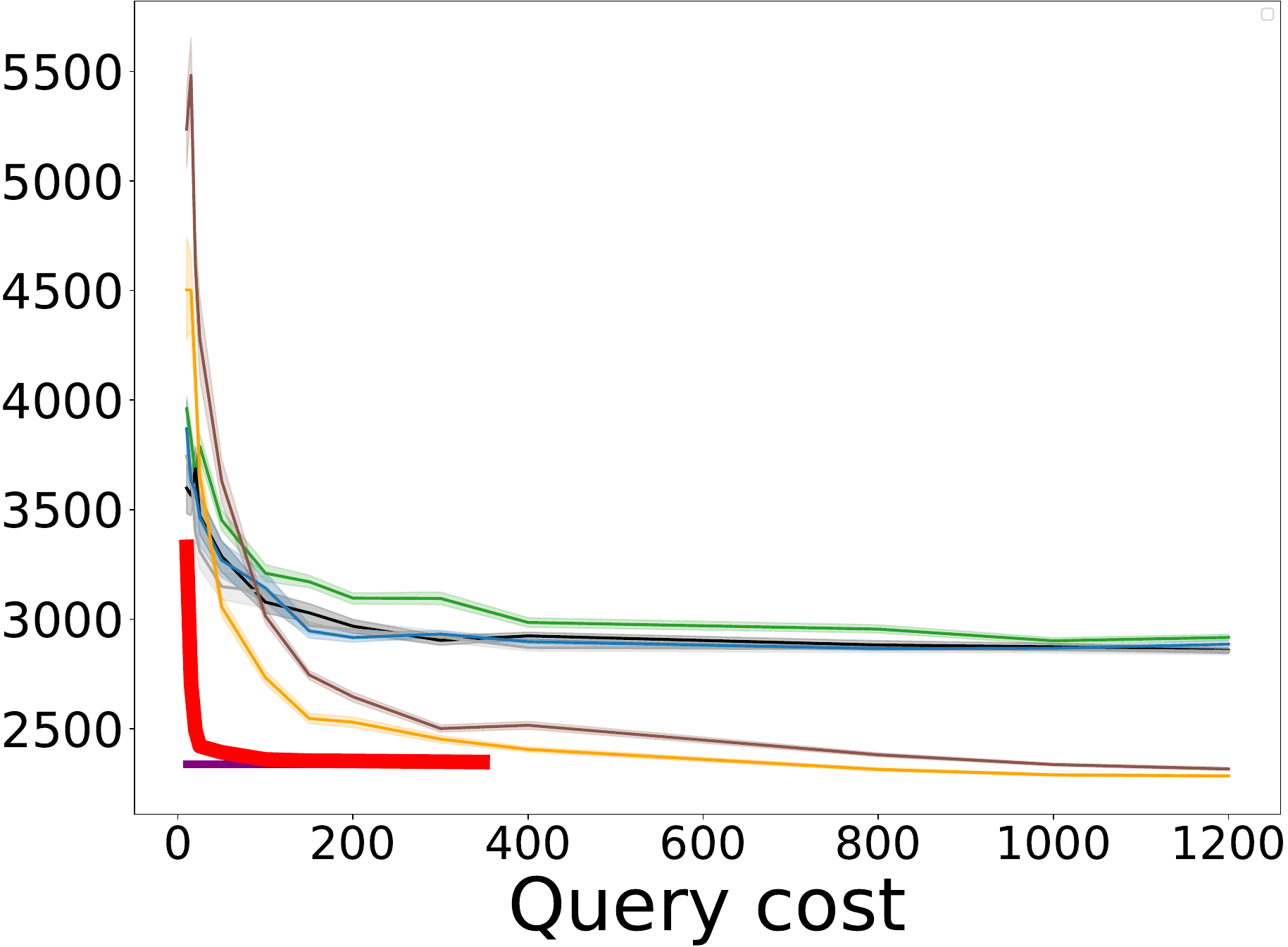}
    \caption{CIFAR10}\label{fig:results:cifar10}
\end{subfigure}
\begin{subfigure}{.24\textwidth}
    \centering
    \includegraphics[%
    width=3.4cm,  clip={0,0,0,0}]{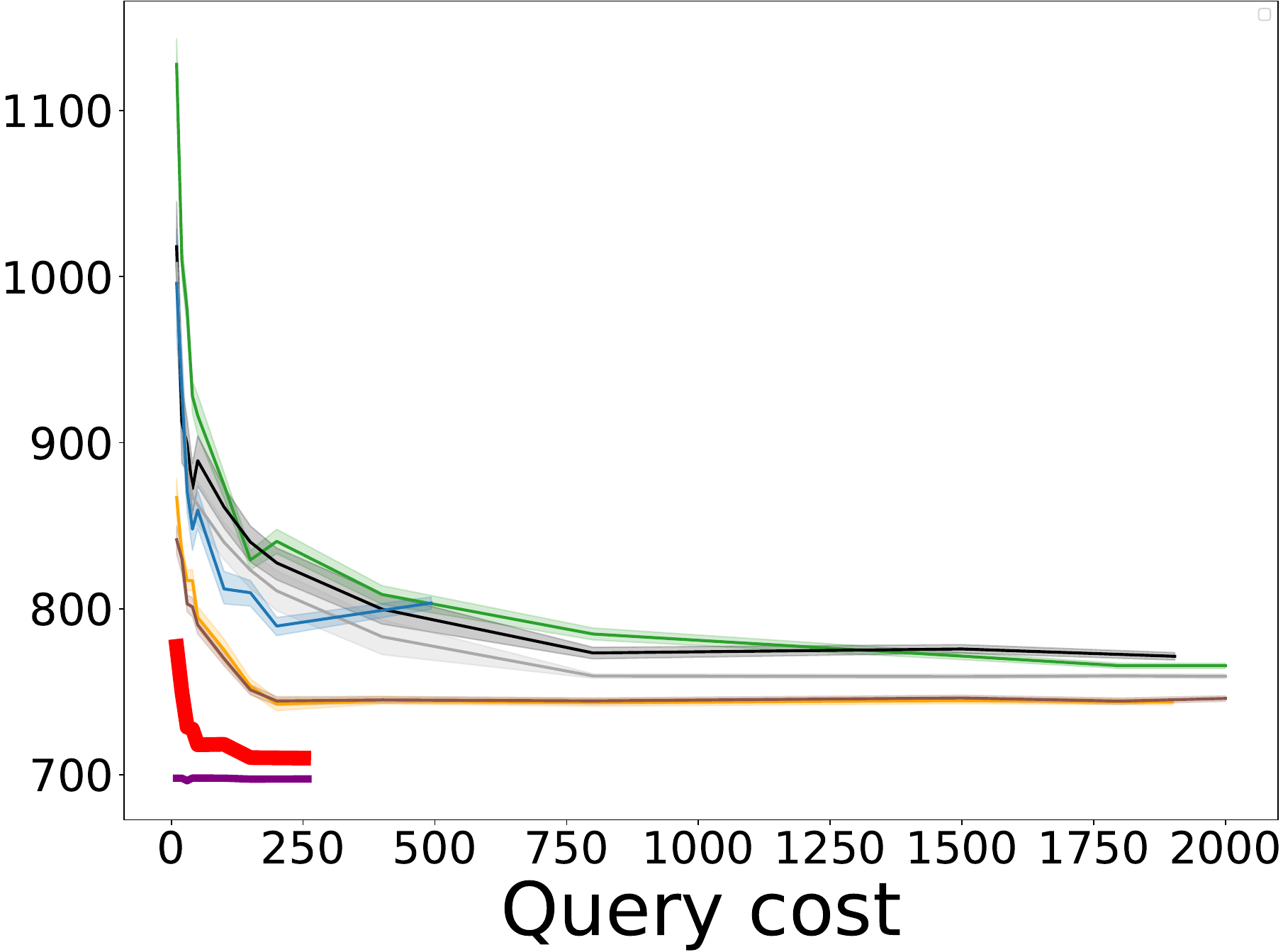}
    \caption{DRIFT}\label{fig:results:drift}
\end{subfigure}
\begin{subfigure}{.24\textwidth}
    \centering
    \includegraphics[%
    width=3.4cm,  clip={0,0,0,0}]{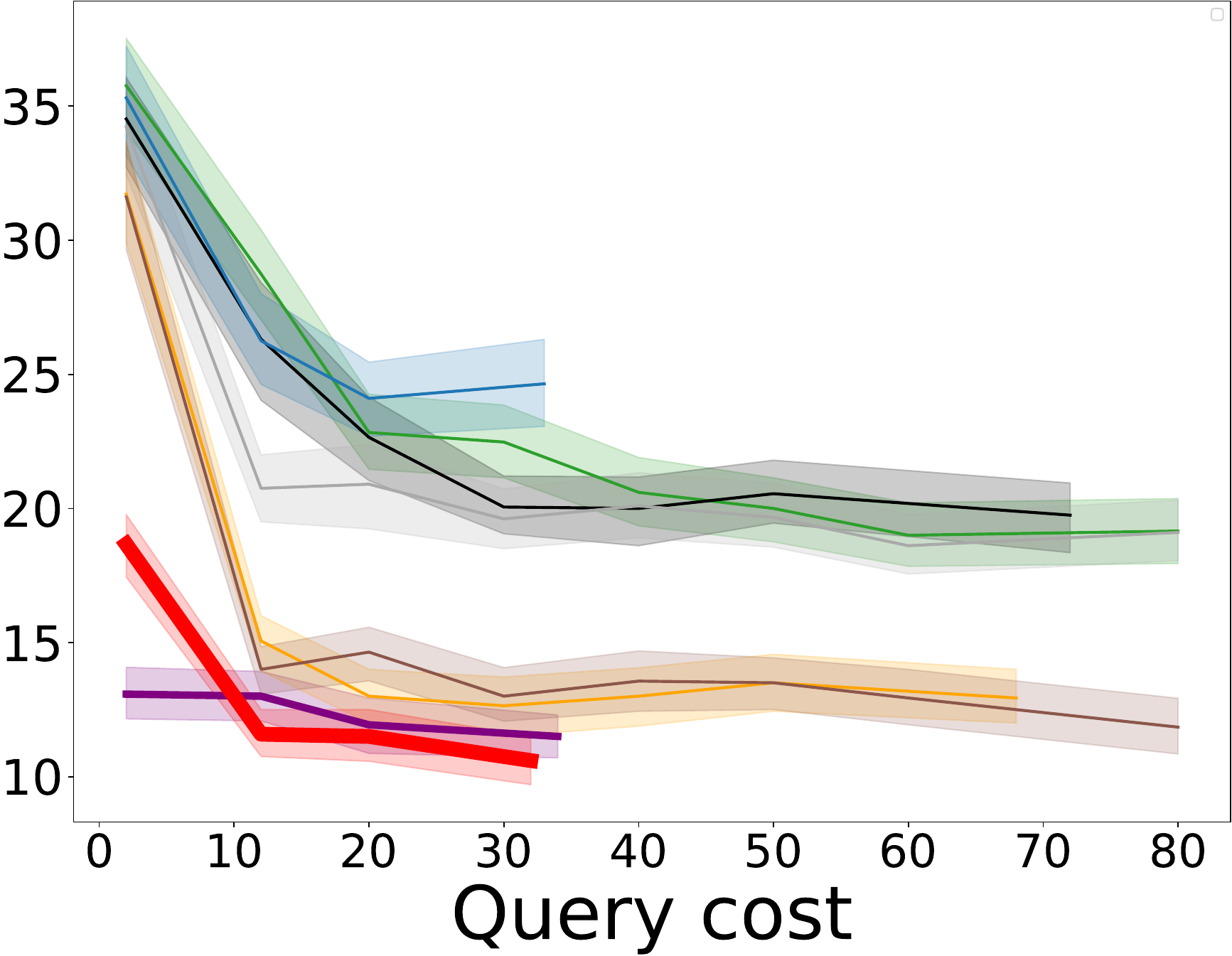}
    \caption{VERTEBRAL}\label{fig:results:vertebral}
\end{subfigure}
\begin{subfigure}{.24\textwidth}
    \centering
    \includegraphics[%
    width=3.4cm,  clip={0,0,0,0}]{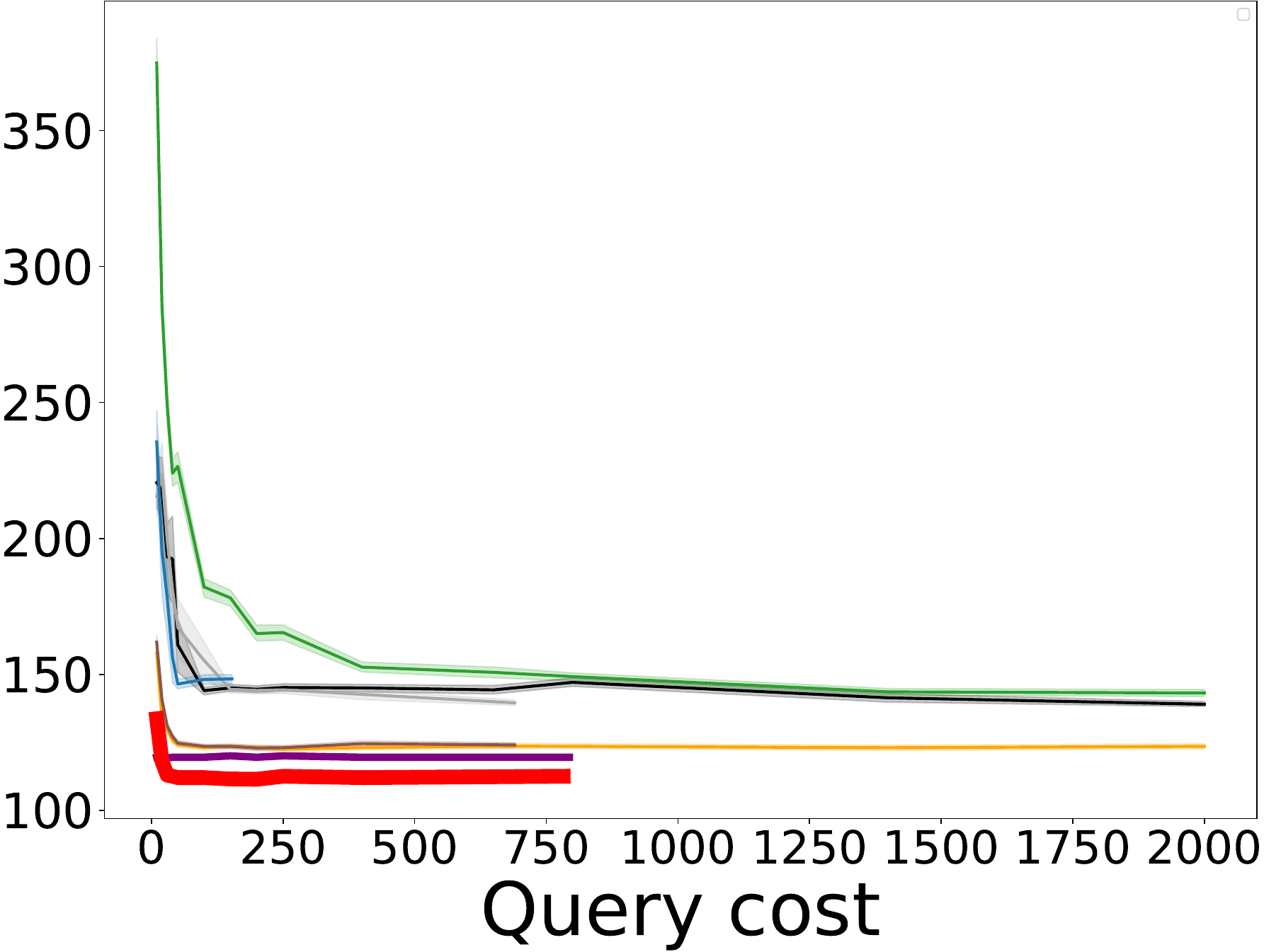}
    \caption{HIV}\label{fig:results:hiv}
\end{subfigure}
        \caption{\textbf{Main results.} Comparison of \algname with 7 baselines across 4 diverse benchmarks in terms of cost effectiveness. We plot the cumulative loss as we increase the query cost for a fixed number of rounds $T$ and maximal query cost $B$ (from left to right:  $T=10000, 3000, 80, 4000$, and $B=1200,2000,80,2000$). \algname outperforms all baselines. \textbf{Algorithms}: 4 contextual \{Oracle, CQBC, CIWAL, CAMS\} and 4 non-contextual baselines \{RS, QBC, IWAL, MP\} are included (see Section \pref{main:baseines}). 90\% confident interval are indicated in shades. %
        }
  \label{fig:exp:results}
\end{figure*}

\paragraph{Datasets.} %
We evaluate our approach using five datasets: %
(1) CIFAR10 \citep{krizhevsky2009learning} contains 60,000 images from 10 different balanced classes. (2) DRIFT \citep{vergara2012chemical} is a tabular dataset with 128-dimensional features, based on 13,910 chemical sensor measurements of 6 types of gases at various concentration levels. 
(3) VERTEBRAL \citep{asuncion2007uci} is a biomedical tabular dataset which %
classifies 310 patients into three classes (Normal, Spondylolisthesis, Disk Hernia) based on 6 attributes. (4) HIV \citep{wu2018moleculenet} contains over 40,000 compounds annotated with molecular graph features and binary labels (active, inactive) indicating their ability to inhibit HIV replication. %
{(5) CovType~\citep{Dua:2019} 
has 580K samples and contains details including slope, aspect, elevation, 
measurements of 
area, 
and type of forest cover.
}

\looseness -1 \noindent \textbf{Policy sets.
} %
We construct the policy sets $\Pi$ for each dataset following a procedure similar to Meta-selector~\citep{luo2020metaselector}. In this approach, a set of recommender algorithms is considered, and Meta-selector assigns varying ratings to these algorithms based on the specific user. Concretely, 
we first construct a set of models trained on %
different subsamples from each dataset. We then construct a set of policies, which include \textit{malicious}, \textit{normal}, \textit{random}, and \textit{biased} policy types for each dataset based on different models and features. %
Details on the classifiers and policies are provided in the supplemental materials. 
The \textit{malicious} policy provides contrary advice; the \textit{random} policy provides random advice; the \textit{biased} policy provides biased advice by training on a biased distribution for classifying specific classes. The \textit{normal} policy gives reasonable advice, being trained under a standard process on the training set. We represent the output of the $i_{th}$ policy as $\policy_i\paren{\instance_t}$, indicating the rewards distribution of all the base classifiers on $\instance_t$. In total, we create 80, 10, 6, 4 classifiers and 85, 11, 17, 20 policies for CIFAR10, DRIFT, VERTEBRAL, and HIV, respectively. %

\paragraph{Baselines.}\label{main:baseines} %
We evaluate \algname against both \textit{contextual} and \textit{non-contextual} active model selection baselines. We consider four \textit{non-contextual} baselines: (1) Random Query Strategy (RS) which queries the instance label with a fixed probability $\frac{\budget}{T}$; (2) Model Picker (MP) \citep{karimi2021online} that employs variance-based active sampling with a coin-flip query probability $\max\curlybracket{v\paren{\pdB_t,\ModelsDistB_t},\eta_t}$, where the variance term is defined as $v\paren{\pdB_t,\ModelsDistB_t}=\max_{\clabel\in \domClabel}\bar\loss_t^{\clabel} \paren{1-\bar\loss_t^{\clabel}}$; (3) Query by Committee (QBC) implementing committee-based sampling \citep{dagan1995committee};
and (4) Importance Weighted Active Learning (IWAL) \citep{beygelzimer2009importance} that calculates query probability based on labeling disagreements of surviving classifiers. %
Since no \textit{contextual} baselines exist yet, we propose contextual versions of QBC and IWAL as (5) CQBC and (6) CIWAL. Both extensions maintain their respective original query strategies but incorporate the context into the cumulative rewards. 
For \textit{model selection}, \algname, MP, CQBC, and CIWAL recommend the classifier with the highest probability. The other baselines use Follow-the-Leader (FTL), recommending the model with the minimum cumulative loss for past queried instances. Finally, we add (7) {Oracle} to represent the best single policy with the minimum cumulative loss, with the same query strategy as \algname. %

\begin{figure*}[t!]

\begin{subfigure}{.245\textwidth}
    \centering
    \includegraphics[height=2.6cm,  clip={0,0,0,0}]{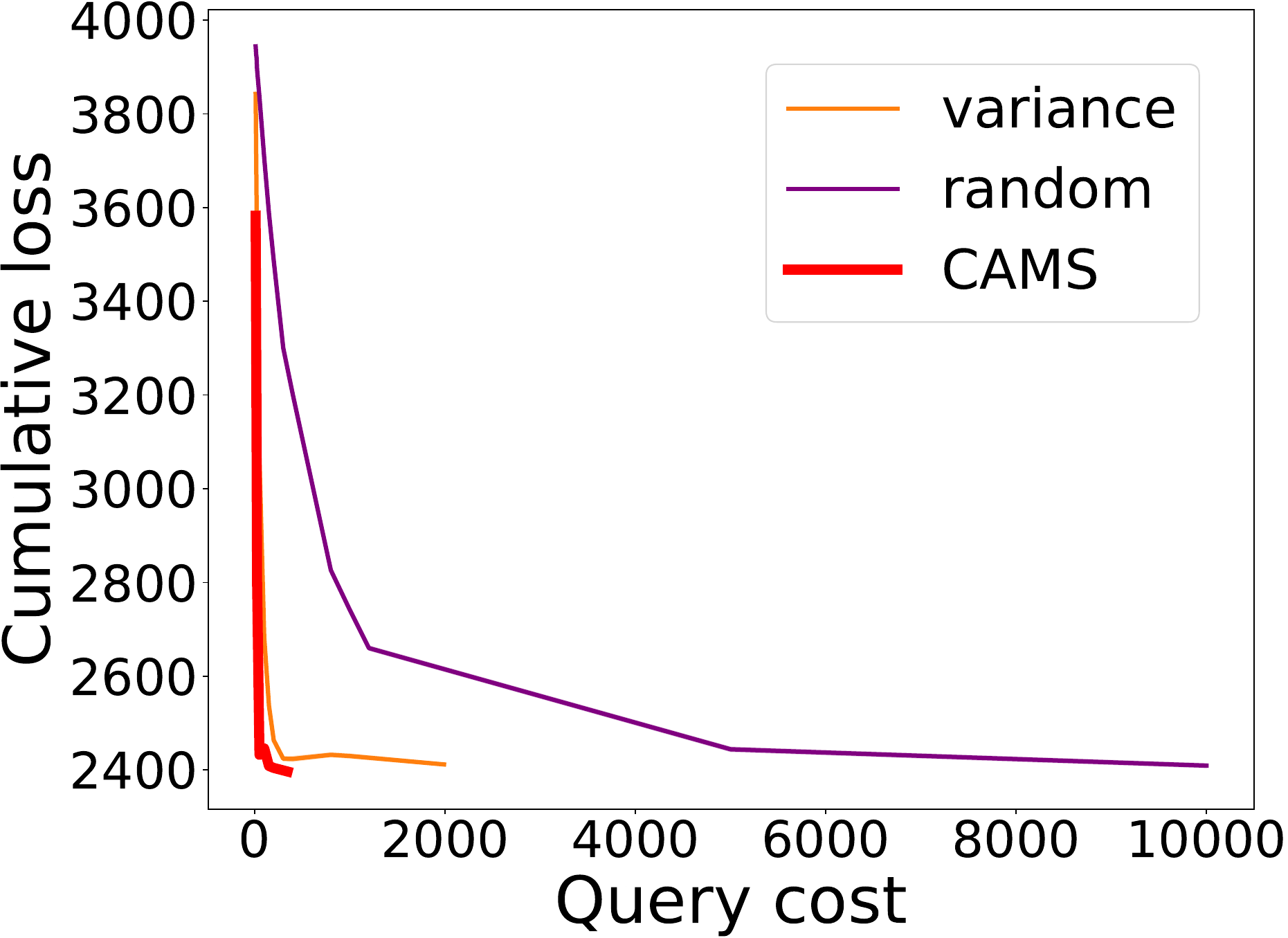}
    \caption{ \scriptsize Query strategy ablation study %
    }\label{fig:exp:active_query_study2}
\end{subfigure}
\begin{subfigure}{.245\textwidth}
    \centering
    \includegraphics[height=2.6cm,  clip={0,0,0,0}]{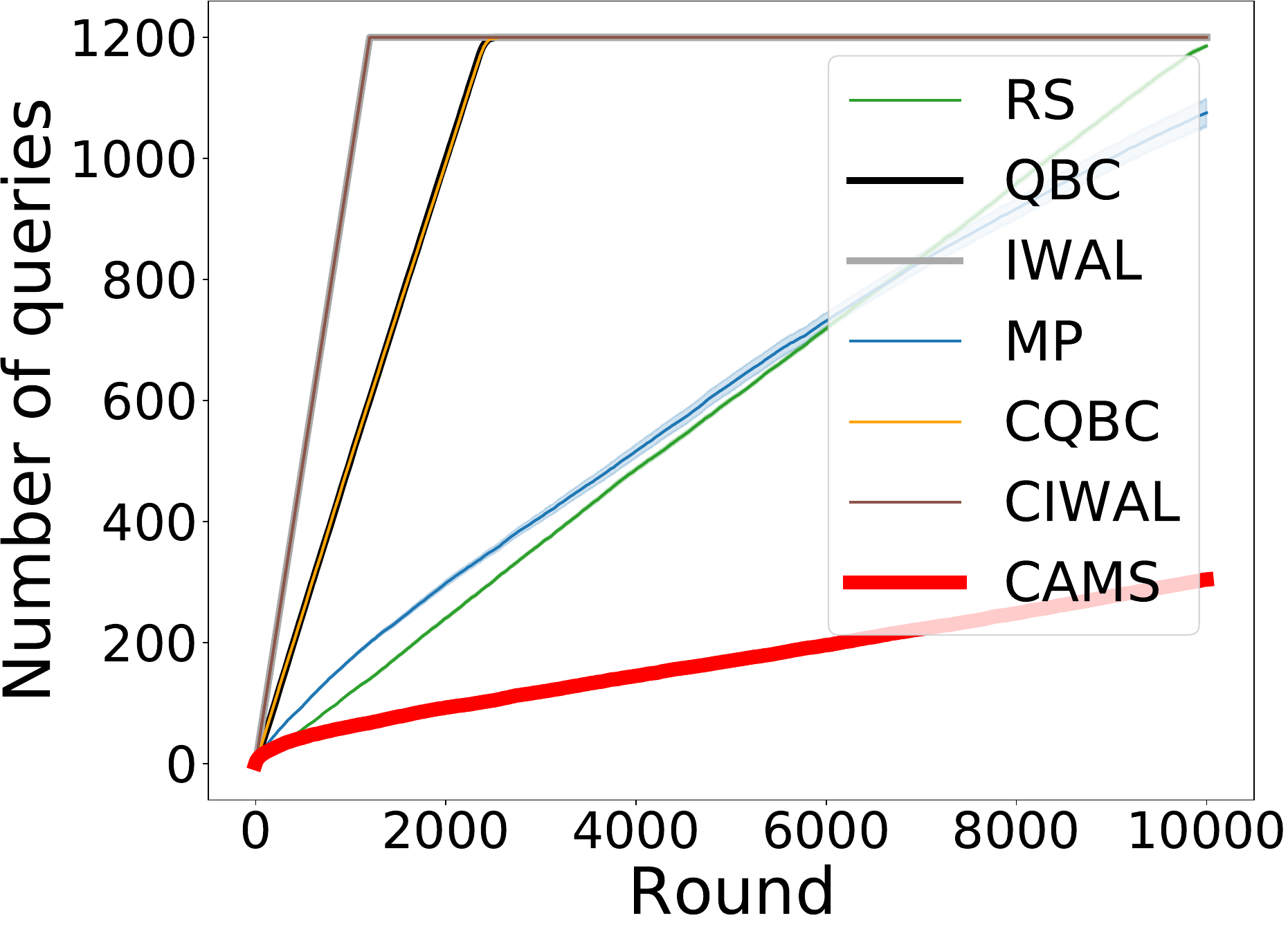}
    \caption{ \scriptsize Query complexity study%
    }
    \label{fig:results:query_complexity}
\end{subfigure}
\begin{subfigure}{.245\textwidth}
    \centering
        \includegraphics[height=2.6cm,  clip={0,0,0,0}]{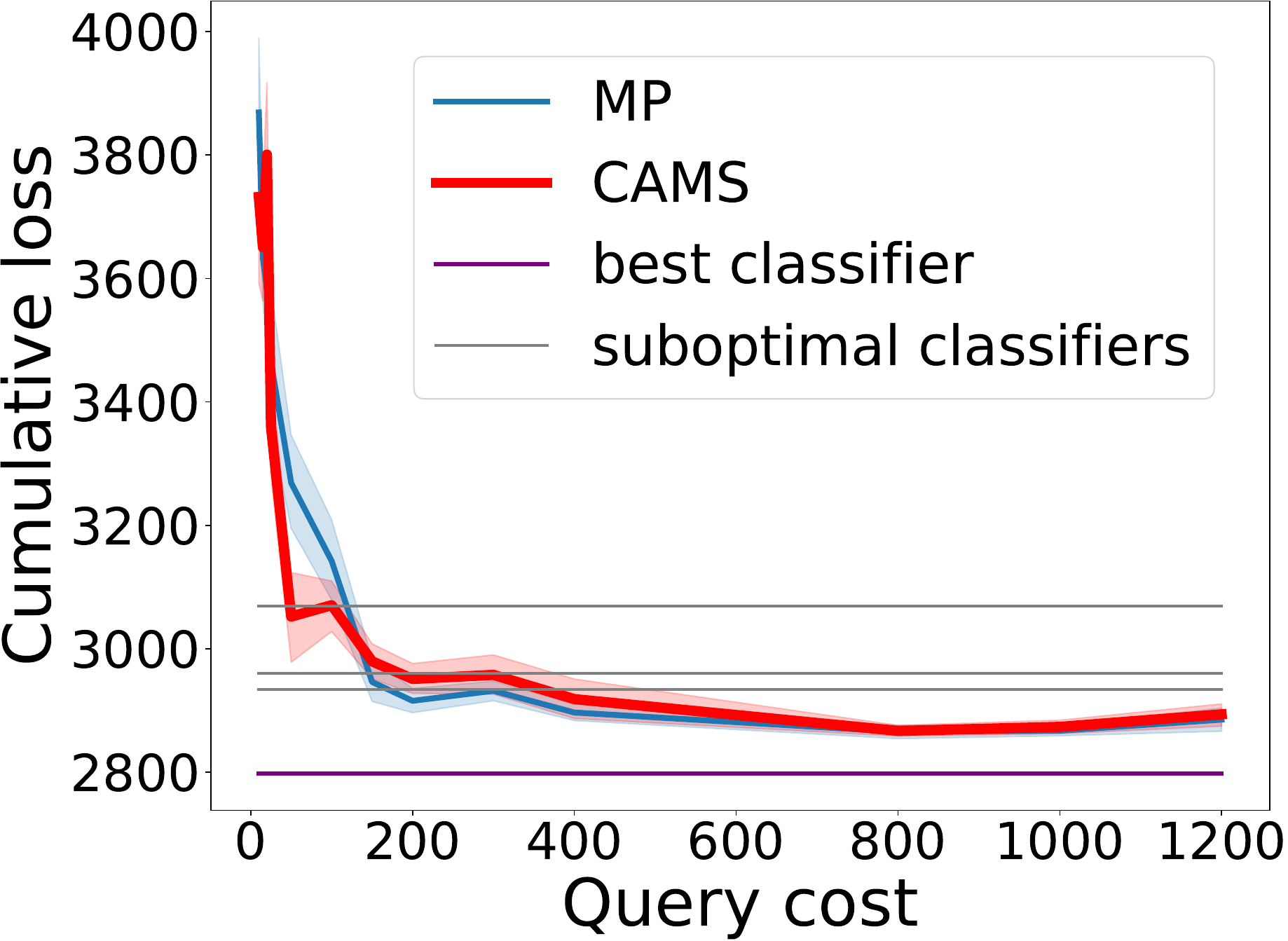}
    \caption{ \scriptsize Context-free environment%
    }\label{fig:exp:context-free}
\end{subfigure}
\begin{subfigure}{.245\textwidth}
    \centering
        \includegraphics[height=2.6cm,  clip={0,0,0,0}]{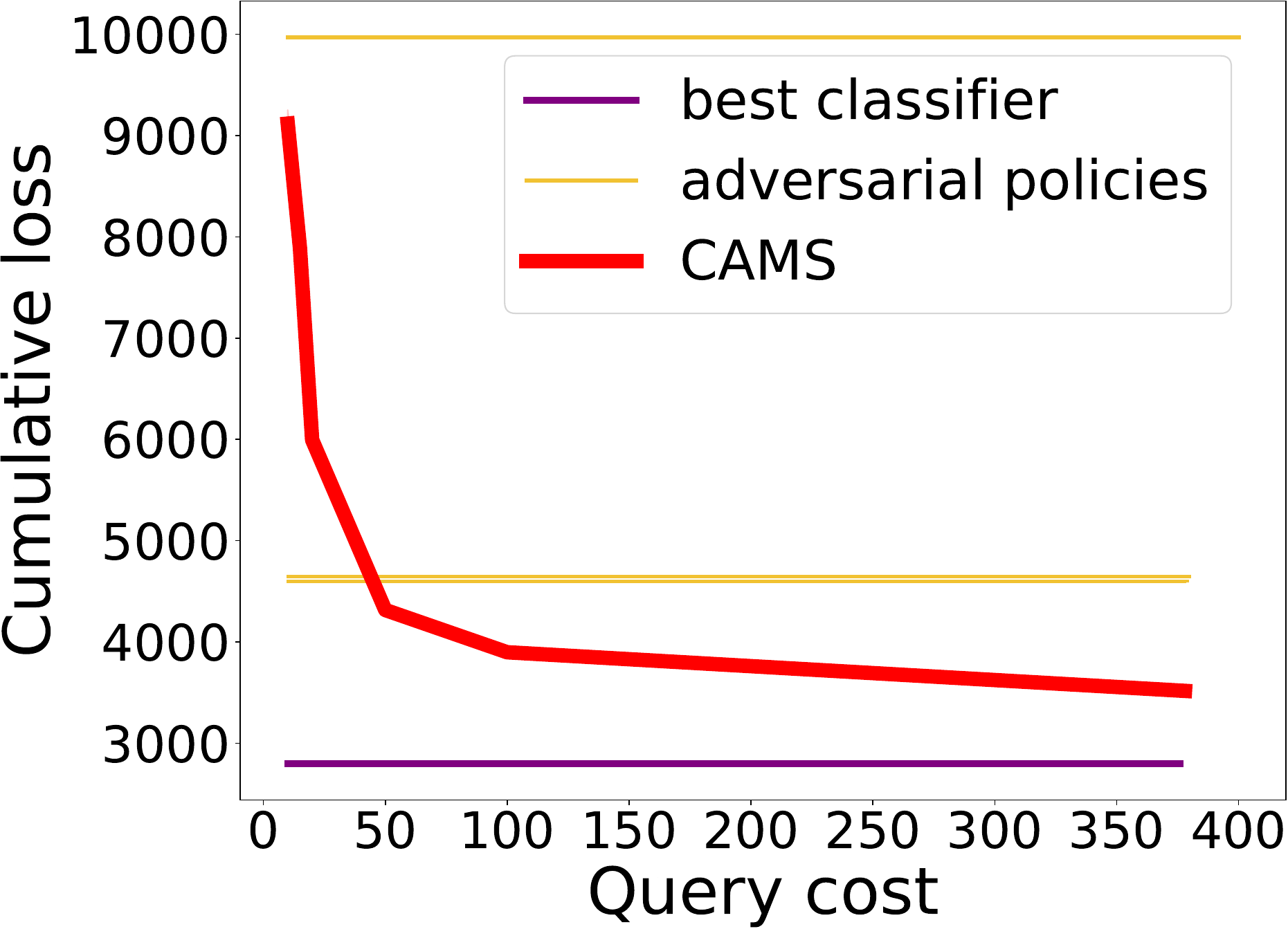}
    \caption{ \scriptsize Robust in pure adversarial set%
    }\label{fig:exp:adversarial}
\end{subfigure}

\vspace{1mm}

\begin{subfigure}{.245\textwidth}
    \centering
        \includegraphics[height=2.45cm, clip={0,0,0,0}]{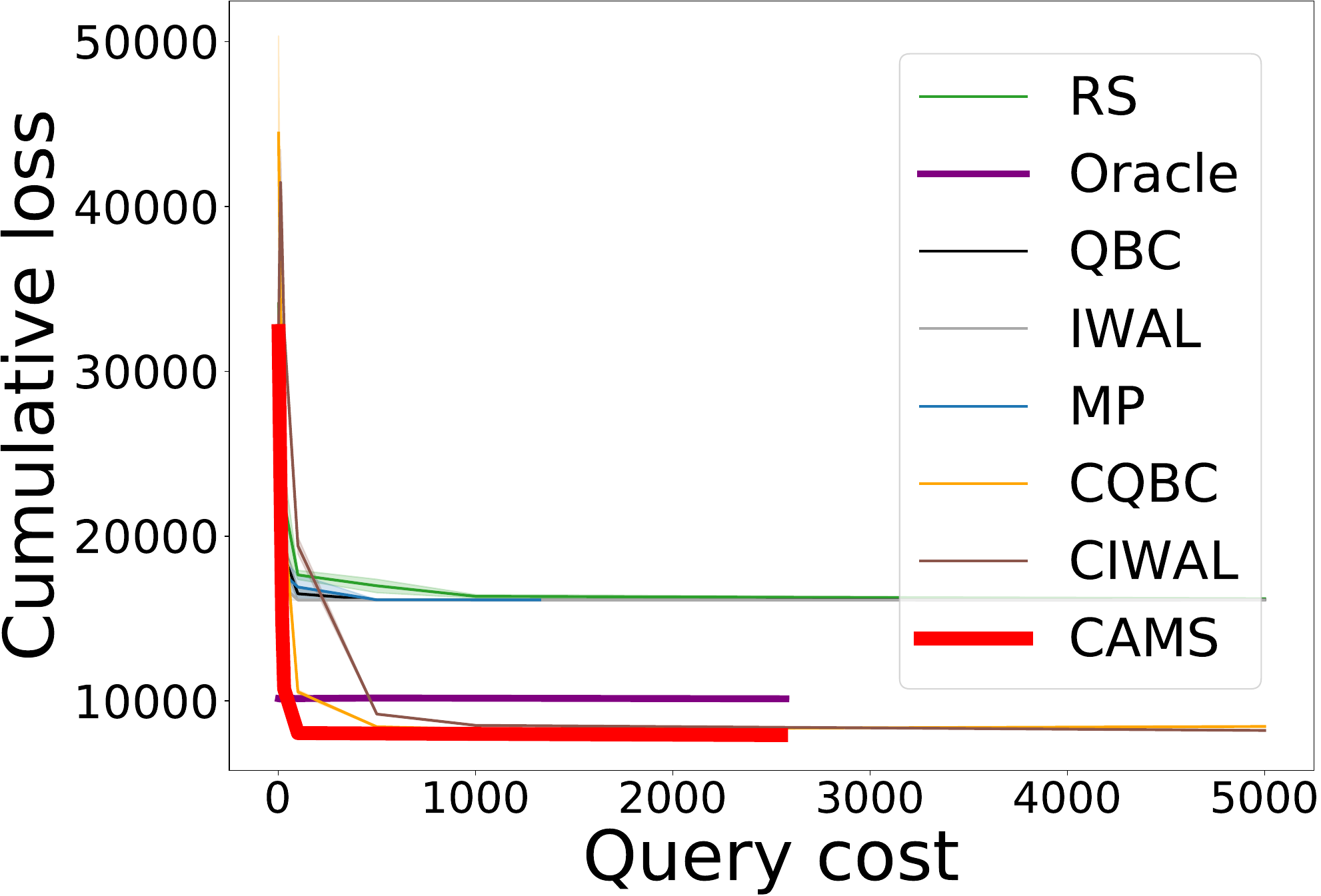}
    \caption{ \scriptsize %
    {Large dataset (COVTYPE)%
    }}\label{fig:exp:scalability}
\end{subfigure}
\begin{subfigure}{.245\textwidth}
    \centering
        \includegraphics[height=2.5cm,  clip={0,0,0,0}]{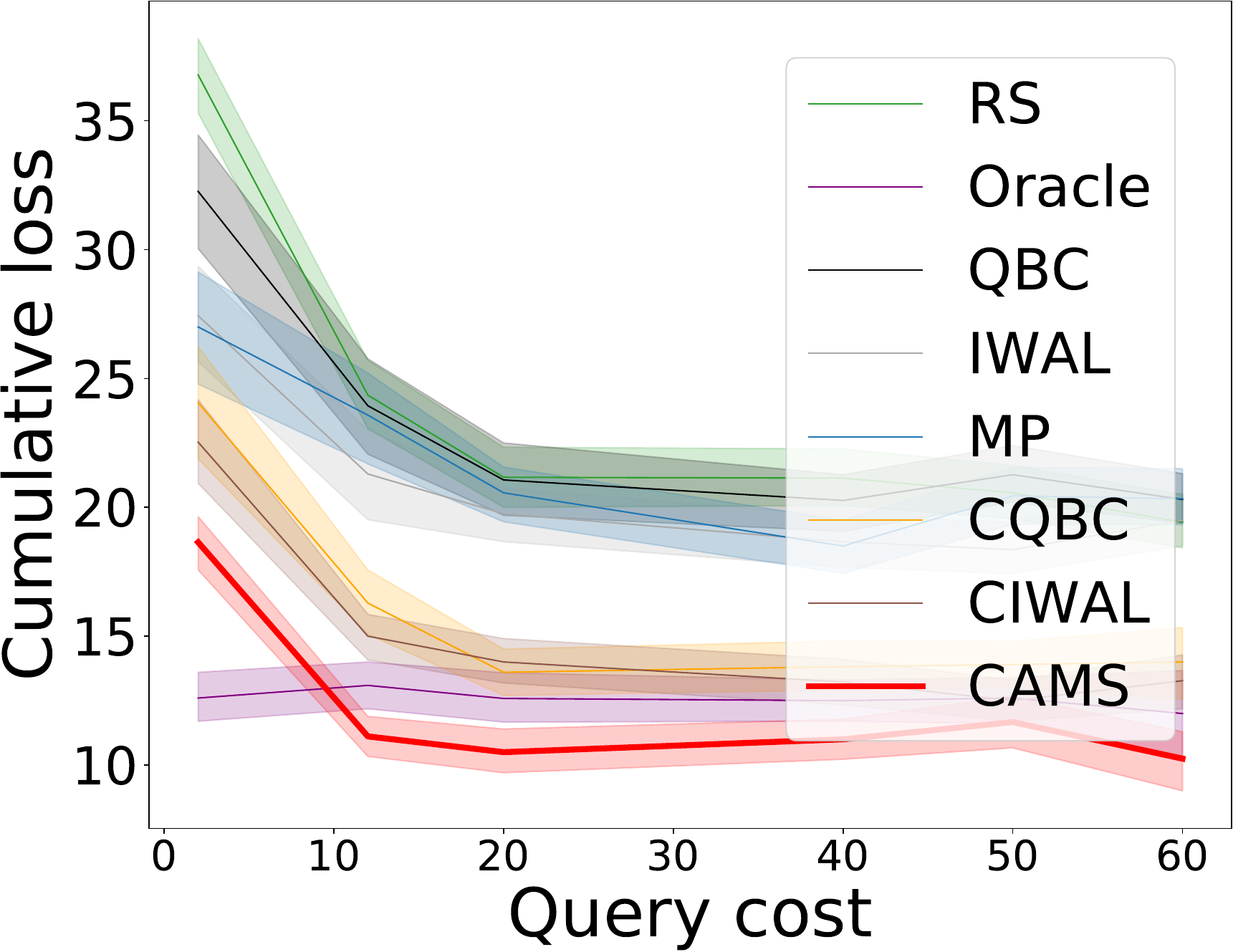}
    \caption{ \scriptsize %
    {Adjust \fix{prob.} (VERTEBRAL)%
    }}
    \label{fig:exp:adjust_vertebral}
\end{subfigure}
\begin{subfigure}{.245\textwidth}
    \centering
        \includegraphics[height=2.6cm,  clip={0,0,0,0}]{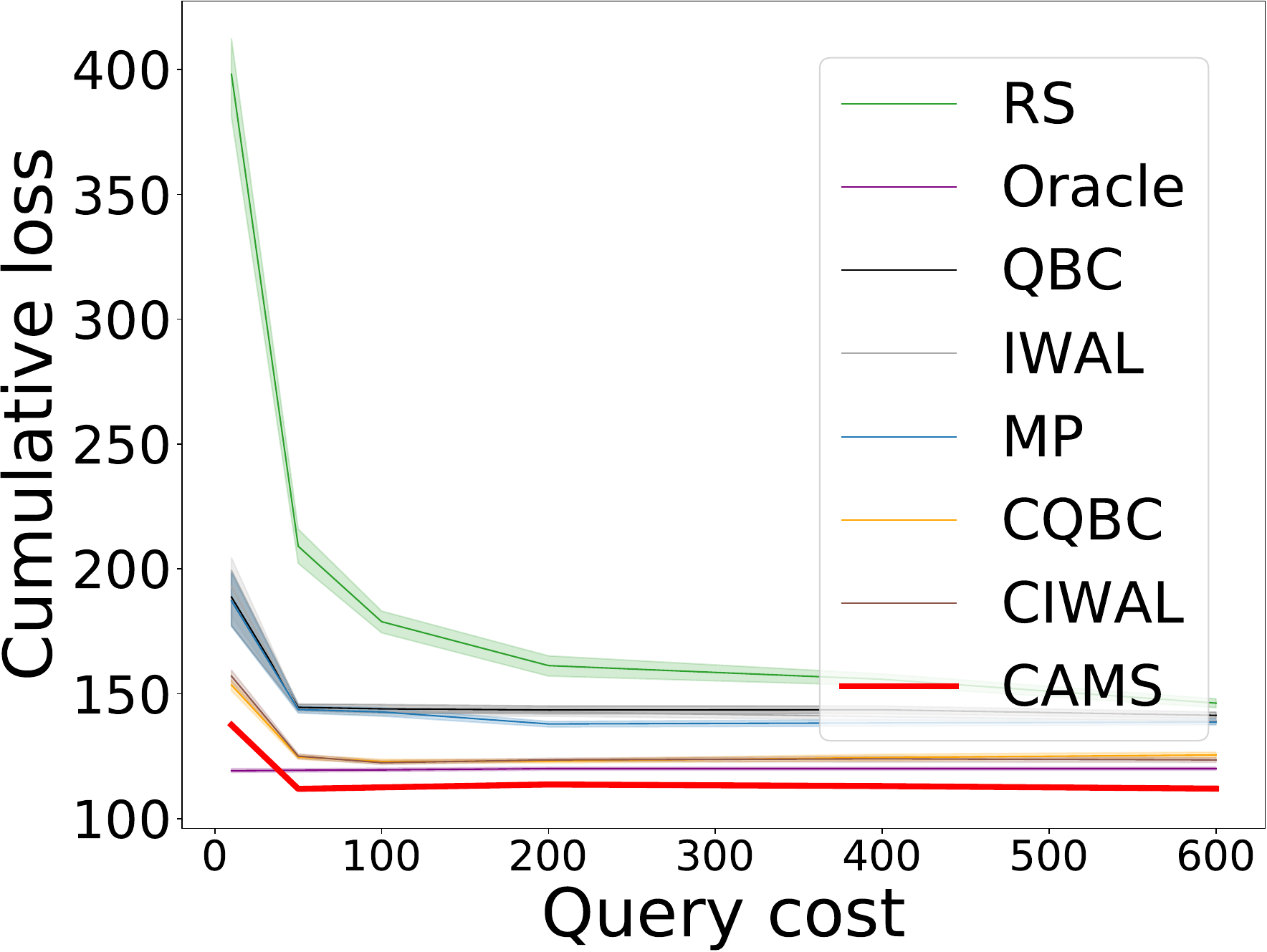}
    \caption{ \scriptsize %
    {Adjust probability (HIV) %
    }
    }
    \label{fig:exp:adjust_hiv}
\end{subfigure}
\begin{subfigure}{.245\textwidth}
    \centering
    \includegraphics[height=2.6cm,  clip={0,0,0,0}]{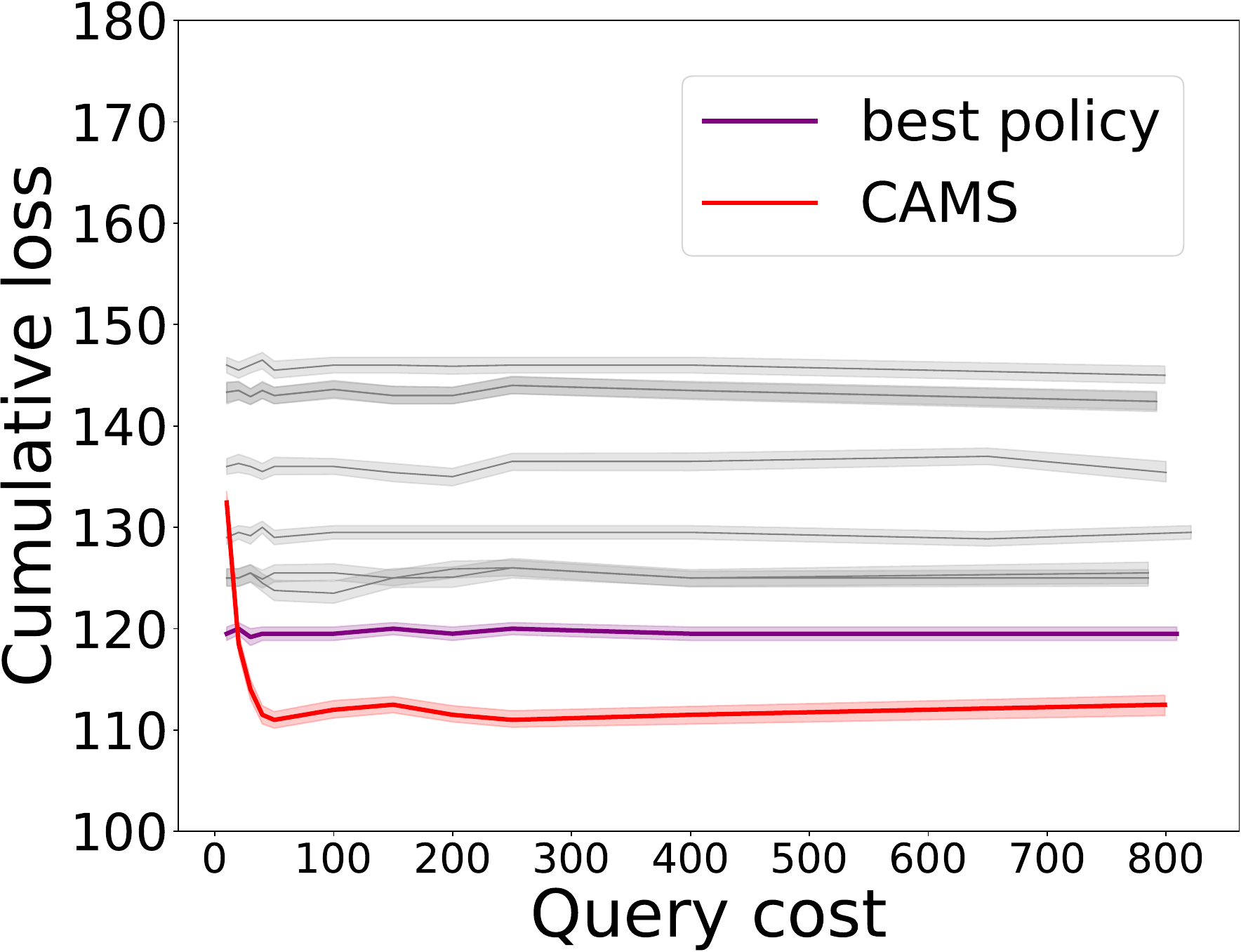}
    \caption{ \scriptsize %
    \algname vs best policy %
    (HIV)}%
    \label{fig:results:outperform_all}
\end{subfigure}

    \centering
    
    \caption{\textbf{Ablation studies.} (a) Comparing three query strategies $\curlybracket{\algname\textrm{, variance-based,    random}}$ under same model selection policy. (b) Comparing the increasing rate of CAMS' query cost over other baselines. (c) Comparing CAMS with MP in context-free environment. (d) Evaluating the performance of \algname under a pure adversarial setting. %
    {(e) Large dataset. (f,g) Adjustable query probability.} (h) \algname outperforms the best single policy.
    {The ablation study (a)-(d) is conducted on CIFAR10. %
    For additional results on other benchmarks, please refer to the supplemental material. }}
    
  \label{fig:exp:ablation}
\vspace{-3mm}
\end{figure*}

\subsection{Main results}
{\figref{fig:exp:results} visualizes the \emph{cost effectiveness} of \algname and the baselines. Here, we define \emph{cost effectiveness} as the measure of how %
{quickly} the cumulative loss decreases in response to an increase in query cost. }
\figref{fig:exp:results}
demonstrates that \algname outperforms all the comparison methods across all benchmarks. 
Remarkably, it outperforms even the oracle on the VERTEBRAL %
{(\figref{fig:results:vertebral})} and HIV %
{(\figref{fig:results:hiv})} benchmarks with fewer than 10 and 20 queries, respectively. In the case of the VERTEBRAL benchmark, \algname outperforms the best baseline in query cost by a margin of $20\%$, despite the fact that 11 out of the 17 experts provided malicious or random advice. This level of performance is attained by utilizing an active query strategy to retrieve highly informative data, thereby maximizing the differentiation between models and policies within the constraints of a limited budget. Additionally, the model selection strategy allows for effectively combining the expertise among the experts. %

\subsection{Ablation studies}

{\textbf{%
{Effectiveness} of active querying.} %
{In  \figref{fig:exp:active_query_study2} and \figref{fig:results:query_complexity}},
we perform ablation studies %
to demonstrate the effectiveness of our active query strategy.
We fix the model recommendation strategy as the one used by \algname, and compare three query strategies: (1) \algname %
, (2) the state-of-the-art variance-based query strategy from Model Picker \citep{karimi2021online} (referred to as ``variance''), 
and (3) a random query strategy. Figure \ref{fig:exp:active_query_study2} demonstrates that \algname has the fastest convergence rate in terms of cumulative loss {on CIFAR10}, implying effective use of queried labels. {Furthermore, CAMS not only achieves the minimum cumulative loss but also incurs significantly lower query costs, with reductions of 71\% and 95\% compared to the variance and random strategies respectively as showed in \figref{fig:results:query_complexity}.}
This suggests that \algname selectively queries data to optimize policy improvement, whereas the {other strategies may query}
unnecessary labels, %
including potentially noisy or uninformative ones, which impede policy improvement and convergence.

\looseness -1 \textbf{{Robustness.}} %
{ In \figref{fig:exp:results}, \ref{fig:exp:context-free},  \ref{fig:exp:adversarial},\ref{fig:exp:scalability}, \ref{fig:exp:adjust_vertebral}, and \ref{fig:exp:adjust_hiv}, \algname exhibits robustness in a variety of environmental settings.
}
Firstly, 
As shown in \figref{fig:exp:results}, \algname outshines other methods in a contextual environment, whereas in \figref{fig:exp:context-free}, a non-contextual (no experts) environment, it achieves comparable performance to the state-of-the-art Model Picker in identifying the best classifier.
Secondly, \algname is robust in both stochastic and adversarial environments. As demonstrated in \figref{fig:exp:results}, %
\algname surpasses other methods in a stochastic environment. Additionally, 
{as illustrated in \figref{fig:exp:adversarial}},
in a worst-case adversarial environment, \algname effectively recovers from adversarial actions and approaches the performance of the best classifier (see \appref{app:recover_in_complete_malicious_environment}). 
{
We further observe that \algname demonstrates robustness to varying scales of data, where the online stream sizes range from 80 to 10K (\figref{fig:exp:results}) to 100K (\figref{fig:exp:scalability}, where we randomly sample 100K samples from the {{CovType}} dataset~\citep{Dua:2019}). %

\looseness -1 In \figref{fig:exp:results}, we assume that the stream length $T$ is hidden and not used as input to \algname. Under the stochastic setting, however, knowing $T$ can provide additional information that one can leverage to optimize the query probability, thereby giving an advantage to some of the baseline algorithms (e.g. random). 
As an ablation study, in \figref{fig:exp:adjust_vertebral} and \figref{fig:exp:adjust_hiv}, 
we assume the stochastic setting where the total length $T$ of the online stream is given. 
Given the stream length $T$ and query budget $b$, we may optimize each algorithm by scaling their query probabilities, so that each algorithm allocates its query budget to the top $b$ informative labels in the entire online stream based on its own query criterion. 
\algname still ourperform the baselines under the setting.
}

{

\looseness -1 \noindent \textbf{Improvement over the best classifier and policy.}
{\figref{fig:results:outperform_all} demonstrates that} when provided with good policies, \algname formulates a stronger policy which incurs no regret. \algname has the potential to outperform an oracle, especially in rounds where the oracle does not make the optimal recommendation. For instance, in the stochastic version of \algname (as shown in lines 22-23 and 30-32 of \figref{alg:CAMS}), \algname recommends a model using a weighted majority vote among all policies, enabling the formation of a new policy in each round by amalgamating the strengths of each sub-optimal policy. This adaptive strategy can potentially outperform any single {policy}. Moreover, in most real-world scenarios and conducted experiments (as depicted in %
\appref{app:outperform_best_expert}), data streams may not be strictly stochastic, and therefore no single policy consistently performs the best. In such cases, \algname's weighted policy may find an enhanced combination of ``advices'', leading to improved performance.}

\section{Conclusion}\label{sec:conclusion}
We introduced \algname, an online contextual active model selection framework based on a novel model selection and active query strategy. The algorithm was motivated by %
many real-world use cases that need to make decision by taking both contextual information and the cost into consideration. We have demonstrated \algname's compelling performance of using the minimum query cost to learn the optimal contextual model selection policy on several diverse online model selection tasks. In addition to the promising empirical performance, we also provided rigorous theoretical guarantees on the regret and query complexity for both stochastic and adversarial settings. %
We hope our work can inspire future works
to handle more complex real-world model selection tasks (e.g. beyond classification or non-uniform loss functions, etc. where our analysis does not readily apply).

\clearpage
\subsubsection*{Acknowledgements}
This work is supported in part by the RadBio-AI project (DE-AC02-06CH11357), U.S. Department of Energy Office of Science, Office of Biological and Environment Research, the IMPROVE project under contract (75N91019F00134, 75N91019D00024, 89233218CNA000001, DE-AC02-06-CH11357, DE-AC52-07NA27344, DE-AC05-00OR22725), the Laboratory Directed Research and Development (LDRD) funding from Argonne National Laboratory provided by the Director, Office of Science, of the U.S.\ Department of Energy under Contract No. DE-AC02-06CH11357, the Exascale Computing Project (17-SC-20-SC), a collaborative effort of the U.S.\ Department of Energy Office of Science and the National Nuclear Security Administration,  the University of Chicago Joint Task Force Initiative, the AI-Assisted Hybrid Renewable Energy, Nutrient, and Water Recovery project (DOE DE-EE0009505), and the National Science Foundation under Grant No. IIS 2313131 and IIS 2332475. %

\bibliographystyle{plainnat}
\bibliography{reference}
\iftoggle{longversion}{
    \clearpage
    \onecolumn
    \appendix

\section{Impact Statements}
This paper introduces a novel framework for adaptive model selection in label-efficient learning. By integrating robust online learning with active query strategies, our algorithm effectively adapts to varying data contexts and minimizes labeling efforts, crucial in domains requiring swift and accurate decisions, such as disease identification and financial predictions. Ethically, the framework's design promotes efficient and context-aware model selection, reducing potential biases associated with context-ignorant model selections. No major ethical concerns are anticipated, given the algorithm's generality and focus on solving practical problems.

\section{Table of Notations Defined in the Main Paper}\label{app:notations}

\begin{table*}[h]
\scalebox{0.67}{
\begin{tabular}{l l}
\toprule
\textbf{notation} & \textbf{meaning} \\
\midrule
&\textbf{\quad\quad\quad\quad\quad Problem Statement}\\
\domInstance & input domain \\
\instance & input instance, $\instance \in \domInstance$\\
$t$, $T$ & index of a round, total number of rounds \\
$\policyIndex$, $\modelIndex$ & index of policies, models/classifiers\\
$\PoliciesNum$  & number of policies \\
$\ModelsNum$  &  number of classifiers \\
\domClabel & \curlybracket{0, \dots, \numClabel-1}, set of $\numClabel$ possible class labels for each input instance\\
$\numClabel$ & number of class labels, { $|\domClabel|$}  \\
$\Delta^{\ModelsNum-1}$ & $\ModelsNum$-dimensional probability simplex $\curlybracket{\ModelsDistB \in \reals^\ModelsNum: |\ModelsDistB|=1, \ModelsDistB \geq 0}$\\
\model  & single pre-trained classifier (model) \\
\Models  &  \curlybracket{\model_1, \dots, \model_k}, set of $\ModelsNum$ pre-trained classifiers over $\domInstance\times \domClabel$\\
\policy, \policy\paren{\instance} &   model selection policy (expert) $\policy: \domInstance \rightarrow \Delta^{\ModelsNum-1}$, probability of selecting each classifier under input $\instance$\\
$\policy^{\text{const}}$ &  $\policy^{\text{const}}_j (\cdot) := \mathbi{e}_j$, $\mathbi{e}_j \in \Delta^{\ModelsNum-1}$ denotes the canonical basis vector with $e_j=1$\\
\Policies &  collection of model selection policies \\
$\Policies^*$ & $\Policies \cup \{\pi^{\text{const}}_1, \dots, \pi^{\text{const}}_k\}$, extended policy set including constant policies that always suggest a fixed model\\
|\Policies|, |$\Policies^*$| & \PoliciesNum, $|\Policies^*|\leq \paren{\PoliciesNum+\ModelsNum}$\\
$\policy^*$ & $\policy^*\in \Policies^*$, best policy\\

$\pd_{t,j}$ & $\model_j\paren{\instance_t}$, predicted label for $j_{\text{th}}$ pre-trained classifier at round $t$\\
$\clabel_t$ & true label of $\instance_t$\\

$\pdB_t$ & $[\pd_{t,1}, \dots, \pd_{t,\ModelsNum}]^\top$, predicted labels by all $\ModelsNum$ models at round t\\

$\loss_{t,j}$ & $\mathbb{I}_{\curlybracket{\pd_{t,j}\neq{\clabel_t}}}$,  0-1 loss for model $j\in [k]$ at round $t$\\
$\lossB_{t}$ & $\mathbb{I}_{\curlybracket{\pdB_t\neq{\clabel_t}}}$, full loss vector upon observing $\clabel_t$\\
\Learner & the learner\\
${L}^\Learner_T$ & $\sum_{t=1}^T{\loss}_{t,j_t}$, cumulative loss over $T$ rounds for a learning algorithm \Learner\\
$\tilde{\loss}_{t,i}$ & $ \langle \policy_i\paren{\instance_t}, {\lossB}_{t} \rangle$, %
expected loss if the learner commits to policy $\policy_i$ and take random selection at round $t$\\
$\maxind{\ModelsDistB}$ & $\argmax_{j, \ModelsDist_{j}\in \ModelsDistB } \ModelsDist_{j}$, index of maximal value entry of $\ModelsDistB$\\
$\mu$ & $\frac1T\sum_{t=1}^T\expctover{\instance_t,\clabel_t}{\hat{\loss}_{t,\maxind{\policy_i(\instance_t)}}}$\\
$\mathcal{R}_T\paren{\Learner}$, $\overline{\mathcal{R}}_T\paren{\Learner}$ & expected regret in adversarial setting, pseudo-regret for stochastic setting\\
$\expectation_t[\cdot]$ &  $\expectation[\cdot|\mathcal{F}_{t}]$, $\mathcal{F}_{t}=\sigma\paren{\adviceMatrix^{\paren{1}},\pdB_1,...,\pdB_{t-1},\adviceMatrix^{\paren{t}}}$ \\

\hline
&\textbf{\quad\quad\quad\quad\quad Algorithm}\\
$\PoliciesDistB_t$ & $(\PoliciesDist_{t,i})_{i\in |\Policies^*|}$, probability distribution over $\Policies^*$ at $t$\\
$\tilde{\Loss}_{t,i}$& $\sum_{\tau=1}^{t} \tilde{\loss}_{\tau,i}$, cumulative loss of policy $i$\\
$\ModelsDistB_t$ & $\sum_{i \in |\Policies^*|} \PoliciesDist_{t,i} \policy_i(\instance_t)$, distribution induced by the weighted policy \\
$\bar\loss_t^{\clabel}$&$\langle \ModelsDistB_t, \mathbb{I}\curlybracket{\pdB_t \neq \clabel} \rangle$, expected loss if the true label is $\clabel$\\
$\entropyF\paren{\pdB_t,\ModelsDistB_t}$ &  model disagreement function\\
$h(x)$&$-x\log x$\\
$\delta_0^t$& $\frac{1}{\sqrt{t}}$, lower bound of query probability\\
$\queryProb_t$&$ \max{\curlybracket{\delta_0^t,\entropyF\paren{\pdB_t,\ModelsDistB_t}}}$, adaptive query probability\\
$\hat{\loss}_{t,j}$&$\frac{\loss_{t,j}}{\queryProb_t}\QueryIndicator_t$\\
$\QueryIndicator$ &  query indicator\\
$\eta_t$ & adaptive learning rate\\
$\minpgap{t}$ & $  {1 - \max_{\tau\in [t{-1}]} \langle \ModelsDistB_{{\tau}}, \mathbb{I}\curlybracket{\pdB_{{\tau}} = \fix{\clabel_{\tau}}} \rangle}$ \\
\budget & query budget\\
$\widehat{\lossB}$, $\paren{\widehat{\loss}_{t,i}}_{i \in{[\ModelsNum]}}$ & unbiased estimate of classifier loss vector\\
$\widetilde{\lossB}$, $\paren{\widetilde{\loss}_{t,i}}_{i \in{[\PoliciesNum]}}$ &  unbiased estimate of policy loss vector\\
$\widehat{\LossB}$%
& unbiased cumulative loss of classifiers, policies\\

\hline
&\textbf{\quad\quad\quad\quad\quad Analysis}\\
$p_{t,y} $ & $\sum_{j \in [k]} \mathbb{I}\curlybracket{\pd_{t,j} = y} \ModelsDist_j$, the total probability of classifiers predicts label $y$ at round $t$\\
\\
${\Delta}$ & $\min_{i\neq i^*}{\Delta}_i=\min_{i\neq i^*} (\mu_i - \mu_{i^*})$, sub-optimality gap\\
$\gamma$& $ \min_{\instance_t}\left\{\max_{\ModelsDist_j \in \ModelsDistB_{i^*}^t} \ModelsDist_j  -\max_{\ModelsDist_j \in \ModelsDistB_{i^*}^t,j\neq\maxind{\ModelsDistB_{i^*}^t}} \ModelsDist_j  \right\}$, sub-optimality model probability gap of $\policy_{i^*}$\\

${\Delta}_i$&$\expectation[\widetilde{\loss}_{.,i}-\widetilde{\loss}_{.,i^*}]$, sub-optimality gap or immediate regret of policy $i$\\
$\Loss_{T,*}$ & the cumulative loss of oracle at round T\\
\bottomrule
\end{tabular}
}
\caption{Notations used in the main paper}
\label{table:notation}
\end{table*}

\clearpage

    \section{Summary of Regret and Query Complexity Bounds}
We summarize the regret and query complexity bounds (if applicable) of related algorithms in \tabref{table:query_complexity_table}.

\begin{table*}[h!]
\centering
\scalebox{0.75}{
\begin{tabular}{l l l}
\toprule
\textbf{Algorithm} &  
\textbf{Regret} & 
\textbf{Query Complexity} 

\\
\midrule
\makecell[l]{{Exp3} \\ \citep{lattimore2020bandit}}& $2\sqrt{T\ModelsNum\log{\ModelsNum}}$  & --  \\
\hline
\makecell[l]{{Exp3.p} \\ \citep{bubeck2012regret}}& $5.15\sqrt{nT\log{\frac{n}{\delta}}}$ & -- \\
\hline
\makecell[l]{{Exp4} \\ \citep{lattimore2020bandit}} & $\sqrt{2T\ModelsNum\log{\PoliciesNum}}$ & -- \\
\hline 
\makecell[l]{{Exp4.p} \\ \citep{beygelzimer2011contextual}} & $6\sqrt{\ModelsNum T\ln{\frac{\PoliciesNum}{\delta}}}$& -- \\
\hline
\makecell[l]{${\text{Model Picker}}_{\text{stochastic}}$ \\ {\citep{karimi2021online}}}
 & \makecell[l]{ $62\max_{i} \Delta_i{\ModelsNum}/\paren{{\lambda}^2\log{\ModelsNum}}$ \\ $\lambda=\min_{j\in \bracket{\ModelsNum} \backslash  \curlybracket{i^*}}{\Delta^2_j}/{\theta_j}$}& $\sqrt{2T\log{\ModelsNum}}(1+4\frac{\numClabel}{\Delta})$  \\
\hline
\makecell[l]{${\text{Model Picker}}_{\text{adversarial}}$ \\  {\citep{karimi2021online}}} & $2\sqrt{2T\log{\ModelsNum}}$ & $5\sqrt{T\log{\ModelsNum}}+2\Loss_{T,*}$ \\
\hline
\makecell[l]{$\textbf{\algname}_{\stochastic}$ 
}& %
$\paren{\frac{\ln{\frac{|\Policies^*|\fix{-1}}{\gamma}}+\sqrt{\ln{|\Policies^*|}\cdot{2\fix{b^2}\ln{\frac{2}{\delta}}}}}{\sqrt{\ln{|\Policies^*|}}\Delta}}^2$
&$%
{\paren{{\paren{\frac{\ln{\frac{|\Policies^*|\fix{-1}}{\gamma}}+\sqrt{\ln{|\Policies^*|}\cdot{2\fix{b^2}\ln{\frac{2}{\delta}}}}}{\sqrt{\ln{|\Policies^*|}}\Delta}}^2}+T \mu_{i^*}}\frac{\ln {T}}{\numClabel\ln{\numClabel}}}
$ 
\\
\hline
\makecell[l]{$\textbf{\algname}_{{\adversarial}}$
} &$2\numClabel\sqrt{\ln \numClabel / {\max}\{\minpgap{T},{\sqrt{1/T}}\}} \cdot \sqrt{{T\log{{|\Policies^*|}}}}$ & %
${O}\paren{\paren{\sqrt{\frac{T\log{|\Policies^*|}}{\fixremoved{\max}\{\minpgap{T},\fixremoved{\sqrt{1/T}}\}}} + \tilde{\Loss}_{T,*}}\paren{\ln{{T}}}}$
\\ %
\bottomrule
\end{tabular}
}
\caption{Regret and query complexity bounds. For the notations in this table: $i^*$ is the model with the highest expected accuracy; $\theta_j=\Pr{\ell_{.,j}\neq\ell_{.,i^*}}$ is the probability that exactly one of $j$ and $i^*$ correctly classifies a sample; $\gamma$ and $\minpgap{T}$ are defined in Eq.~\eqref{eq:gamma_sto} and \eqref{eq:rho_sto}, respectively.
\fix{$b=p_{\min}\log_{c}{(1/p_{\min})}$, where $p_{\min} = \min_{s,i} \pi(\mathbf{x}_s)$ denotes the minimal model selection probability by any policy.}
}
\label{table:query_complexity_table}
\end{table*}

\begin{remark}
When $T\mu_{i^*},\tilde{\Loss}_{T,*}$ are regarded as constants (given by an oracle), %
the query-complexity bound is then sub-linear \textit{w.r.t.} $T$. 
\end{remark}

\begin{remark} 
Note that the number of class labels $\numClabel$ affects the quality of the query complexity bound. The intuition behind this result is, with larger number of classes, \emph{each query may carry more information upon observation}. For instance, in an extreme case where only one expert always recommends the best model and others gives random recommendations of models (and predicts random labels), having more classes lowers the chance of a model making the correct guess, and therefore helps to "filter out" those suboptimal experts in fewer rounds---hence being more query efficient.
\end{remark}

\begin{remark}
    {To prove the practical feasibility of CAMS, we have analyzed its time and space complexity. Our analysis shows that \algname has a time complexity of ${{O}\paren{Tnk}}%
    $ in total or $O(nk)$ %
    per round (due to the \textsc{Recommend} procedure under the stochastic setting), and a space complexity of ${{O}\paren{\paren{\PoliciesNum+\ModelsNum}\cdot \ModelsNum}}$. Here, $T$ refers to the online horizon, $\PoliciesNum$ denotes the number of policies, and $\ModelsNum$ denotes the number of models. Taking into account these complexities, we can confirm that CAMS is practically feasible.
    }
\end{remark}
    
\section{Supplemental Materials on Experimental Setup}\label{sec:supp:expsetup}

\subsection{Baselines}\label{app:baselines}

\paragraph{Model Picker (MP)}
Model Picker \citep{karimi2021online} is a context-free online active model selection method inspired by EXP3. Model Picker aims to find the best classifier in hindsight while making a small number of queries. For query strategy, it uses a variance-based active learning sampling method to select the most informative label to query to differentiate a pool of models, where the variance is defined as $v\paren{\pdB_t,\ModelsDistB_t}=\max_{\clabel\in \domClabel}\bar\loss_t^{\clabel} \paren{1-\bar\loss_t^{\clabel}}$. The coin-flip query probability is defined as $\max\curlybracket{v\paren{\pdB_t,\ModelsDistB_t},\eta_t}$ when $v\paren{\pdB_t,\ModelsDistB_t}\neq 0$, or 0 otherwise. For model recommendation, it uses an exponential weight algorithm to recommend the model with minimal exponential cumulative loss based on the past queried labels at each round.

\paragraph{Query by Committee (QBC)} 
For query strategy, we have adapted the method of \cite{dagan1995committee} as a disagreement-based selective sampling query strategy for online streaming data. We treat each classifier as a committee member and compute the query probability by measuring disagreement between models for each instance. The query function is coin-flip by vote entropy probability $-\frac{1}{\log{\min{\paren{k, |C|}}}}\sum_{c}\frac{V\paren{
c, x}}{k}\log\frac{V\paren{c, x}}{k}$, where $V\paren{c,x}$ stands for the number of committee members assigning a class c for input context x and k is the number of committee.
For the model recommendation part, we use the method of Follow-the-Leader (FTL) \citep{lattimore2020bandit}, which greedily recommends the model with the minimum cumulative loss for past queried instances.

\paragraph{Importance Weighted Active Learning (IWAL)}
We have implemented \cite{beygelzimer2009importance} as the IWAL baseline. For the query strategy part, IWAL computes an adaptive rejection threshold for each instance and assigns an importance weight to each classifier in the hypothesis space $\mathcal{H}_t$. IWAL retains the classifiers in the hypothesis space according to their weighted error versus the current best classifier's weighted error at round $t$. The query probability is calculated based on labeling disagreements of surviving classifiers through function $\max_{i,j\in \mathcal{H}_t,\clabel\in[\numClabel]}\loss^{\paren{\clabel}}_{t,i}-\loss^{\paren{\clabel}}_{t,j}$. For model recommendation, we also adopt the Follow-the-Leader (FTL) strategy.

\paragraph{Random Query Strategy (RS)}
The RS method queries the label of incoming instances by the coin-flip fixed probability $\frac{b}{T}$. It also uses the FTL strategy based on queried instances for model recommendation.

\paragraph{Contextual Query by Committee (CQBC)} 
We have created a contextual variant of QBC termed CQBC, which has the same entropy query strategy as the original QBC. For model recommendation, we combine two model selection strategies. The first strategy calculates the cumulative reward of each classifier based on past queries and normalizes it as a probability simplex vector. We also adopt Exp4's arm recommending vector to use contextual information. Finally, we compute the element-wise product of the two vectors and normalize it to be CQBC's model recommendation vector. At each round, CQBC would recommend the top model based on the classifiers' historical performance on queried instances and the online advice matrix for streaming data.

\paragraph{Contextual Importance Weighted Active Learning (CIWAL)} 
We have created a variant version of importance-weighted active learning. Similar to CQBC, CIWAL adopts the query strategy from IWAL and converts the model selection strategy to be contextual. For model selection, we incorporate Exp4's arm recommendation strategy based on the side-information advice matrix and each classifier's historical performance according to queried instances. We compute the element-wise product of the two vectors as the model selection vector of CIWAL and normalize it as a weighted vector. Finally, CIWAL recommends the classifier with the highest weight.

\paragraph{Oracle:} 
Among all the given policies, oracle represents the best single policy that achieves the minimum cumulative loss, and it has the same query strategy as \algname.

\subsection{Details on policies and classifiers
}\label{app:policy_classifier}

We constructed different expert-model configurations to reflect the cases in real-world applications\footnote{To list a few other scenarios beyond the ones used in the paper: In \emph{healthcare}, models could be the treatments, experts could be the doctors and the context could be the condition of a patient. For any patient (context), doctors (experts) will have their own advice on the treatment (model) recommendation for this patient based on their past experience. In the \emph{finance} domain, models could be trading strategies, experts could be portfolio managers, and the context could be the stock/equity. Some trading strategies (models) might work well for the information technology sector, and some other models might work well for the energy sector, so depending on the sector of stock (context), different portfolio managers (experts) might have their own advice on different trading strategies (models) based their past trading experience. 
}. This section lists the collection of policies and models used in our experiments.

\paragraph{CIFAR10:}  We have constructed 80 diversified classifiers based on VGG \citep{simonyan2014very}, ResNet \citep{he2016deep}, DenseNet \citep{huang2017densely}, GoogLeNet \citep{szegedy2015going}. We have also used EfficientNet \citep{tan2019efficientnet}, MobileNets \citep{howard2017mobilenets}, RegNet \citep{schneider2017regnet}, and ResNet to construct 85 diversified policies.

\paragraph{DRIFT:}  We have constructed ten classifiers using Decision Tree \citep{quinlan1986induction}, SVM \citep{cortes1995support}, AdaBoost \citep{freund1999short}, Logistic Regression \citep{cramer2002origins}, KNN \citep{cover1967nearest} models. We have also created 8 diversified policies with multilayer perceptron (MLP) models of different layer configurations: (128, 30, 10); (128, 60, 30, 10); (128, 120, 30, 10); (128, 240, 120, 30, 10).

\paragraph{VERTEBRAL:} We have built six classifiers using Random Forest \citep{breiman2001random}, Gaussian Process \citep{rasmussen2003gaussian}, linear discriminant analysis \citep{fisher1938statistical}, Naive Bayes \citep{hand2001idiot} algorithms. We have constructed policies by using standard scikit-learn built-in models including Random Forest Classifier, Extra Trees Classifier \citep{geurts2006extremely}, Decision Tree Classifier, Radius Neighbors Classifier \citep{musavi1992training}, Ridge Classifier \citep{rifkin2007notes} and K-Nearest-Neighbor classifiers.

\paragraph{HIV:} We have used graph convolutional networks (GCN) \citep{kipf2016semi}, Graph Attention Networks (GAT) \citep{velivckovic2017graph}, AttentiveFP \citep{xiong2019pushing}, and Random Forest to construct 4 classifiers. We have also used various feature representations of molecules such as MACCS key \citep{durant2002reoptimization}, ECFP2, ECFP4, and ECFP6 \citep{rogers2010extended} molecular fingerprints to build 6 MLP-based policies, respectively.  

\paragraph{CovType:} We have built 6 classifiers using Random Forest, Gaussian Process, linear discriminant analysis, Naive Bayes algorithms. We have constructed 17 policies by using standard scikit-learn built-in models including Random Forest Classifier, Extra Trees Classifier, Decision Tree Classifier, Radius Neighbors Classifier, Ridge Classifier and K-Nearest-Neighbor classifiers.

\subsection{Implementation details} 
We build our evaluation pipeline on top of prior work \citep{karimi2021online} around the four benchmark datasets. Specifically,
\begin{itemize}\denselist
    \item \textit{Context} $\instance_t$ is the raw context of the data (e.g., the 32x32 image for CIFAR10).
    \item \textit{Predictions} $\pdB_t$ contain the predicted label vector of all the classifiers' predictions according to the online context $x_t$. \item \textit{Oracle} contains the true label $y_t$ of $\instance_t$. 
    \item \textit{Advice matrix} contains %
    all policies' probability distribution $\lambda$ over all the classifiers on context $x_t$. 
\end{itemize}
To adapt to an online setting, we sequentially draw random $T$ i.i.d. instances $\instance_{1:T}$ from the test pool and define it as a realization. For a fair comparison, all algorithms receive data instances in the same order within the same realization. %

\begin{algorithm}
   \caption{Regularized policy $\overline{\policy}\paren{\instance_t}$}
   \label{alg:policy_regularized}
\begin{algorithmic}[1]
    \State {\bfseries Input:} context $\instance_t$, Models $\Models$, policy $\policy \in \Policies^*$
           \State $\eta=\sum_{j=1}^{|\Models|}{\paren{\bracket{\policy\paren{x_t}}_j-\frac{1}{|\Models|}}^2}$
            \State \Return $\frac{\policy_i\paren{x_t}+\eta}{1+|\Models|\cdot\eta}$ 
\end{algorithmic}
\end{algorithm}
\subsection{Regularized policy}\label{app:regularized_policy}
As discussed in adversarial section, we wish to ensure that the probability a policy selecting any model is bounded away from 0 so that the regret bound in \thmref{thm:adversarial-regret-bound} is non vacuous. In our experiments, we achieve this goal by applying a regularized policy $\overline{\policy}$ as shown in Algorithm~\ref{alg:policy_regularized}.

\subsection{Summary of datasets and models}\label{app:dataset_table}
We summarize the attributes of datasets, the models, and the model selection policies as follows.
\begin{table}[ht!]
\centering
\scalebox{1}{
\begin{tabular}{l l l l l l l}
\toprule
\textbf{dataset} & \textbf{classification}& \textbf{total instances} & \textbf{test set} & \textbf{stream size}  & \textbf{classifier} & \textbf{policy} \\
\midrule
CIFAR10 & 10  & 60000 & 10000 &10000 & 80 & 85 \\
DRIFT & 6  & 13910 & 3060 & 3000 & 10 & 11 \\
VERTEBRAL & 3 & 310 & 127 & 80 & 6 & 17\\
HIV & 2 & 40000 & 4113 & 4000 & 4 & 20 \\
\fix{CovType} & 55 & 580000 & 100000 & 100000 & 6 & 17 \\
\bottomrule\\
\end{tabular}
}
\caption{Attributes of benchmark datasets}
\label{table:dataset_attribute}
\end{table}

\subsection{Hyperparameters}\label{app:Hyperparameters}

We performed our experiments on a Linux server with 80 Intel(R) Xeon(R) Gold 6148 CPU @ 2.40GHz and total 528 Gigabyte memory.

By considering the resource of server, 
We set 100 realizations and 3000 stream-size for DRIFT, 20 realizations and 10000 stream-size for CIFAR10, 200 realizations and 4000 stream size for HIV, 300 realization and 80 stream-size for VERTEBRAL. In each realization, we randomly selected stream-size aligned data from testing-set and make it as online streaming data which is the input of each algorithm. Thus, we got independent result for each realization. 

A small realization number would increase the variance of the results due to the randomness of stream order. A large realization number would make the result be more stable but at the cost of increasing computational cost (time, memory, etc.). We chose the realization number by balancing both aspects.

    \section{Proofs for the Stochastic Setting}\label{app:stochastic_results}

In this section, we focus on the stochastic setting. We first prove the regret bound presented in \thmref{thm:stochasticReg} and then prove the query complexity presented in \thmref{thm:stochastic-query-complexity} for \algoref{alg:CAMS}.

\subsection{Proof of \thmref{thm:stochasticReg}}\label{app:stochastic_results_regret}

Before providing the proof of \thmref{thm:stochasticReg}, we first introduce the following lemma.
\begin{lemma}\label{lem:optpolicyprob}
Fix $\tau \in \paren{0,1}$. Let $\PoliciesDist_{t,i^*}$ be the probability of the optimal policy $i^*$ maintained by \algoref{alg:CAMS} at $t$, \rebuttal{}{and let $b=p_{\min}\log_{c}{(1/p_{\min})}$, where $p_{\min} = \min_{s,i} \pi(\mathbf{x}_s)$ denotes the minimal model selection probability by any policy\footnote{We assume $p_{\min} > 0$ per the policy regularization criterion in Appendix C.3. (cf. Algorithm 1 on ``Regularized policy $\bar{\policy}(\mathbf{x}_t)$)''.}}. %
When $t \geq \paren{\frac{\ln{\frac{(|\Policies^*|\fix{-1})\tau}{1-\tau}}}{\sqrt{\ln{|\Policies^{*}|}}\paren{\Delta-\sqrt{\frac{2\rebuttal{}{b^2}}{t}\ln{\frac{2}{\delta}}}}}}^2$, with probability at least $1-\delta$, it holds that $q_{t,i^*}\geq \tau$.
\end{lemma}
\begin{proof}[Proof of \lemref{lem:optpolicyprob}]
W.l.o.g, we assume $\mu_1 \leq \mu_2 \leq \dots \mu_{n+k}$. 
Recall that we define ${\Delta}=\min_{i\neq i^*}{\Delta}_i=\mu_{2}-\mu_{1}=\frac{\expct{{\widetilde{\Loss}_{t,2}-\widetilde{\Loss}_{t,1}}}}{t}$, and \rebuttal{}{$\pi_1$ is the policy with the minimal expected loss}.

\rebuttal{}{Define
\begin{align}\label{eq:empirical_subopt}
    \delta_t &\triangleq \tilde{\ell}_{t-1,i'} - \tilde{\ell}_{t-1,1}.
\end{align}
where $i' \triangleq \arg\min_{i\neq 1} \tilde{L}_{t-1,i}$ denotes the index of the best empirical policy up to $t-1$ other than $\pi_1$. Therefore for $i\geq 2$, it holds that $$\widetilde{\Loss}_{t-1,i'} - \widetilde{\Loss}_{t-1,i} = \sum_{s=1}^{t-1}\delta_s \leq 0.$$
}

We have $\PoliciesDist_{t,i^*}=\PoliciesDist_{t,1}=\frac{\exp\paren{-\eta_t\widetilde{\Loss}_{t-1,1}}}{\sum_{i=1}^{|\Policies^*|} {\exp\paren{-\eta_t\widetilde{\Loss}_{t-1,i}}}}$ as the weight of optimal expert at round $t$. Therefore
\begin{flalign}
{\PoliciesDist_{t,i^*}}={\PoliciesDist_{t,1}}
&={\frac{\exp\paren{-\eta_t\widetilde{\Loss}_{t-1,1}}}{\sum_{i=1}^{|\Policies^*|} {\exp\paren{-\eta_t\widetilde{\Loss}_{t-1,i}}}} } \nonumber \\
&\stackrel{\paren{a}}{=}{\frac{\exp\paren{-\eta_t{\widetilde{\Loss}_{t-1,1}}+\eta_t{\widetilde{\Loss}_{t-1,\rebuttal{2}{i'}}}}}{\sum_{i=1}^{|\Policies^*|}{\exp\paren{-\eta_t{\widetilde{\Loss}_{t-1,i}}+\eta_t{\widetilde{\Loss}_{t-1,\rebuttal{2}{i'}}}}}} \nonumber} \\
&\stackrel{\paren{b}}{=}\frac{\exp\paren{\eta_t \sum_{s=1}^{t}\delta_s}}{\exp\paren{\eta_t  \sum_{s=1}^{t}\delta_s}+\sum_{i=2}^{|\Policies^*|} {\exp\paren{-\eta_t{\widetilde{\Loss}_{t-1,i}}+\eta_t{\widetilde{\Loss}_{t-1,\rebuttal{2}{i'}}}}}} \nonumber \\
&\geq \frac{\exp\paren{\eta_t \sum_{s=1}^{t}\delta_{s}}}{\exp\paren{\eta_t \sum_{s=1}^{t}\delta_{s}}+|\Policies^*|\rebuttal{}{-1}}\nonumber \\
\end{flalign}

where step \paren{a} is by dividing the cumulative loss of sub-optimal policy $\policy_{\rebuttal{2}{i'}}$ and step (b) is by \rebuttal{defining $\delta_t \triangleq \tilde{\loss}_{t-1,2}-\tilde{\loss}_{t-1,1}$.}{the definition of $\delta_t$ in Equation~\eqref{eq:empirical_subopt}.}

Let $\tau \in \paren{0,1}$, such that
${\PoliciesDist_{t,i^*}}\geq \frac{\exp\paren{\eta_t \sum_{s=1}^{t}\delta_{s}}}{\exp\paren{\eta_t \sum_{s=1}^{t}\delta_{s}}+|\Policies^*|\rebuttal{}{-1}} \geq \tau$. 
Plugging in $\eta_t=\sqrt{\frac{\ln{|\Policies^*|}}{t}}$ and define $\overline{\delta_t} = \frac{1}{t}\sum_{s=1}^{t}\delta_{s}$, we get
\begin{flalign*}
\frac{\exp\paren{\sqrt{\ln{|\Policies^*|}}\sqrt{t} \cdot \overline{\delta_t} }}{\exp\paren{\sqrt{\ln{|\Policies^*|}}\sqrt{t} \cdot \overline{\delta_t}}+|\Policies^*|\rebuttal{}{-1}} &\stackrel{}{\geq} \tau\\
\end{flalign*}
Therefore, we obtain %
$\exp\paren{\sqrt{\ln{|\Policies^*|}}\sqrt{t} \cdot \overline{\delta_t}} \geq{\frac{(|\Policies^*|\rebuttal{}{-1})\tau}{1-\tau}}$. Rearranging the terms, we get
\begin{flalign*}
t &\geq \paren{\frac{\ln{\frac{(|\Policies^*|\rebuttal{}{-1})\tau}{1-\tau}}}{\sqrt{\ln{|\Policies^{*}|}}\cdot \overline{\delta_t}}}^2%
\end{flalign*}
\rebuttal{Now by Hoeffding's inequality, we know  %
$\Pr{|\overline{\delta}_t-\Delta|\geq \epsilon}\leq 2e^{-\frac{{t}\epsilon^{2}}{2}}$.}{
Next, we seek a high probability upper bound on $\overline{\delta}_t$. Denote $\Delta_i \triangleq \mu_i - \mu_1$ for $i \in {1, \dots, |\Policies^*|}$. 
We know 
\begin{align}\label{eq:hoeffding-stochastic}
P(\overline{\delta}_t \leq \Delta_2 - \epsilon) 
&\stackrel{\text{(a)}}{\leq} P(\overline{\delta}_t \leq \Delta_{i'} - \epsilon) 
\stackrel{}{=} P(\frac{1}{t}\sum_{s=1}^{t}\delta_{s} - \Delta_{i'} \leq -\epsilon) 
\stackrel{\text{(b)}}{\leq} e^{-\frac{t \epsilon^{2}}{2b^2}}
\end{align} 
Here, step (\ref{eq:hoeffding-stochastic}a) is by the fact that $\Delta_{2} = \min_{i\neq 1} \Delta_i \leq \Delta_{i'}$, 
and step (\ref{eq:hoeffding-stochastic}b) is by Hoeffding's inequality where $b$ denotes the upper bound on $|\delta_{s}|$. %
Further note that %
\begin{align*}
    {\delta}_{s+1} = \tilde{\ell}_{s,i'} - \tilde{\ell}_{s,1} 
    &=  \frac{U_s}{z_s} \langle \pi_{i'}
(\mathbf{x}_s) - \pi_1(\mathbf{x}_s), %
    \mathbb{I}\curlybracket{\pdB_{s}\neq \clabel_s}
\rangle \\ 
& \leq \frac{\langle \pi_{i'}(\mathbf{x}_s), \mathbb{I}\curlybracket{\pdB_{s}\neq \clabel_s} \rangle}{z_s}
\\ 
&\stackrel{\text{Eq.~\eqref{eq:query-probability}}}{\leq} U_s \frac{\langle \pi_{i'}(\mathbf{x}_s), \mathbb{I}\curlybracket{\pdB_{s}\neq \clabel_s} \rangle}{\frac{1}{\numClabel} \sum_{\clabel \in \domClabel} \langle \ModelsDistB_s, \mathbb{I}\curlybracket{\pdB_s \neq \clabel} \rangle \log_{\numClabel}{\frac{1}{\langle \ModelsDistB_s, \mathbb{I}\curlybracket{\pdB_s \neq \clabel} \rangle}}}
\end{align*}
Given $p_{\min} = \min_{s,i} \pi(\mathbf{x}_s)$, %
we obtain $\delta_{s+1} \leq \frac{1}{p_{\min} \log_c (1/p_{\min})}$ and similarly, $\delta_{s+1} \geq - \frac{\langle \pi_{1}(\mathbf{x}_s), \mathbb{I}\curlybracket{\pdB_{s}\neq \clabel_s} \rangle}{z_s} \geq -\frac{1}{p_{\min} \log_c (1/p_{\min})}$. We hence conclude that $|\delta_{s+1}| \leq b$. %

}

Let $2e^{-\frac{t\epsilon^{2}}{2\rebuttal{}{b^2}}}=\delta$. Therefore, when $t \geq \paren{\frac{\ln{\frac{(|\Policies^*|\fix{-1})\tau}{1-\tau}}}{\sqrt{\ln{|\Policies^{*}|}}\paren{\Delta-\epsilon}}}^2 = \paren{\frac{\ln{\frac{(|\Policies^*|\fix{-1})\tau}{1-\tau}}}{\sqrt{\ln{|\Policies^{*}|}}\paren{\Delta-\sqrt{\frac{2\rebuttal{}{b^2}}{t}\ln{\frac{2}{\delta}}}}}}^2$, it holds that ${q_{t,i^*}}\geq \tau$ with probability at least $1-\delta$. %

\end{proof}

\begin{lemma}\label{lem:optpolicyprobdominate}
At round t, when $t \geq  \paren{\frac{\ln{\frac{|\Policies^*|\rebuttal{}{-1}}{\gamma}}+\sqrt{\ln{|\Policies^*|}\cdot{2\rebuttal{}{b^2}\ln{\frac{2}{\delta}}}}}{\sqrt{\ln{|\Policies^*|}}\Delta}}^2$, it holds that the arm chosen by the best policy $i^*$ will be the arm chosen by \algoref{alg:CAMS} with probability at least $1-\delta$. That is, ${\argmax\curlybracket{ {\sum_{i \in [|\Policies^*|]} \PoliciesDist_{t,i} \policy_i(\instance_t)}}}={\argmax\curlybracket{\policy_{i^*}(\instance_t)}}$.
\end{lemma}

\begin{proof}[Proof of \lemref{lem:optpolicyprobdominate}]

At round t, for \algoref{alg:CAMS}, we have loss $\sum_{j=1}^{\ModelsNum} \mathbb{I}\curlybracket{j=\argmax\curlybracket{ \sum_{i \in [|\Policies^*|]} \PoliciesDist_{t,i} \policy_i(\instance_t)}} \widehat{\loss}_{t,j}$. Let ${\PoliciesDist_{t,i^*}} \geq \tau$. At round $t$, the best policy $i^*$'s top weight arm $j_{t,i^*}$'s probability 
$\max\curlybracket{\policy_{i^*}{(\instance_t)}}$ is at least $\frac{1}{\ModelsNum}$. The second rank probability of $\policy_{i^*}{(\instance_t)}$ is $\max_j{[\policy_{i^*}\paren{\instance_t}]_{j\neq \maxind{\policy_{i^*}{(\instance_t)}}}}$.
Let us define

\begin{equation}\label{eq:min_gap_model}
\gamma := \min_{\instance_t}\left\{\max_{\ModelsDist_j \in \ModelsDistB_{i^*}^t} \ModelsDist_j  -\max_{\ModelsDist_j \in \ModelsDistB_{i^*}^t,j\neq\maxind{\ModelsDistB_{i^*}^t}} \ModelsDist_j  \right\} 
\end{equation}
\[
=\max\curlybracket{\policy_{i^*}\paren{\instance_t}}-\max_j\curlybracket{[\policy_{i^*}\paren{\instance_t}]_{j\neq\maxind{\policy_{i^*}\paren{\instance_t}}}},
\]

as the minimal gap in model distribution space of best policy. The arm recommended by the best policy $i^*$ of \algname will dominate \algname's selection, when we have

\begin{flalign}
{\PoliciesDist_{t,i^*}}\cdot \max\curlybracket{\policy_{i^*}{(\instance_t)}} &\geq \paren{1-{\PoliciesDist_{t,i^*}}} +{\PoliciesDist_{t,i^*}}\paren{\max_j{[\policy_{i^*}\paren{\instance_t}]_{j\neq \maxind{\policy_{i^*}{(\instance_t)}}}}} 
\end{flalign}
Rearranging the terms, and by 
\begin{flalign*}
{\PoliciesDist_{t,i^*}}\cdot\gamma& \stackrel{\eqnref{eq:min_gap_model}}{=} {\PoliciesDist_{t,i^*}}\paren{\max\curlybracket{\policy_{i^*}{(\instance_t)}} -\max_j{[\policy_{i^*}\paren{\instance_t}]_{j\neq \maxind{\policy_{i^*}{(\instance_t)}}}}} \geq \paren{1-{\PoliciesDist_{t,i^*}}} \\
\end{flalign*}
Therefore, we get $\tau\cdot \paren{\gamma  } \stackrel{}{\geq} \paren{1-\tau}$, and thus $\tau \geq \frac{1}{\gamma  +1}$.

Set $\tau \geq \frac{1}{\gamma  +1}$. By \lemref{lem:optpolicyprob}, we get 
\begin{flalign*}
t &\geq \paren{\frac{\ln{\frac{|\Policies^*\fix{-1}|\tau}{1-\tau}}}{\sqrt{\ln{|\Policies^{*}|}}\paren{\Delta-\epsilon}}}^2 \\
&\geq \paren{\frac{\ln{\paren{\frac{|\Policies^*|\fix{-1}}{\gamma }}}}{\sqrt{\ln{|\Policies^*|}}\paren{\Delta-\epsilon}}}^2\\
&\stackrel{\paren{c}}{\geq} \paren{\frac{\ln{\frac{|\Policies^*|\fix{-1}}{\gamma}}}{\sqrt{\ln{|\Policies^{*}|}}\Delta- \sqrt{\ln{|\Policies^*|}\cdot{\frac{2\rebuttal{}{b^2}}{t}}\ln{\frac{2}{\delta}}}}}^2
\end{flalign*}
where the last step is by applying $2e^{-\frac{t\epsilon^{2}}{2\rebuttal{}{b^2}}}=\delta$, thus, $\epsilon=\sqrt{\frac{2\rebuttal{}{b^2}}{t}\ln{\frac{2}{\delta}}}$. 
Dividing both sides by $t$ 
\begin{flalign*}
1 &\stackrel{\paren{d}}{\geq} \paren{\frac{\ln{\frac{|\Policies^*|\fix{-1}}{
\gamma
}}}{\sqrt{\ln{|\Policies^{*}|}\cdot{t}}\Delta- \sqrt{\ln{|\Policies^*|}\cdot{{2\rebuttal{}{b^2}}}\ln{\frac{2}{\delta}}}}}^2\\
\ln{\frac{|\Policies^*|\fix{-1}}{
\gamma
}} &\leq \sqrt{t}\sqrt{\ln\paren{|\Policies^*|}}\Delta- \sqrt{\ln\paren{|\Policies^*|}\cdot{2\rebuttal{}{b^2}\ln{\frac{2}{\delta}}}}\\
t &\geq \paren{\frac{\ln{\frac{|\Policies^*|\fix{-1}}{
\gamma
}}+\sqrt{\ln{|\Policies^*|}\cdot{2\rebuttal{}{b^2}\ln{\frac{2}{\delta}}}}}{\sqrt{\ln{|\Policies^*|}}\Delta}}^2.
\end{flalign*}
So, 
when $t \geq \paren{\frac{\ln{\frac{|\Policies^*|\fix{-1}}{
\gamma
}}+\sqrt{\ln{|\Policies^*|}\cdot{2\rebuttal{}{b^2}\ln{\frac{2}{\delta}}}}}{\sqrt{\ln{|\Policies^*|}}\Delta}}^2$, it holds that ${\argmax\curlybracket{ {\sum_{i \in [|\Policies^*|]} \PoliciesDist_{t,i} \policy_i(\instance_t)}}}={\argmax\curlybracket{\policy_{i^*}(\instance_t)}}$.
\end{proof}

\begin{proof}[Proof of \thmref{thm:stochasticReg}]

Therefore, with probability at least $1-\delta$ 
, we get constant regret $\paren{\frac{\ln{\frac{|\Policies^*|\fix{-1}}{
\gamma
}}+\sqrt{\ln{|\Policies^*|}\cdot{2\rebuttal{}{b^2}\ln{\frac{2}{\delta}}}}}{\sqrt{\ln{|\Policies^*|}}\Delta}}^2$. 

Furthermore, with probability at most $\delta$, the regret is upper bounded by $T$.  Thus, we have
\begin{align*}
\overline{\mathcal{R}}\paren{T}&\leq \paren{1-\delta}\paren{\frac{\ln{\frac{|\Policies^*|\fix{-1}}{\gamma}}+\sqrt{\ln{|\Policies^*|}\cdot{2\rebuttal{}{b^2}\ln{\frac{2}{\delta}}}}}{\sqrt{\ln{|\Policies^*|}}\Delta}}^2+\delta T\\
&\stackrel{\paren{a}}{\leq} \paren{1-\frac{1}{T}}\paren{\frac{\ln{\frac{|\Policies^*|\fix{-1}}{\gamma}}+\rebuttal{}{b}\sqrt{\ln{|\Policies^*|}\cdot{\paren{2\ln{T}+2\ln{2}}}}}{\sqrt{\ln{|\Policies^*|}}\Delta}}^2+1\\
&= O\paren{\frac{\rebuttal{}{b}\ln{T}}{\Delta^2}+\paren{\frac{\ln{\frac{|\Policies^*|\fix{-1}}{\gamma}}}{\sqrt{\ln{|\Policies^*|}}\Delta}}^2},
\end{align*}
where step (a) by setting $\delta=\frac{1}{T}$, and where $\gamma$ in \eqnref{eq:min_gap_model} is the min gap. 
\end{proof}

\subsection{Proof of \thmref{thm:stochastic-query-complexity}}\label{app:stochastic_results_query_complexity}

In this section, we analyze the query complexity of \algname in the stochastic setting, where we take a similar approach as proposed by \citet{karimi2021online} for the context-free model selection problem. Our main idea is to derive from query indicator $\QueryIndicator_t$ and query probability $\queryProb_t$. We first used %
\lemref{lem:query_complexity}
to bound the expected number of queries $\sum_{t=1}^T\QueryIndicator_t$ by the sum of query probability as $\sum_{t=1}^T\delta_0^t+\sum_{t=1}^T\entropyF\paren{\pdB_t,\ModelsDistB_t}$. %
Then we used \lemref{lem:eta_query_ub} to bound the first item (which corresponds to the lower bound of query probability over $T$ rounds) and applied \lemref{lem:entrophy_query_ub} %
to bound the second %
term (which characterizes the model disagreement). %
Finally, we combined the upper bounds on the two parts %
to reach the desired result. 

\begin{lemma}\label{lem:query_complexity}
The query complexity of \algoref{alg:CAMS} is upper bounded by
\begin{equation}
\expct{\sum_{t=1}^T{\paren{\frac{1}{\sqrt{t}}+\frac{\sum_{y\in \domClabel} \langle \ModelsDistB_t,\lossB_t^y \rangle \log_{|\domClabel|}{\frac{1}{\langle \ModelsDistB_t,\lossB_t^y \rangle}}}{|\domClabel|}}}}.
\end{equation}
\end{lemma}
\begin{proof}%
Now we have model disagreement defined in  \eqnref{eq:instance-info}, the query probability defined in  \eqnref{eq:query-probability}, and the query indicator $\QueryIndicator$. %
Let us assume, at each round, we have query probability $\queryProb_t > 0$, which indicates %
we will not process the instance that all the models' prediction are the same.

At round $t$, from query probability \eqnref{eq:query-probability}, we have
\begin{flalign*}
\queryProb_t&= \max{\curlybracket{\delta_0^t,\entropyF\paren{\pdB_t,\ModelsDistB_t}}}\\
&{\leq} \delta_0^t+\entropyF\paren{\pdB_t,\ModelsDistB_t},
\end{flalign*}
where the inequality is by applying that $\forall{A,B\geq{0}}, \max\{A,B\}\leq A+B$.

Thus, in total round $T$, we could get the following equation as the cumulative query cost,
\begin{equation}
\expct{\sum_{t=1}^TU_t} {\leq}
\expct{\sum_{t=1}^T{\paren{
\frac{1}{\sqrt{t}}+\frac{\sum_{y\in \domClabel} \langle \ModelsDistB_t,\lossB_t^y \rangle \log_{|\domClabel|}{\frac{1}{\langle \ModelsDistB_t,\lossB_t^y \rangle}}}{|\domClabel|}}}},
\end{equation}
where the inequality is by inputting $\delta_0^t=\qlb$ and \eqnref{eq:instance-info}.
\end{proof}

\begin{lemma}\label{lem:eta_query_ub}
$\sum_{t=1}^T\frac{1}{\sqrt{t}} \leq 2\sqrt{T}.$
\end{lemma}
\begin{proof}%
We can bound the LHS as follows:
\begin{flalign*}
\sum_{t=1}^T\frac{1}{\sqrt{t}}&=\sum_{t=1}^{\lfloor \sqrt{T}\rfloor}\frac{1}{\sqrt{t}}+\sum_{t=\lfloor \sqrt{T}\rfloor +1}^T \frac{1}{\sqrt{t}}\\
&\leq \sqrt{T}+\sum_{t=\lfloor \sqrt{T}\rfloor +1}^T \frac{1}{\sqrt{T}}\\
&=\sqrt{T}+\paren{T-\sqrt{T}}\frac{1}{\sqrt{T}}\\
&\leq 2\sqrt{T}.
\end{flalign*}
\end{proof}

\begin{lemma}\label{lem:entrophy_query_ub}
Denote the true label at round $t$ by $y_t$, and define $p_{t,y} := \sum_{j \in [k]} \mathbb{I}\curlybracket{\pd_{t,j} = y} \ModelsDist_j$. Further define $R_t := \sum_t 1-p_{t,y_t}$ as the expected cumulative loss of \algoref{alg:CAMS} at $t$. %
Then
\[
\sum_{t=1}^{T}\frac{\sum_{y\in \domClabel} \langle \ModelsDistB_t,\lossB_t^y \rangle \log_{|\domClabel|}{\frac{1}{\langle \ModelsDistB_t,\lossB_t^y \rangle}}}{|\domClabel|} 
\leq \frac{R_T\cdot \paren{\log_{|\domClabel|}\frac{T^2\paren{|\domClabel|-1}}{R_T^2}}}{|\domClabel|}.
\]
\end{lemma}

\begin{proof}[Proof of \lemref{lem:entrophy_query_ub}]
Suppose at round $t$, the true label is $y_t$. $\sum_{y\neq y_t}p_{t,y} =1-p_{t,y_t} = 1- \tuple{{\sum_{i \in |\Policies^*|} \PoliciesDist_{t,i} \policy_i(\instance_t)},\lossB_t}=r_t$,

\begin{flalign*}
\frac{\sum_{y\in \domClabel} \langle \ModelsDistB_t,\lossB_t^y \rangle \log_{|\domClabel|}{\frac{1}{\langle \ModelsDistB_t,\lossB_t^y \rangle}}}{|\domClabel|}
&= \frac{\paren{1-p_{t,y_t}} \log_{|\domClabel|}\frac{1}{ 1-p_{t,y_t}}}{|\domClabel|} +\frac{\sum_{y\neq y_t}\paren{1-p_{t,y}}\log_{|\domClabel|}\frac{1}{1-p_{t,y}}}{|\domClabel|}\\
&\stackrel{\paren{
a}}{\leq} \frac{\paren{1-p_{t,y_t}} \log_{|\domClabel|}\frac{1}{ 1 - p_{t,y_t}}}{|\domClabel|}
+\paren{|\domClabel|-1}\frac{\frac{\paren{1-p_{t,y_t}}}{|\domClabel|-1}\log_{|\domClabel|}\frac{|\domClabel|-1}{1-p_{t,y_t}}}{|\domClabel|}\\
&\leq \frac{\paren{1-p_{t,y_t}} \log_{|\domClabel|}\frac{1}{ 1-p_{t,y_t}}}{|\domClabel|} +\frac{\paren{1-p_{t,y_t}}\log_{|\domClabel|}\frac{|\domClabel|-1}{1-p_{t,y_t}}}{|\domClabel|}\\
&= \frac{\paren{1-p_{t,y_t}}\log_{|\domClabel|}\frac{|\domClabel|-1}{\paren{1-p_{t,y_t}}^2}}{|\domClabel|}\\
&\stackrel{\paren{b}}{\leq} \frac{r_{t}\log_{|\domClabel|}\frac{|\domClabel|-1}{r^2_{t}}}{|\domClabel|},
\end{flalign*}
where step \paren{a} is by applying Jensen's inequality and using ${1-p_{t,y}}=\frac{1-p_{t,y_t}}{|\domClabel|-1}$, and step \paren{b} is by replacing the expected loss $1-p_{t,y_t}$ by its short-hand notation $r_t$.

Recall that we define the expected cumulative loss as $R_T=\sum_{t=1}^T r_t$. Since when $r_t \in [0,1]$, $\frac{r_{t}\log_{|\domClabel|}\frac{|\domClabel|-1}{r^2_{t}}}{|\domClabel|}$ is concave, we get
\begin{equation}\label{eq:budget-eq}
\sum_{t=1}^{T}\frac{\sum_{y\in \domClabel} \langle \ModelsDistB_t,\lossB_t^y \rangle \log_{|\domClabel|}{\frac{1}{\langle \ModelsDistB_t,\lossB_t^y \rangle}}}{|\domClabel|} 
\leq \frac{T\paren{\frac{\sum r_t}{T}}\paren{\log_{|\domClabel|}\frac{|\domClabel|-1}{\frac{\sum r_t}{T}\frac{\sum r_t}{T}}}}{|\domClabel|}
=\frac{R_T\paren{\log_{|\domClabel|}\frac{T^2\paren{|\domClabel|-1}}{R_T^2}}}{|\domClabel|}.
\end{equation}

Since $R_T$ is the cumulative loss up to round $T$, $T$'s incremental rate is no less than $R_T$'s incremental rate. Thus, $R_T \leq T$ and 
$\frac{T_{t}}{R_{t}} \leq \frac{T_{t+1}}{R_{t+1}}$. So we get \eqnref{eq:budget-eq}. %
\end{proof}

Now we are ready to prove \thmref{thm:stochastic-query-complexity}. 
\begin{proof}[Proof of \thmref{thm:stochastic-query-complexity}]
From \lemref{lem:query_complexity}, we get the following equation as the cumulative query cost
\[
\expct{\sum_{t=1}^TU_t} \leq
\expct{\sum_{t=1}^T{\paren{\frac{1}{\sqrt{t}}+\frac{\sum_{y\in \domClabel} \langle \ModelsDistB_t,\lossB_t^y \rangle \log_{|\domClabel|}{\frac{1}{\langle \ModelsDistB_t,\lossB_t^y \rangle}}}{|\domClabel|}}}}.
\]

Let us assume the expected total loss of best policy is $T \mu_{i^*}$. From \thmref{thm:stochasticReg}, 
we get

\[
    \expct{R_T}=\expct{\sum_{t=1}^T r_t}\leq \fixremoved{\paren{\frac{\ln{\frac{|\Policies^*|\fix{-1}}{\gamma}}+\sqrt{\ln{|\Policies^*|}\cdot{2\fix{b^2}\ln{\frac{2}{\delta}}}}}{\sqrt{\ln{|\Policies^*|}}\Delta}}^2}+T \mu_{i^*}.
\]

Plugging this result into the query complexity bound given by \lemref{lem:eta_query_ub} and \lemref{lem:entrophy_query_ub}, 
we have

\begin{flalign*}
\expct{\sum_{t=1}^TU_t}&\leq 2\sqrt{T}+ \frac{\paren{\fixremoved{\paren{\frac{\ln{\frac{|\Policies^*|\fix{-1}}{\gamma}}+\sqrt{\ln{|\Policies^*|}\cdot{2\fix{b^2}\ln{\frac{2}{\delta}}}}}{\sqrt{\ln{|\Policies^*|}}\Delta}}^2}+T \mu_{i^*}}}{|\domClabel|}
{{\log_{|\domClabel|}\frac{T^2\paren{|\domClabel|-1}}{\paren{\fixremoved{\paren{\frac{\ln{\frac{|\Policies^*|\fix{-1}}{\gamma}}+\sqrt{\ln{|\Policies^*|}\cdot{2\fix{b^2}\ln{\frac{2}{\delta}}}}}{\sqrt{\ln{|\Policies^*|}}\Delta}}^2}+T \mu_{i^*}}^2}}}\\
&\leq  \frac{\paren{\fixremoved{\paren{\frac{\ln{\frac{|\Policies^*|\fix{-1}}{\gamma}}+\sqrt{\ln{|\Policies^*|}\cdot{2\fix{b^2}\ln{\frac{2}{\delta}}}}}{\sqrt{\ln{|\Policies^*|}}\Delta}}^2}+T \mu_{i^*}}\paren{\log_{|\domClabel|}{\paren{T|\domClabel|}}}}{|\domClabel|}\\
&=  {\frac{\paren{\fixremoved{\paren{\frac{\ln{\frac{|\Policies^*|\fix{-1}}{\gamma}}+\sqrt{\ln{|\Policies^*|}\cdot{2\fix{b^2}\ln{\frac{2}{\delta}}}}}{\sqrt{\ln{|\Policies^*|}}\Delta}}^2}+T \mu_{i^*}}{\ln{\paren{T}}}}{|\domClabel|\ln{|\domClabel|}}}\\
&\stackrel{\paren{a}}{=} \fixremoved{ {\frac{\paren{\fixremoved{\paren{\frac{\ln{\frac{|\Policies^*|\fix{-1}}{\gamma}}+\sqrt{\ln{|\Policies^*|}\cdot{2\fix{b^2}\ln{\frac{2}{\delta}}}}}{\sqrt{\ln{|\Policies^*|}}\Delta}}^2}+T \mu_{i^*}}{\ln{\paren{T}}}}{\numClabel\ln{\numClabel}}}},
\end{flalign*}

where $\gamma$ is defined as \eqnref{eq:min_gap_model} \fixremoved{
and step (a) by applying $\numClabel=|\domClabel|$}.
\end{proof}

    \section{Proofs for the Adversarial Setting}\label{app:adversarial_results}

In this section, we first prove the regret bound presented in \thmref{thm:adversarial-regret-bound} and then prove the query complexity bound presented in \thmref{thm:adversarial-query-complexity} for \algoref{alg:CAMS} in the adversarial setting. \lemref{lem:adversarial-inititial-bound} builds upon the proof of the hedge algorithm \citep{freund1997decision}, but with an \textit{adaptive} learning rate.  %

\subsection{Proof of \thmref{thm:adversarial-regret-bound}}\label{app:adversarial_regret_bound}

\begin{lemma}\label{lem:adversarial-inititial-bound} Consider the setting of \algoref{alg:CAMS}, Let us
define $\pw_{t,i}=\exp{\paren{-\eta_t\tilde{\Loss}_{t-1,i}}}~\forall \policyIndex \in |\Policies^*|$ as exponential cumulative loss of policy $i$, $\eta_t$ is the adaptive learning rate and $\PoliciesDistB_t$ is the probability distribution of policies, then
\begin{align*}
\log{\frac{\sum_{i\in [|\Policies^*|]} \pw_{T+1,i}}{\sum_{i\in [|\Policies^*|]} \pw_{1,i}}}
\leq -\sum_{t=1}^T\eta_t\sum_{i=1}^{|\Policies^*|}\PoliciesDist_{t,i}\widetilde{\loss}_{t,i}+\sum_{t=1}^T\frac{\eta_t^2}{2}\sum_{i=1}^{|\Policies^*|} \PoliciesDist_{t,i}\paren{\widetilde{\loss}_{t,i}}^2.
\end{align*}
\end{lemma}
\begin{proof}%
We first bound the following term
\begin{flalign*}
\frac{\sum_{i\in [|\Policies^*|]} \pw_{t+1,i}}{\sum_{i\in [|\Policies^*|]} \pw_{t,i}}&=\sum_{i=1}^{|\Policies^*|}{\frac{\pw_{t+1,i}}{\sum_{i\in [|\Policies^*|]} \pw_{t,i}}}\\ 
&=\sum_{i=1}^{|\Policies^*|}{{\PoliciesDist}_{t,i}\exp{\paren{-\eta_t\widetilde{\loss}_{t,i}}}}\\
&{\leq} \sum_{i=1}^{|\Policies^*|}\PoliciesDist_{t,i}\paren{1-\eta_t\widetilde{\loss}_{t,i}+\frac{\eta_t^2\paren{\widetilde{\loss}_{t,i}}^2}{2}}\\
&= 1 - \eta_t \sum_{i=1}^{|\Policies^*|}\PoliciesDist_{t,i}\widetilde{\loss}_{t,i}+\frac{\eta_t^2}{2}\sum_{i=1}^{|\Policies^*|}\PoliciesDist_{t,i}\paren{\widetilde{\loss}_{t,i}}^2,
\end{flalign*}

where the inequality is by applying that for $x\leq0$, we have $e^{x}\leq 1+x+\frac{x^2}{2}$.

By taking $\log$ on both side, we get
\begin{flalign*}
\log{\frac{\sum_{i\in [|\Policies^*|]} \pw_{t+1,i}}{\sum_{i\in [|\Policies^*|]} \pw_{t,i}}}
&\leq \log{\paren{1 - \eta_t \sum_{i=1}^{|\Policies^*|}\PoliciesDist_{t,i}\widetilde{\loss}_{t,i}+\frac{\eta_t^2}{2}\sum_{i=1}^{|\Policies^*|}\PoliciesDist_{t,i}\paren{\widetilde{\loss}_{t,i}}^2}}\\
&\stackrel{\paren{a}}{\leq}   -\eta_t\sum_{i=1}^{|\Policies^*|}\PoliciesDist_{t,i}\widetilde{\loss}_{t,i}+\frac{\eta_t^2}{2}\sum_{i=1}^{|\Policies^*|} \PoliciesDist_{t,i}\paren{\widetilde{\loss}_{t,i}}^2,
\end{flalign*}

where step \paren{a} is by applying that $\log\paren{1+x}\leq x$, when $x\geq -1$.

Now summing over $t=1:T$ yields:\\
\begin{flalign*}
\log{\frac{\sum_{i\in [|\Policies^*|]} \pw_{T+1,i}}{\sum_{i\in [|\Policies^*|]} \pw_{1,i}}}
&=\sum_{t=1}^T \log{\frac{\sum_{i\in [|\Policies^*|]} \pw_{t+1,i}}{\sum_{i\in [|\Policies^*|]} \pw_{t,i}}} \\
&\leq -\sum_{t=1}^T\eta_t\sum_{i=1}^{|\Policies^*|}\PoliciesDist_{t,i}\widetilde{\loss}_{t,i}+\sum_{t=1}^T\frac{\eta_t^2}{2}\sum_{i=1}^{|\Policies^*|} \PoliciesDist_{t,i}\paren{\widetilde{\loss}_{t,i}}^2.
\end{flalign*}
\end{proof}

\begin{lemma}\label{lem:adversarial_query_prob}
Consider the setting of \algoref{alg:CAMS}. 
Let $p_{t,y} = \sum_{j \in [k]} \mathbb{I}\curlybracket{\pd_{t,j} = y} \ModelsDist_j$. The query probability $\queryProb_t$ satisfies 
\begin{align*}
\queryProb_t \geq \frac{1}{|\domClabel|\ln{|\domClabel|}}\paren{p_{t,y_t}\paren{1-p_{t,y_t}}+p_{t,y}\paren{1-p_{t,y}}}, \forall{\fixremoved{y}\neq y_t}.
\end{align*}

\end{lemma}
\begin{proof}%
We first bound the query probability term
\begin{flalign*}
\queryProb_t &= \max{\{\delta_0^t,\entropyF\paren{\pdB_t,\ModelsDistB_t}\}}\\
&=\max\{\delta_0^t,\frac{1}{|\domClabel|}\sum_{y \in \domClabel}\langle \ModelsDistB_t, \lossB_t^{y} \rangle \log_{|\domClabel|}{\frac{1}{\langle \ModelsDistB_t, \lossB_t^{y} \rangle}}\}\\
&=\max\{\delta_0^t,\frac{1}{|\domClabel|}\sum_{y \in \domClabel}\paren{1-p_{t,y}}\cdot \ln{\frac{1}{1-p_{t,y}}}{\frac{1}{\ln{|\domClabel|}}}\}\\
&\stackrel{\paren{a}}{\geq} \max\{\delta_0^t,\frac{1}{|\domClabel|}\sum_{y \in \domClabel}\paren{1-p_{t,y}}\cdot 
p_{t,y}\cdot{\frac{1}{\ln{|\domClabel|}}}\}\\
&=\max\{\delta_0^t,\frac{1}{|\domClabel|\ln{|\domClabel|}}\sum_{y \in \domClabel}\paren{1-p_{t,y}}\cdot p_{t,y}\}\\
&\stackrel{\paren{b}}{\geq} \frac{1}{|\domClabel|\ln{|\domClabel|}}\paren{p_{t,y_t}\paren{1-p_{t,y_t}}+p_{t,y}\paren{1-p_{t,y}}}, \forall{y\neq y_t},
\end{flalign*}

where step \paren{a} is by applying $\ln{\paren{1+x}}\geq \frac{x}{1+x}$ for $x>-1$,\\
\[
\ln{\frac{1}{1-p_{t,y}}}=\ln{\paren{1+\frac{p_{t,y}}{1-p_{t,y}}}}\geq{\frac{\frac{p_{t,y}}{1-p_{t,y}}}{\frac{1}{1-p_{t,y}}}}=p_{t,y},
\]
and where step \paren{b} is by applying $\forall{a,b \in \mathbb{R}},\max\curlybracket{a,b}\geq{a}$.

\end{proof}

\begin{proof}[Proof of \thmref{thm:adversarial-regret-bound}] 

By applying \lemref{lem:adversarial-inititial-bound}, we got
\begin{flalign*}
&\log{\frac{\sum_{i\in [|\Policies^*|]} \pw_{T+1,i}}{\sum_{i\in [|\Policies^*|]} \pw_{1,i}}}
\leq -\sum_{t=1}^T\eta_t\sum_{i=1}^{|\Policies^*|}\PoliciesDist_{t,i}\widetilde{\loss}_{t,i}+\sum_{t=1}^T\frac{\eta_t^2}{2}\sum_{i=1}^{|\Policies^*|} \PoliciesDist_{t,i}\paren{\widetilde{\loss}_{t,i}}^2.
\end{flalign*}

For any policy $s$, we have a lower bound\\
\begin{flalign}
\log{\frac{\sum_{i\in [|\Policies^*|]} \pw_{T+1,i}}{\sum_{i\in [|\Policies^*|]} \pw_{1,i}}} %
&\geq \log{\frac{\pw_{T+1,s}}{\sum_{i\in [|\Policies^*|]} \pw_{1,i}}} \nonumber \\
&\stackrel{\paren{a}}{=} \log{\frac{\pw_{T+1,s}}{|\Policies^*|}} \nonumber \\
&= -\log{\paren{\PoliciesNum + \ModelsNum}}{-\eta_T\sum_{t=1}^T\widetilde{\loss}_{t,s}} \label{eq:app:thm:adv:lmm},
\end{flalign}
where step \paren{a} in \eqnref{eq:app:thm:adv:lmm} is by initializing $\widetilde{\LossB}_0=0$, $e^0=1$, and  $\sum_{i\in[|\Policies^*|]}\pwB_{1}={e^{\paren{-\eta_t\widetilde{\LossB}_{0}}}}=|\Policies^*|$.\\

Thus, we have
\begin{flalign*}
-\sum_{t=1}^T\eta_t\sum_{i=1}^{|\Policies^*|}\PoliciesDist_{t,i}\widetilde{\loss}_{t,i}+\sum_{t=1}^T\frac{\eta_t^2}{2}\sum_{i=1}^{|\Policies^*|} \PoliciesDist_{t,i}\paren{\widetilde{\loss}_{t,i}}^2 &\geq -\log{\paren{\PoliciesNum + \ModelsNum}}{-\eta_T\sum_{t=1}^T\widetilde{\loss}_{t,s}}\\
\sum_{t=1}^T\eta_t\sum_{i=1}^{|\Policies^*|}\PoliciesDist_{t,i}\widetilde{\loss}_{t,i}{-\eta_T\sum_{t=1}^T\widetilde{\loss}_{t,s}} &\leq \log{\paren{\PoliciesNum + \ModelsNum}}+\sum_{t=1}^T\frac{\eta_t^2}{2}\sum_{i=1}^{|\Policies^*|} \PoliciesDist_{t,i}\paren{\widetilde{\loss}_{t,i}}^2\\
\eta_T\sum_{t=1}^T\sum_{i=1}^{|\Policies^*|}\PoliciesDist_{t,i}\widetilde{\loss}_{t,i}{-\eta_T\sum_{t=1}^T\widetilde{\loss}_{t,s}} &\stackrel{\paren{b}}{\leq} \log{\paren{\PoliciesNum + \ModelsNum}}+\sum_{t=1}^T\frac{\eta_t^2}{2}\sum_{i=1}^{|\Policies^*|} \PoliciesDist_{t,i}\paren{\widetilde{\loss}_{t,i}}^2\\
\sum_{t=1}^T\sum_{i=1}^{|\Policies^*|} \PoliciesDist_{t,i}\widetilde{\loss}_{t,i} - \sum_{t=1}^T\widetilde{\loss}_{t,s}
&\stackrel{\paren{c}}{\leq}  \frac{\log{|\Policies^*|}}{\eta_T} +\frac{1}{\eta_T}\sum_{t=1}^T\frac{\eta_t^2}{2}\sum_{i=1}^{|\Policies^*|}\PoliciesDist_{t,i}\paren{\widetilde{\loss}_{t,i}}^2,\\
\end{flalign*}

where step \paren{b} is by applying
\begin{align*}
\eta_T\sum_{t=1}^T\sum_{i=1}^{|\Policies^*|}\PoliciesDist_{t,i}\widetilde{\loss}_{t,i}{-\eta_T\sum_{t=1}^T\widetilde{\loss}_{t,s}}
\leq \sum_{t=1}^T\eta_t\sum_{i=1}^{|\Policies^*|}\PoliciesDist_{t,i}\widetilde{\loss}_{t,i}{-\eta_T\sum_{t=1}^T\widetilde{\loss}_{t,s}},
\end{align*}
and step \paren{c} is by dividing $\eta_T$ on both side.

Because we have
\begin{flalign*}
\expctover{T}{\PoliciesDist_{t,i}\paren{\widetilde{\loss}_{t,i}}^2}%
&=\PoliciesDist_{t,i}\expctover{T}{\paren{\policy_{i}\paren{ \instance_t}\cdot\widehat{\lossB}_t}^2}\\
&=\PoliciesDist_{t,i}\paren{P\paren{\QueryIndicator_t=1}\paren{\policy_{i}\paren{ \instance_t}\cdot\frac{\lossB_t}{\queryProb_t}}^2+P\paren{\QueryIndicator_t=0}\cdot 0}\\
&=\PoliciesDist_{t,i}\paren{\queryProb_t\paren{\policy_{i}\paren{ \instance_t}\cdot\frac{\lossB_t}{\queryProb_t}}^2}\\
&= \frac{\PoliciesDist_{t,i}}{\queryProb_t}\paren{\policy_{i}\paren{ \instance_t}\cdot{\lossB_t}}^2\\
&\leq \frac{\PoliciesDist_{t,i}}{\queryProb_t}\policy_{i}\paren{ \instance_t}\cdot\lossB_t\\
&= \frac{\PoliciesDist_{t,i}}{\queryProb_t}\langle\policy_{i}\paren{ \instance_t},\lossB_t\rangle,
\end{flalign*}

it leads to 

\begin{flalign*}
\sum_{t=1}^T\sum_{i=1}^{|\Policies^*|} \PoliciesDist_{t,i}\widetilde{\loss}_{t,i} - \sum_{t=1}^T\widetilde{\loss}_{t,s} 
&\leq  \frac{\log{|\Policies^*|}}{\eta_T} +\frac{1}{\eta_T}\sum_{t=1}^T\frac{\eta_t^2}{2}\sum_{i=1}^{|\Policies^*|}\frac{\PoliciesDist_{t,i}}{\queryProb_t}\langle \policy_{i}\paren{
\instance_t},\lossB_t\rangle\\
&\stackrel{\paren{d}}{\leq} \frac{\log{|\Policies^*|}}{\eta_T} +\frac{1}{\eta_T}\sum_{t=1}^T\frac{\eta_t^2}{2}\frac{\langle\ModelsDistB_t,\lossB_t\rangle}{\queryProb_t},
\end{flalign*}

where step \paren{d} is by applying $\sum_{i=1}^{|\Policies^*|}{\PoliciesDist_{t,i}}\langle \policy_{i}\paren{
\instance_t},\lossB_t\rangle=\langle\ModelsDistB_t,\lossB_t\rangle$.

So we have,

\begin{flalign*}
\sum_{t=1}^T\sum_{i=1}^{|\Policies^*|} \PoliciesDist_{t,i}\widetilde{\loss}_{t,i} - \sum_{t=1}^T\widetilde{\loss}_{t,s} 
&\leq  \frac{\log{|\Policies^*|}}{\eta_T} +\frac{1}{\eta_T}\sum_{t=1}^T\frac{\eta_t^2}{2}\frac{\langle\ModelsDistB_t,\lossB_t\rangle}{\queryProb_t}\\
&\stackrel{\paren{e}}{\leq} \frac{\log{|\Policies^*|}}{\eta_T} +\frac{1}{\eta_T}\sum_{t=1}^T\frac{\eta_t^2}{2}\frac{1-p_{t,y_t}}{\queryProb_t}\\
&\stackrel{\paren{f}}{\leq} \frac{\log{|\Policies^*|}}{\eta_T}+\frac{1}{\eta_T}\sum_{t=1}^T\frac{\eta^2_t}{2}\frac{1-p_{t,y_t}}{{\domClabel_0}\paren{\paren{1-p_{t,y_t}}p_{t,y_t}+\paren{1-p_{t,y}}p_{t,y}}}\\
&\leq \frac{\log{|\Policies^*|}}{\eta_T}+\frac{1}{\eta_T}\sum_{t=1}^T\frac{\eta^2_t}{2}\frac{1}{{\domClabel_0}\paren{p_{t,y_t}+ \frac{1-p_{t,y}}{1-p_{t,y_t}}p_{t,y}}},
\end{flalign*}

where step \paren{e} is by using $\langle\ModelsDistB_t,\lossB_t\rangle = 1-p_{t,y_t}$ and
step \paren{f} by using \lemref{lem:adversarial_query_prob} and get lower bound of $\queryProb_t$ as $\frac{1}{|\domClabel|\ln{|\domClabel|}}\paren{p_{t,y_t}\paren{1-p_{t,y_t}}+p_{t,y}\paren{1-p_{t,y}}}$ and applying $\frac{1}{|\domClabel|\ln{|\domClabel|}}=\domClabel_0$.

If $p_{t,y_t}\geq \frac{1}{|\domClabel|}$,
\begin{flalign*}
{p_{t,y_t}+\frac{1-p_{t,y}}{1-p_{t,y_t}}p_{t,y}} \geq
\frac{1}{|\domClabel|}.
\end{flalign*}

If $p_{t,y_t}< \frac{1}{|\domClabel|}$, $\exists y,p_{t,y}\rightarrow1$, %
$\delta_1^t=1-\max_{y,\tau\in[t]}p_{\tau,y}$. Let $p_{t,\hat{y}}=\max_y p_{t,y}$. Thus, we have
$\ModelsDist_{\hat{y}}>\frac{1}{|\domClabel|}$ and

\[
p_{t,y_t}+\frac{1-p_{t,y}}{1-p_{t,y_t}}p_{t,y} \geq p_{t,y_t}+\ModelsDist_{\hat{y}}\frac{\delta_1^t}{1-p_{t,y_t}} \geq 0 + \frac{1}{|\domClabel|}\frac{\delta^t_1}{1}=\frac{\delta_1^t}{|\domClabel|}.
\]

Therefore 
\[
    \max\{p_{t,y_t}+p_{t,y}\frac{1-p_{t,y}}{1-p_{t,y_t}}\}=\left\{
    \begin{array}{lr}
    {\frac{1}{|\domClabel|}} &
    \textrm{ if } p_{t,y_t}  \geq \frac{1}{|\domClabel|} , \\ 
    \frac{\delta_1^t}{|\domClabel|} &  
    \textrm{ if } p_{t,y_t} < \frac{1}{|\domClabel|}.
    \end{array}\right.
\]

So we have
\begin{flalign*}
\sum_{t=1}^T\sum_{i=1}^{|\Policies^*|} \PoliciesDist_{t,i}\widetilde{\loss}_{t,i} - \sum_{t=1}^T\widetilde{\loss}_{t,s} &\leq  \frac{\log{|\Policies^*|}}{\eta_T}+\frac{1}{\eta_T}\sum_{t=1}^T\frac{\eta^2_t}{2}\frac{{1}}{{\domClabel_0}\paren{p_{t,y_t}+ \frac{1-\ModelsDistB_{y}}{1-p_{t,y_t}}p_{t,y}}}\\
&\stackrel{\paren{g}}{\leq} \frac{\log{|\Policies^*|}}{\eta_T}+\frac{1}{\eta_T}\sum_{t=1}^T\frac{\eta^2_t}{2}\frac{1}{\max\{\domClabel_0\frac{\delta_1^t}{|\domClabel|},\delta_0^t\}}\\
&=\frac{\log{|\Policies^*|}}{\eta_T}+\frac{1}{\eta_T}\sum_{t=1}^T\frac{\eta^2_t}{2}\frac{|\domClabel|^2\ln{|\domClabel|}}{\max\{\delta_1^t,{\delta_0^t|\domClabel|^2\ln{|\domClabel|}}\}}\\
& \stackrel{\paren{h}}{\leq} \frac{\log{|\Policies^*|}}{\eta_T}+ \frac{1}{\eta_T}\sum_{t=1}^{T}\frac{\eta_t^2}{2}\cdot\frac{|\domClabel|^2\ln{|\domClabel|}}{\frac{\delta_1^t+\delta_0^t|\domClabel|^2\ln{|\domClabel|}}{2}}\\
&= \frac{\log{|\Policies^*|}}{\eta_T}+ \frac{1}{\eta_T}\sum_{t=1}^{T}{\eta_t^2}\frac{1}{\delta_1^t+\delta_0^t|\domClabel|^2\ln{|\domClabel|}}\cdot  {|\domClabel|^2\ln{|\domClabel|}},
\end{flalign*}

 where step \paren{g} is by getting the lower bound of $\queryProb_t$ as $\frac{\delta_1^t}{|\domClabel|}\leq\frac{1}{|\domClabel|}$, $\delta_0^t\leq \frac{\delta_0^t}{1-p_{t,y_t}}$ and step \paren{h} is by applying $\max\{A,B\}\geq{\frac{A+B}{2}}$.

Let us define %
$\minpgap{t}
\triangleq \min_{\tau\in[t]}\delta_1^\tau=1-\max_{c,\tau\in[t]}p_{t,y}^\tau$. We get
\begin{flalign*}
\expctover{T}{\mathcal{R}_T}\leq \frac{\log|\Policies^*|}{\eta_T}+\frac{1}{\eta_T}\sum_{t=1}^T{\log{|\Policies^*|}}\cdot \frac{1}{T}\leq \frac{2\log{|\Policies^*|}}{\eta_T}
\end{flalign*}
Let $\eta_t=\sqrt{\frac{ \minpgap{t}+\delta_0^t|\domClabel|^2\ln{|\domClabel|} }{|\domClabel|^2\ln{|\domClabel|}}}\cdot \sqrt{\frac{\log{|\Policies^*|}}{T}}$, we obtain
\begin{flalign*}
\expctover{T}{\mathcal{R}_T}&\leq \frac{2\sqrt{\log{|\Policies^*|}}\cdot\sqrt{T}\cdot\sqrt{|\domClabel|^2\ln{|\domClabel|}}}{\sqrt{\minpgap{T} + \delta_0^T|\domClabel|^2\ln{|\domClabel|}} }\\
&\leq 2|\domClabel|\sqrt{\frac{T\ln{|\domClabel|}\log{|\Policies^*|}}{\fixremoved{\max}\{\minpgap{T},\fixremoved{\sqrt{1/T}}\}}}
\end{flalign*}
\fixremoved{where the last inequality is due to the fact that $$\minpgap{T} + \delta_0^T|\domClabel|^2\ln{|\domClabel| > \max\{\minpgap{T},\delta_0^T}\} = \max\{\minpgap{T},\sqrt{1/T}\}$$ which completes the proof.}
\end{proof}

\subsection{Proof of \thmref{thm:adversarial-query-complexity}}\label{app:adversarial_query_complexity}

\begin{proof}[Proof of \thmref{thm:adversarial-query-complexity}]  
From \lemref{lem:query_complexity}, we get the following equation as the cumulative query cost
\[
\expct{\sum_{t=1}^TU_t} \leq
\expct{\sum_{t=1}^T{\paren{\frac{1}{\sqrt{t}}+\frac{\sum_{y\in \domClabel} \langle \ModelsDistB_t,\lossB_t^{y} \rangle \log_{|\domClabel|}{\frac{1}{\langle \ModelsDistB_t,\lossB_t^{y} \rangle}}}{|\domClabel|}}}}.
\]

Let us assume the expected total loss of best policy is $\tilde{\Loss}_{T,*}$. Thus, from \thmref{thm:adversarial-regret-bound}, 
we get the expected cumulative loss
\[
    \expct{R_T}=\expct{\sum_{t=1}^T r_t}\leq 2|\domClabel|\sqrt{\frac{T\ln{|\domClabel|}\log{|\Policies^*|}}{\fixremoved{\max}\{\minpgap{T},\fixremoved{\sqrt{1/T}}\}}} + \tilde{\Loss}_{T,*}.
\]

Now plugging the regret bound $\mathcal{R}_T$ proved in  \thmref{thm:adversarial-regret-bound} into the query complexity bound given by \lemref{lem:entrophy_query_ub}, we have
\begin{flalign*}
\sum_{t=1}^T{\frac{\sum_{y\in \domClabel} \langle \ModelsDistB_t,\lossB_t^y \rangle \log_{|\domClabel|}{\frac{1}{\langle \ModelsDistB_t,\lossB_t^y \rangle}}}{|\domClabel|}}
&\leq \frac{\paren{ 2|\domClabel|\sqrt{\frac{T\ln{|\domClabel|}\log{|\Policies^*|}}{\fixremoved{\max}\{\minpgap{T},\fixremoved{\sqrt{1/T}}\}}} + \tilde{\Loss}_{T,*}}\paren{\log_{|\domClabel|}\frac{T^2\paren{|\domClabel|-1}}{\paren{2|\domClabel|\sqrt{\frac{T\ln{|\domClabel|}\log{|\Policies^*|}}{\fixremoved{\max}\{\minpgap{T},\fixremoved{\sqrt{1/T}}\}}} + \tilde{\Loss}_{T,*}}^2}}}{|\domClabel|}\\
&\leq \frac{\paren{2|\domClabel|\sqrt{\frac{T\ln{|\domClabel|}\log{|\Policies^*|}}{\fixremoved{\max}\{\minpgap{T},\fixremoved{\sqrt{1/T}}\}}} + \tilde{\Loss}_{T,*}}\paren{\log_{|\domClabel|}{{T}|\domClabel|}}}{|\domClabel|}\\
&= \frac{\paren{2|\domClabel|\sqrt{\frac{T\ln{|\domClabel|}\log{|\Policies^*|}}{\fixremoved{\max}\{\minpgap{T},\fixremoved{\sqrt{1/T}}\}}} + \tilde{\Loss}_{T,*}}\paren{\log_{|\domClabel|}{{T}}+1}}{|\domClabel|}.
\end{flalign*}

Finally, by applying query complexity upper bound of \lemref{lem:eta_query_ub}, we got
\begin{flalign*}
\expct{\sum_{t=1}^TU_t} \leq 2\sqrt{T}+\frac{\paren{2|\domClabel|\sqrt{\frac{T\ln{|\domClabel|}\log{|\Policies^*|}}{\fixremoved{\max}\{\minpgap{T},\fixremoved{\sqrt{1/T}}\}}} + \tilde{\Loss}_{T,*}}\paren{\log_{|\domClabel|}{{T}}+1}}{|\domClabel|}.\\
\end{flalign*}
Since the second term on the RHS dominates the upper bound,
we have
\begin{flalign*}
{O}\paren{\expct{\sum_{t=1}^TU_t}} &={O}\paren{\frac{\paren{\sqrt{\frac{T\log{|\Policies^*|}}{\fixremoved{\max}\{\minpgap{T},\fixremoved{\sqrt{1/T}}\}}} + \tilde{\Loss}_{T,*}}\paren{\ln{{T}}}}{\sqrt{\ln{\paren{|\domClabel|}}}}}
\stackrel{\paren{a}}{=}{O}\paren{\paren{\sqrt{\frac{T\log{|\Policies^*|}}{\fixremoved{\max}\{\minpgap{T},\fixremoved{\sqrt{1/T}}\}}} + \tilde{\Loss}_{T,*}}\paren{\ln{{T}}}},
\end{flalign*}
\fixremoved{where step (a) is obtained by suppressing constant coefficients involving  $|\domClabel|$ into the $O$ notation.}
\end{proof}

\clearpage
\section{Additional Experiments}\label{app:additional_experiment}
{

In this section, we further evaluate \algname and provide additional experimental results (complementary to the main results presented in \figref{fig:exp:results}) under the following scenarios: 
\begin{enumerate}
\item In \appref{app:exp:large_scale}, we demonstrate that \algname outperforms the baselines on a large scale dataset as well.
\item In \appref{app:query_strategies_comparison}, we perform ablation study of three query strategies CAMS (entropy), variance and random strategy.  \emph{\algname (Entropy)} achieves the minimum cumulative loss for CIFAR10, DRIFT, and VERTEBRAL under the same query cost and outperform the other query strategies.
\item In a \textit{mixture of experts} environment, \algname converges to the best policy and outperforms all others (\appref{app:idvidual_expert_comparison}); 
\item In a \textit{non-contextual} (no experts) environment, \algname has approximately equal performance as Model Picker to reach the best classifier effectively (\appref{app:compare_CAMS_with_MP}); 
\item In an \textit{adversarial} environment, \algname can efficiently recover from the adversary and approach the performance of the best classifier (\appref{app:recover_in_complete_malicious_environment});
\item In a \textit{complete sub-optimal expert} environment, a variant of the \algname algorithm, namely \algname-MAX, which deterministically picks the most probable policy and selects the most probable model, %
outperforms \algname-Random-Policy, which randomly samples a policy and selects the most probable model %
(\appref{app:CAMS_MAX} \& \appref{app:CAMS_random_policy}). However, \algname-MAX at most approaches the performance of the best policy. In contrast, perhaps surprisingly, \algname is able to outperform the best policy on both VERTEBRAL and HIV (\appref{app:outperform_best_expert}). 
\item In \appref{app:exp:max_query}, we summarize the maximum query cost under a fixed number of realizations with its associated cumulative loss for all baselines (exclude oracle) on all benchmarks in experiment section.

\item In \appref{app:exp:query_complexity_experiment}, we compare the query complexity of each baselines and demonstrate that \algname has the lowest query cost increasing rate on CIFAR10, DRIFT and VERTEBRAL dataset.
\item In previous studies, we assume that the data comes in an online format. In \appref{app:scaling_param}, we assume we know the data stream length ahead and applying the scaling parameter to each algorithm to query their top data points from hindsight. \algname still outperforms all the baselines.
\item In \appref{app:sample_efficiency}, we compare \algname with \algname-nonactive, a greedy version (query label for each incoming data point). Although \algname query much less data, it still performs equally well or even better than the greedy version.
\item  In \appref{app:exp:rcl}, we demonstrates that \algname can achieve negative RCL on all benchmarks, which means it outperforms any algorithms that chase the best classifier where the horizontal 0 line represents the performance benchmark of best classifier.

\end{enumerate}
}

\subsection{Performance of \algname at scale: Experimental results on CovType}\label{app:exp:large_scale}
We scaled up our experiments on a larger dataset, {{CovType}} \citep{Dua:2019}. The CovType dataset offers details about different types of forest cover in the United States. It contains details including slope, aspect, elevation, measurements of the wilderness area, and the type of forest cover. CovType has 580K samples, of which 100K instances were chosen at random as online stream for testing. \figref{fig:exp:covtype} demonstrated that \algname outperforms all baselines which is consistent with the existing results in experiment section.

\begin{figure*}[ht]
        \vspace{-2mm}
\begin{subfigure}{1\textwidth}
    \centering
    \includegraphics[height=0.3cm, clip={0,0,0,0}]{./figures/legend_horizontal.png}
\end{subfigure}
\rotatebox[origin=l]{90}{\scriptsize \qquad\qquad Relative cumulative loss}
\begin{subfigure}{.3\textwidth}
    \centering
        \includegraphics[height=3.7cm, clip={0,0,0,0}]{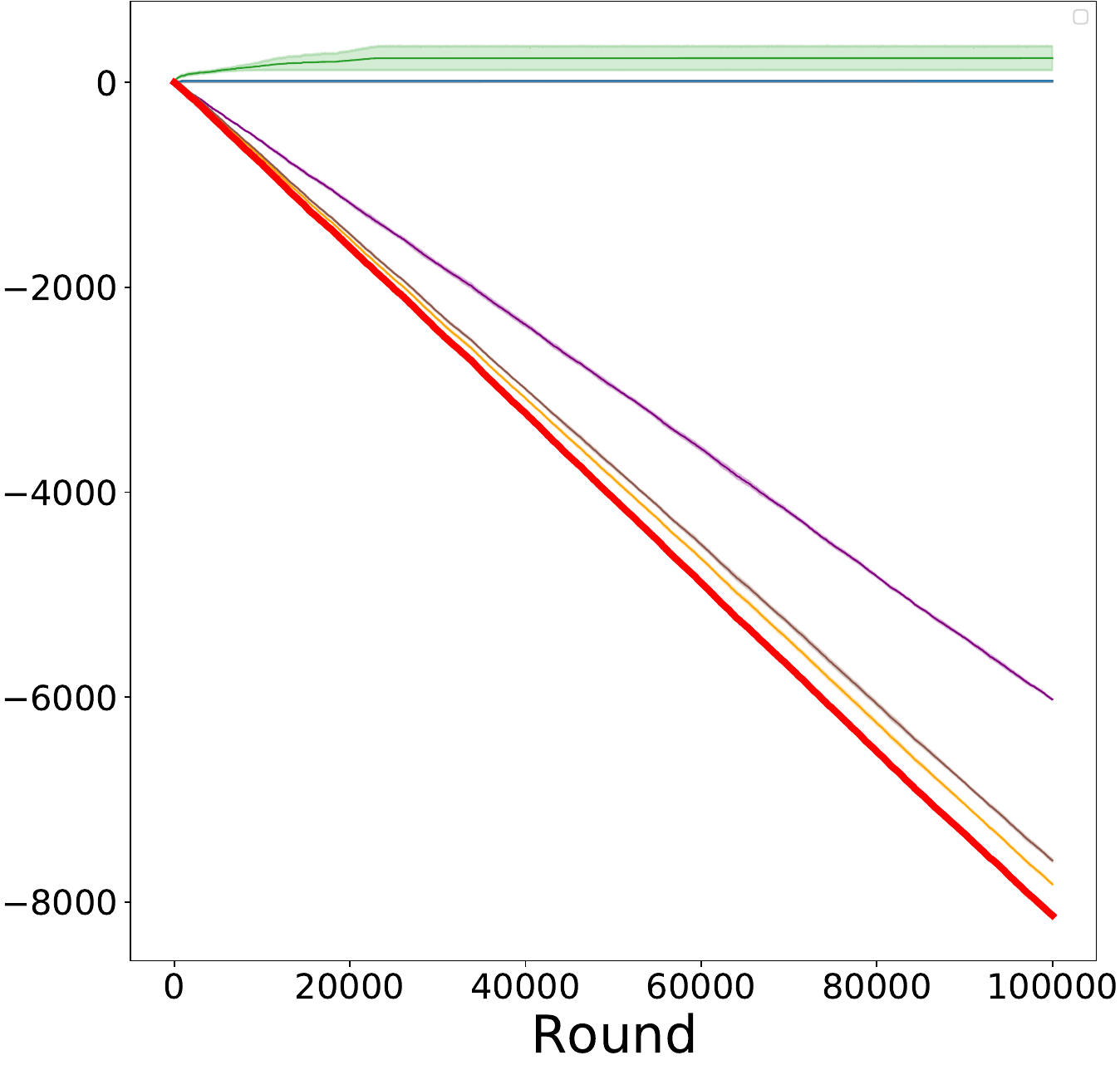}
        \vspace{-2mm}
    \caption{{Relative Cumulative Loss}}
\end{subfigure}
\rotatebox[origin=l]{90}{\scriptsize \qquad\qquad Number of queries}
\begin{subfigure}{.3\textwidth}
    \centering
        \includegraphics[height=3.7cm, clip={0,0,0,0}]{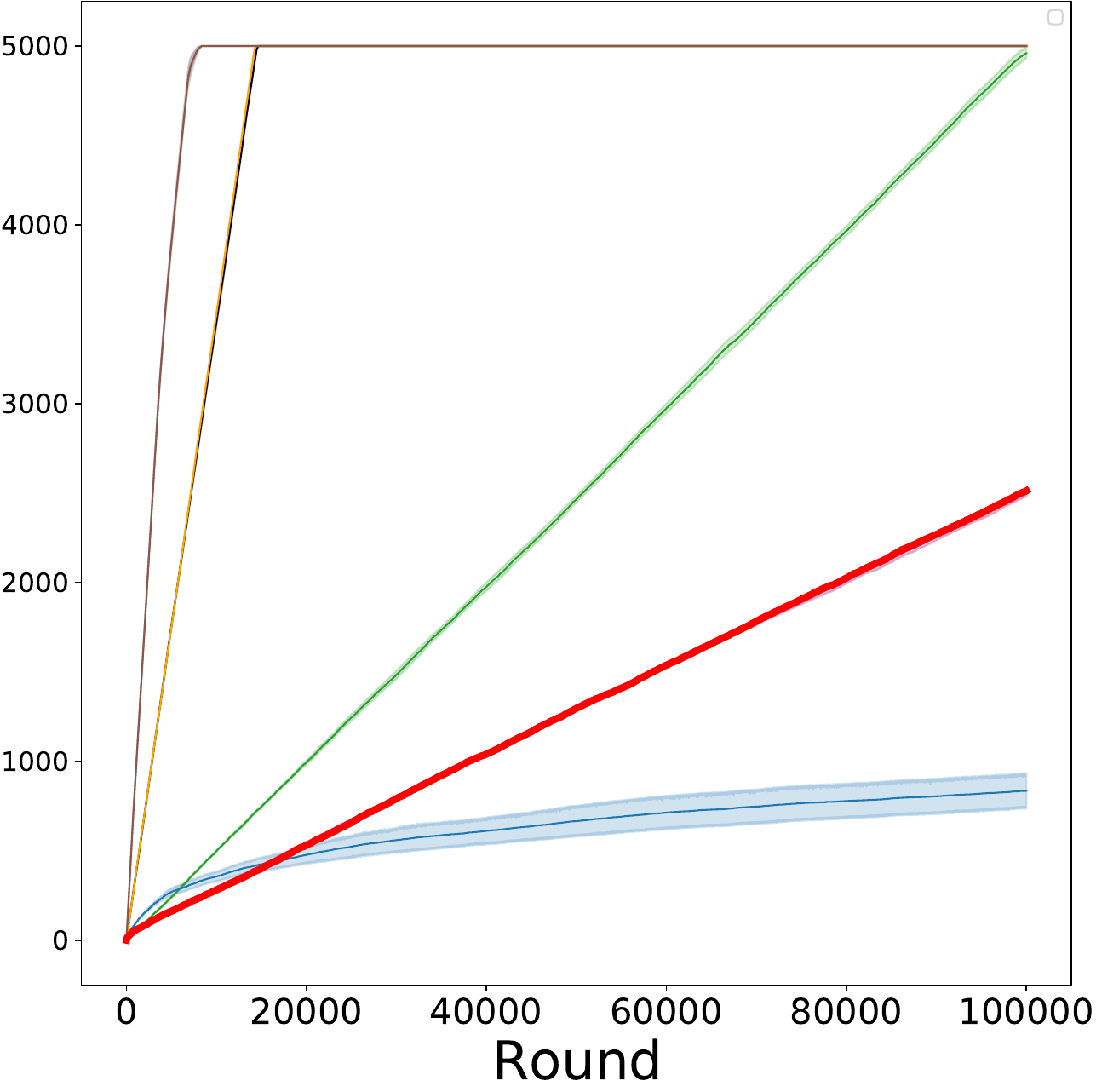}
        \vspace{-2mm}
    \caption{{ Cost}}
\end{subfigure}
\rotatebox[origin=l]{90}{\scriptsize \qquad\qquad Cumulative loss}
\begin{subfigure}{.3\textwidth}
    \centering
        \includegraphics[height=3.7cm, clip={0,0,0,0}]{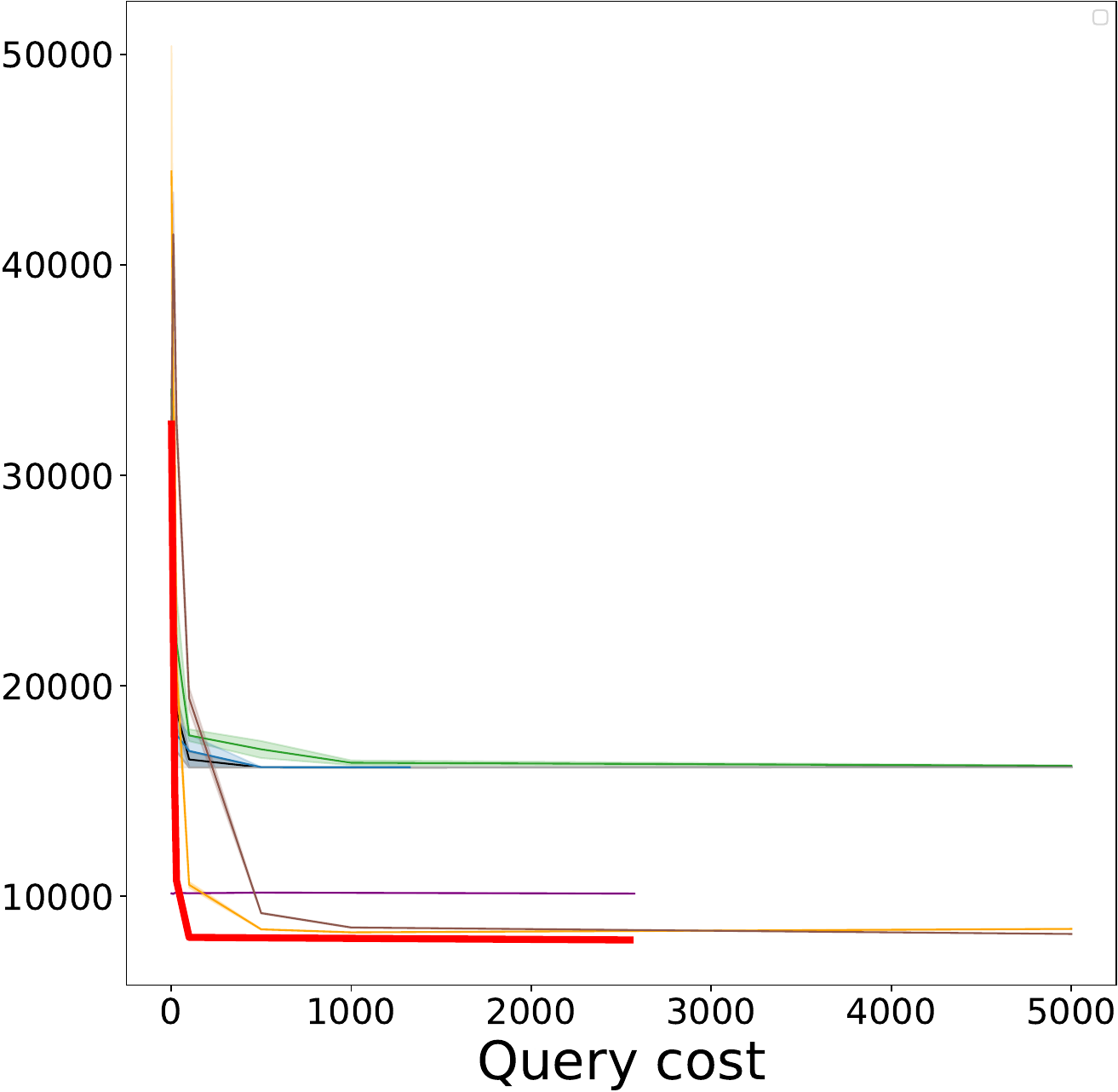}
        \vspace{-2mm}
    \caption{{Cumulative Loss VS Cost}}
\end{subfigure}
        \caption{Comparing \algname with 7 baselines on CovType in terms of relative cumulative loss, query complexity, and cost effectiveness. \algname outperforms all baselines. \textbf{(Left)} Performance measured by relative cumulative loss (i.e. loss against the best classifier) under a fixed query cost $B$ (where $B={1000}$). \textbf{(Middle)} Number of queries %
        and %
        \textbf{(Right)} Performance of cumulative loss by increasing the query cost, for a fixed number of rounds $T$ (where $T=100,000$) and maximal query cost $B$ (where $B=5000$ ). \textbf{Algorithms}: 4 contextual \{Oracle, CQBC, CIWAL, CAMS\} and 4 non-contextual baselines \{RS, QBC, IWAL, MP\} are included (see Section \pref{main:baseines}). 90\% confident interval are indicated in shades.}
    
    \label{fig:exp:covtype}
        \vspace{-4mm}
\end{figure*}

\subsection{Query strategies ablation comparison}\label{app:query_strategies_comparison}
 Using the same \algname model recommendation section, we compare three query strategies: the adaptive model-disagreement-based query strategy in Line~\ref{alg:cams:line:queryprob}-\ref{alg:cams:line:observe} of \figref{alg:CAMS} {(referred to as \emph{entropy} in the following)}, the variance-based query strategy from Model Picker \citep{karimi2021online} {(referred to as \emph{variance})}, and a random query strategy. \figref{fig:compare_3_query_strategies} shows that \algname's adaptive query strategy has the sharpest converge rate on cumulative loss, which demonstrates the effectiveness of the queried labels. Moreover, \emph{entropy} achieves the minimum cumulative loss for CIFAR10, DRIFT, and VERTEBRAL under the same query cost. For the HIV dataset, there is no clear winner between \emph{entropy} and \emph{variance} since the mean of their performance lie within the error bar of each other for the most part.

\begin{figure}[h]
\begin{subfigure}{1\textwidth}
    \centering
    \includegraphics[height=0.3cm, clip={0,0,0,0}]{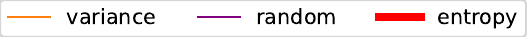}
\end{subfigure}
\rotatebox[origin=l]{90}{\scriptsize \qquad\qquad Cumulative loss}
\begin{subfigure}{.24\textwidth}
    \centering
        \includegraphics[height=3.3cm, clip={0,0,0,0}]{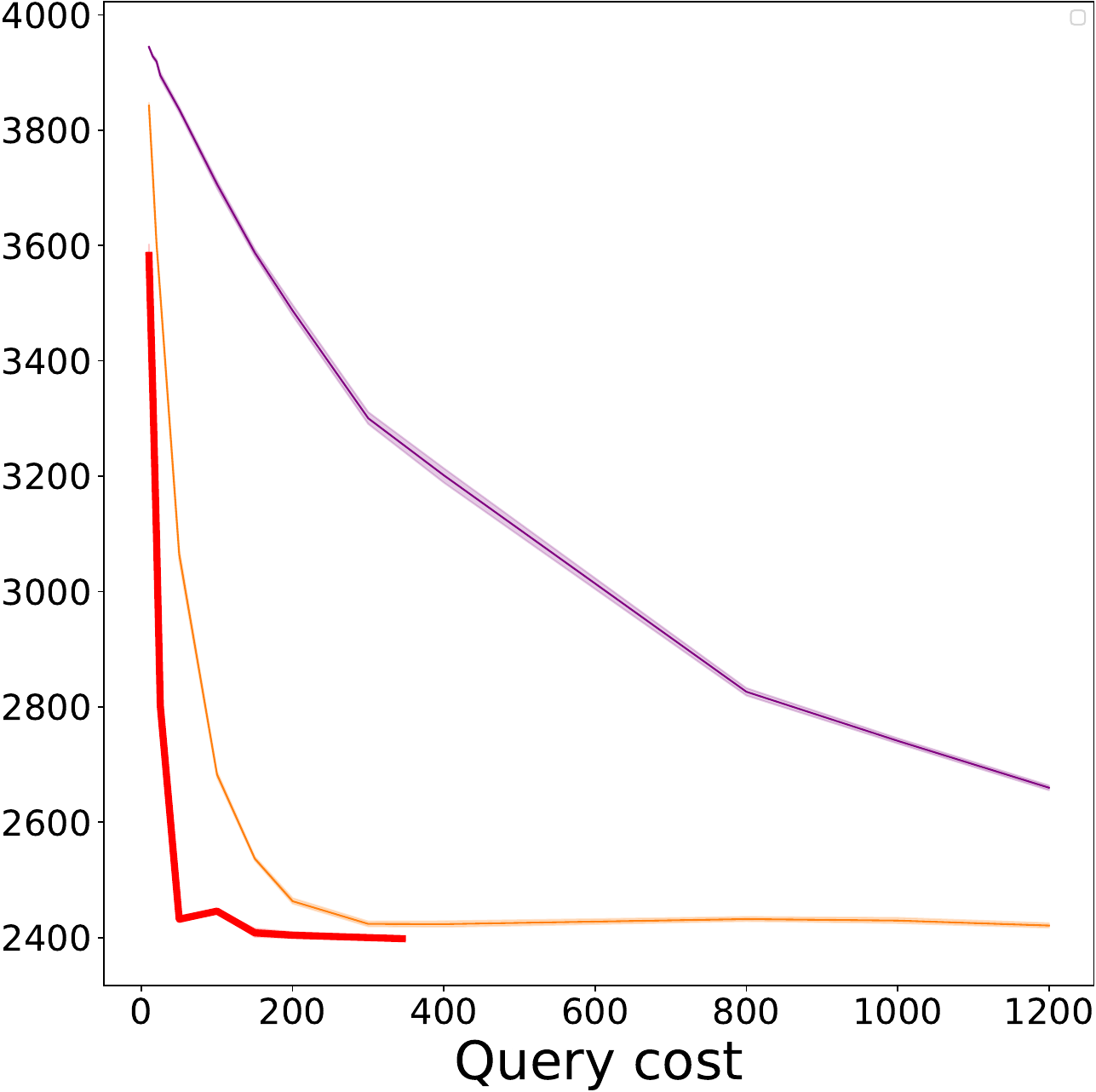}
    \caption{CIFAR10}
\end{subfigure}
\begin{subfigure}{.24\textwidth}
    \centering
        \includegraphics[height=3.3cm, clip={0,0,0,0}]{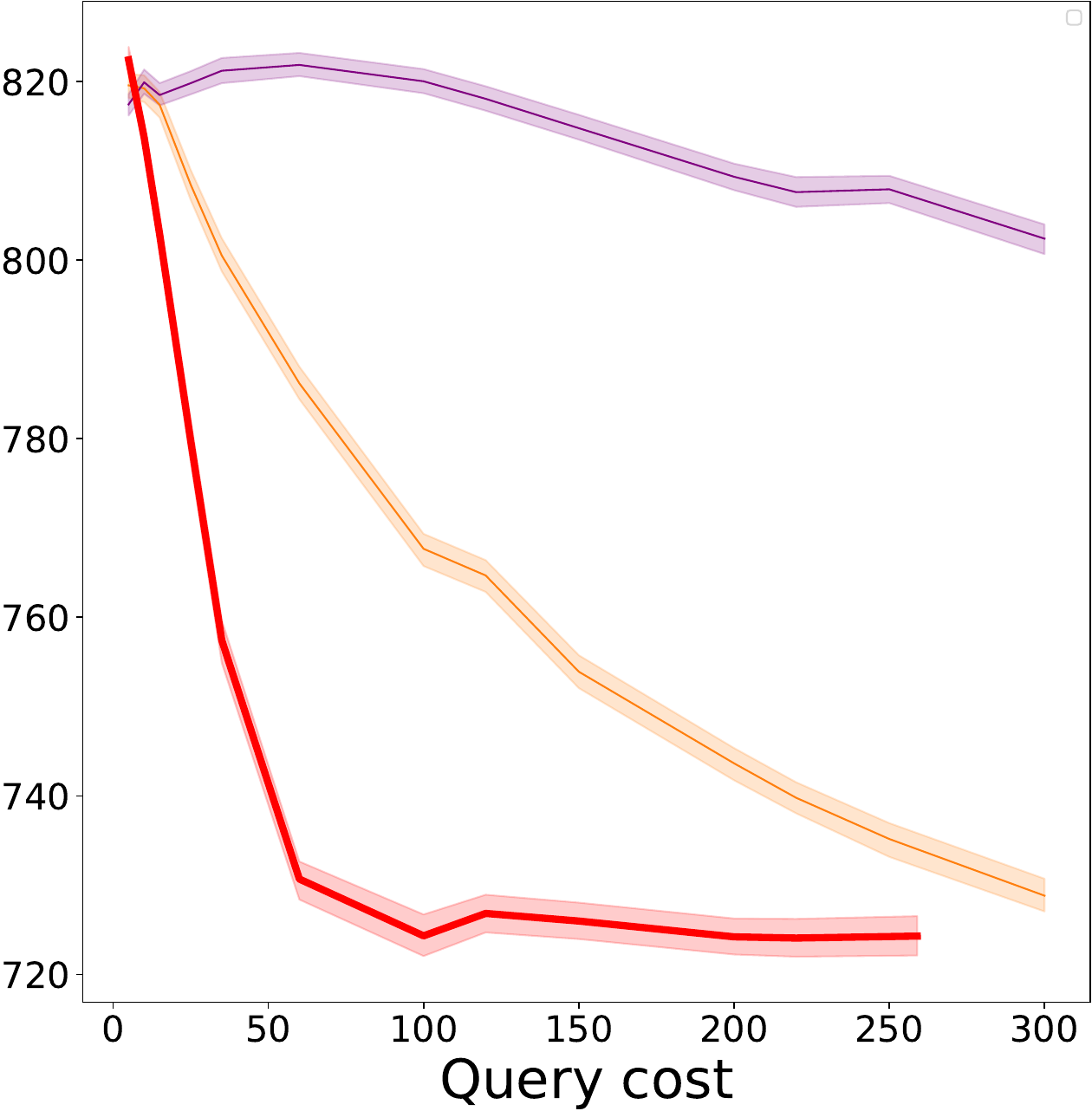}
    \caption{DRIFT}
\end{subfigure}
\begin{subfigure}{.24\textwidth}
    \centering
        \includegraphics[height=3.3cm, clip={0,0,0,0}]{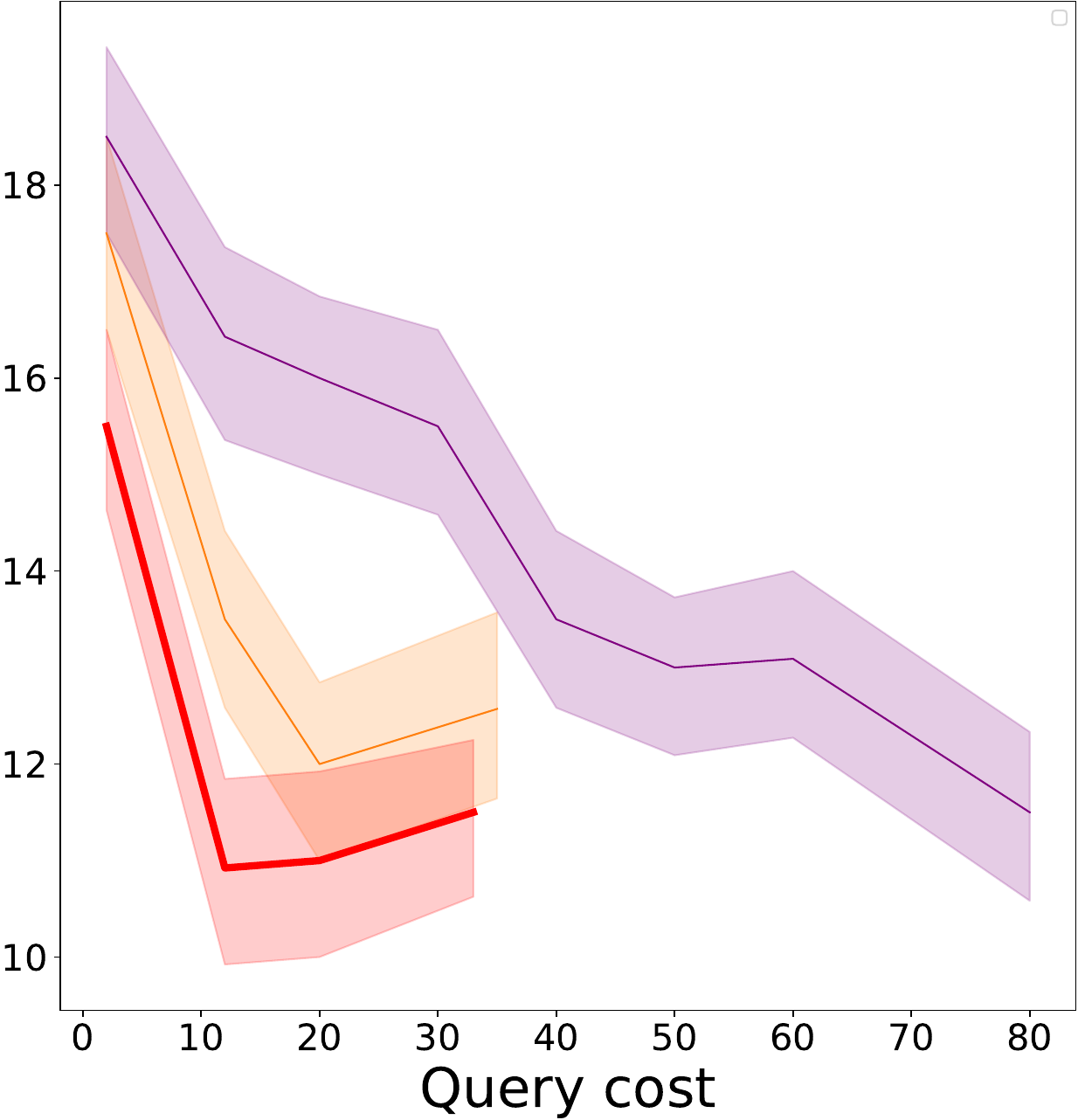}
    \caption{VERTEBRAL}
\end{subfigure}
\begin{subfigure}{.24\textwidth}
    \centering
        \includegraphics[height=3.3cm, clip={0,0,0,0}]{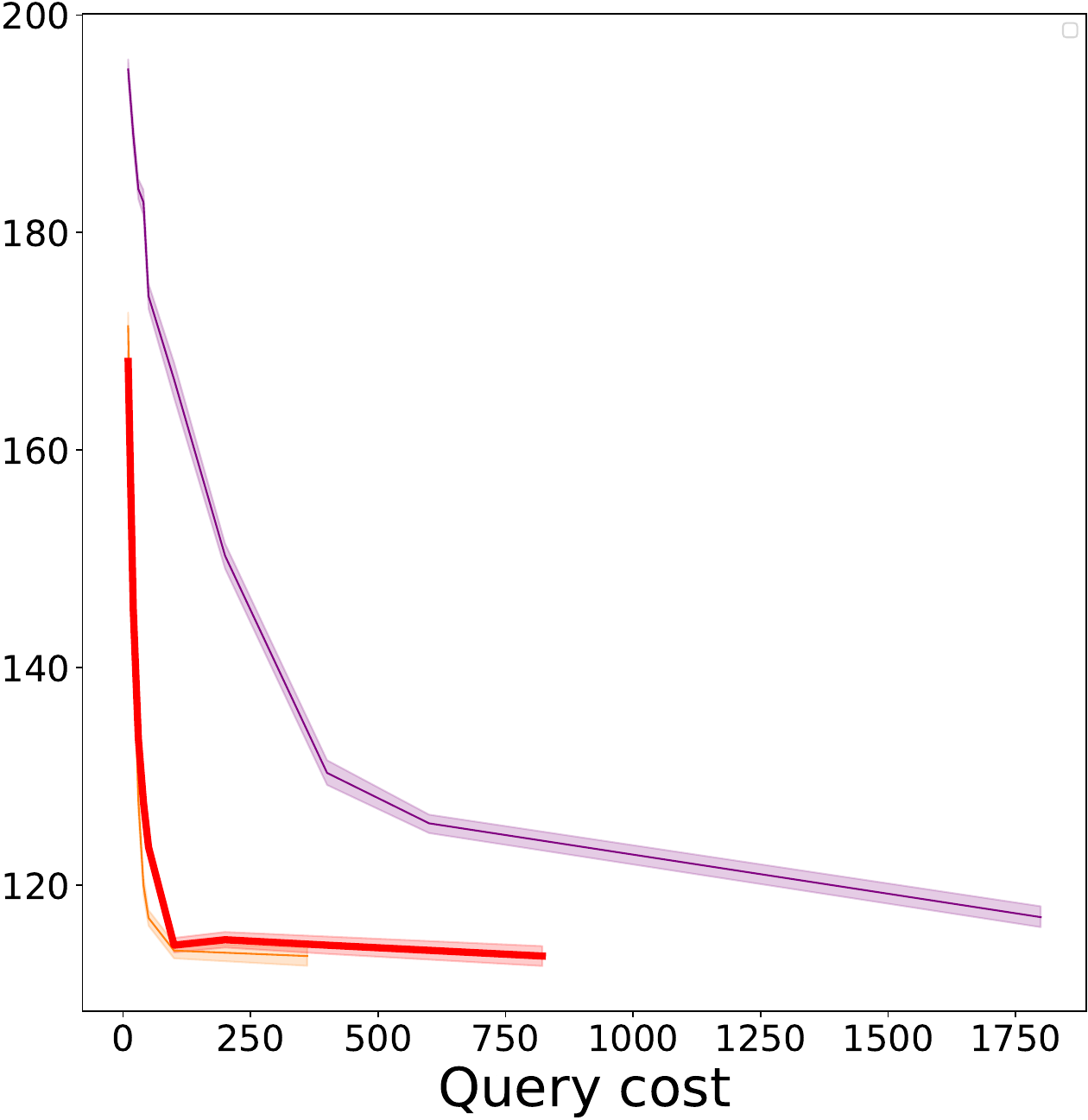}
    \caption{HIV}
\end{subfigure}
        \caption{Ablation study of three query strategies (entropy, variance, random) for 4 diverse benchmarks based on the same model recommendation strategy. Under the same query cost constraint, \algname's entropy-based strategy exceeds the performance of the other two strategies on non-binary benchmarks in terms of query cost and cumulative lost. 90\% confident intervals are indicated in shades.}
        \label{fig:compare_3_query_strategies}
\end{figure}

\subsection{Comparing \algname with each individual expert}\label{app:idvidual_expert_comparison} 

We evaluate \algname by comparing it with all the policies available in various benchmarks. The policies in each benchmark are summarized in \appref{app:policy_classifier} and \tabref{app:dataset_table}. The empirical results in \figref{fig:policy_individual_policy} demonstrate that \algname could efficiently outperform all policies and converge to the performance of the best policy with only slight increase in query cost in all benchmarks. In particular, on the VERTEBRAL and HIV benchmarks, \algname even outperforms the best policy.

\begin{figure}[h]
\begin{subfigure}{1\textwidth}
    \centering
    \includegraphics[height=0.3cm, clip={0,0,0,0}]{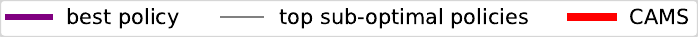}
\end{subfigure}
\rotatebox[origin=l]{90}{\quad \quad \scriptsize Cumulative loss}
\begin{subfigure}{.24\textwidth}
    \centering
        \includegraphics[height=3.3cm, clip={0,0,0,0}]{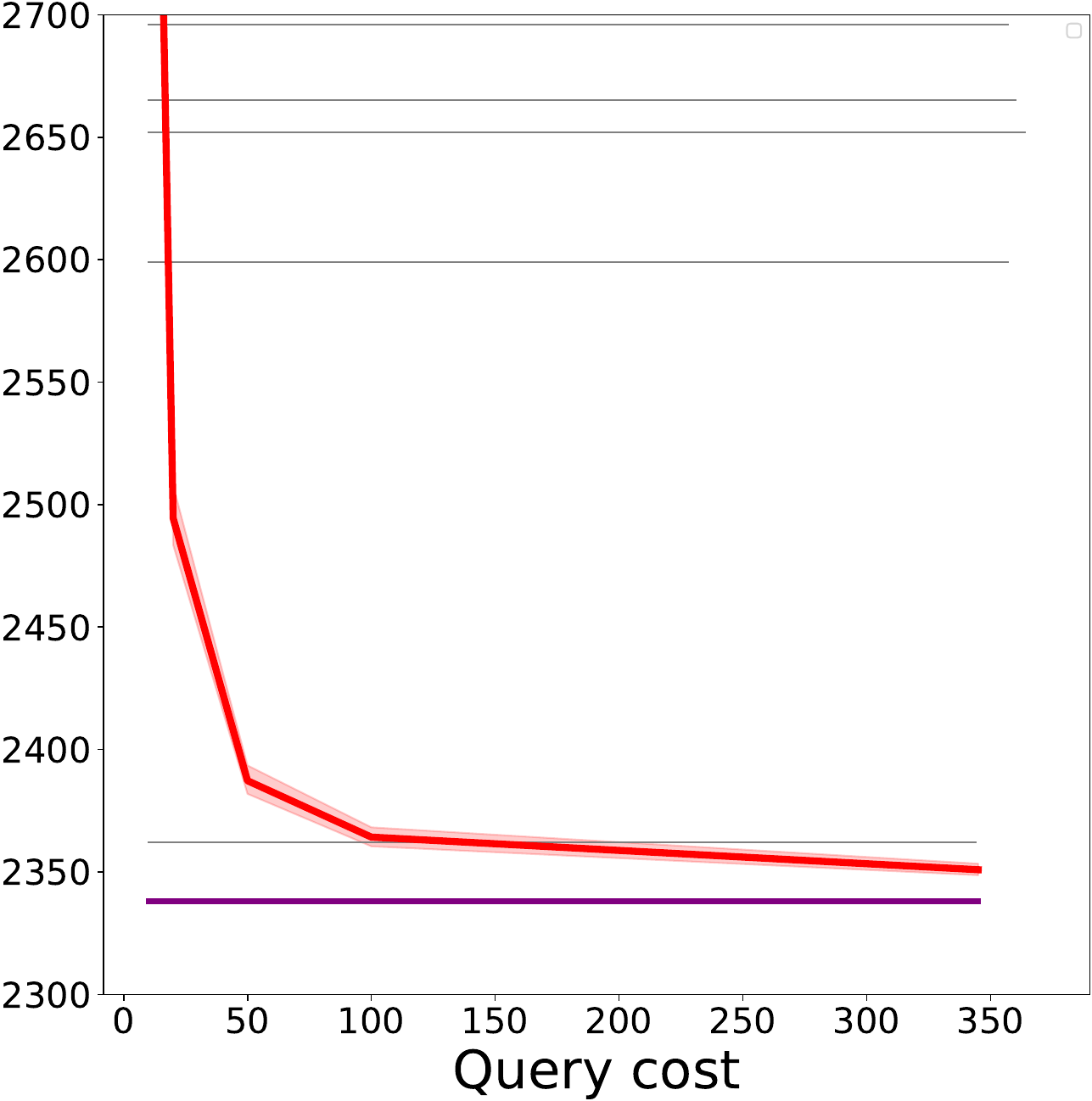}
    \caption{CIFAR10}
\end{subfigure}
\begin{subfigure}{.24\textwidth}
    \centering
        \includegraphics[height=3.3cm, clip={0,0,0,0}]{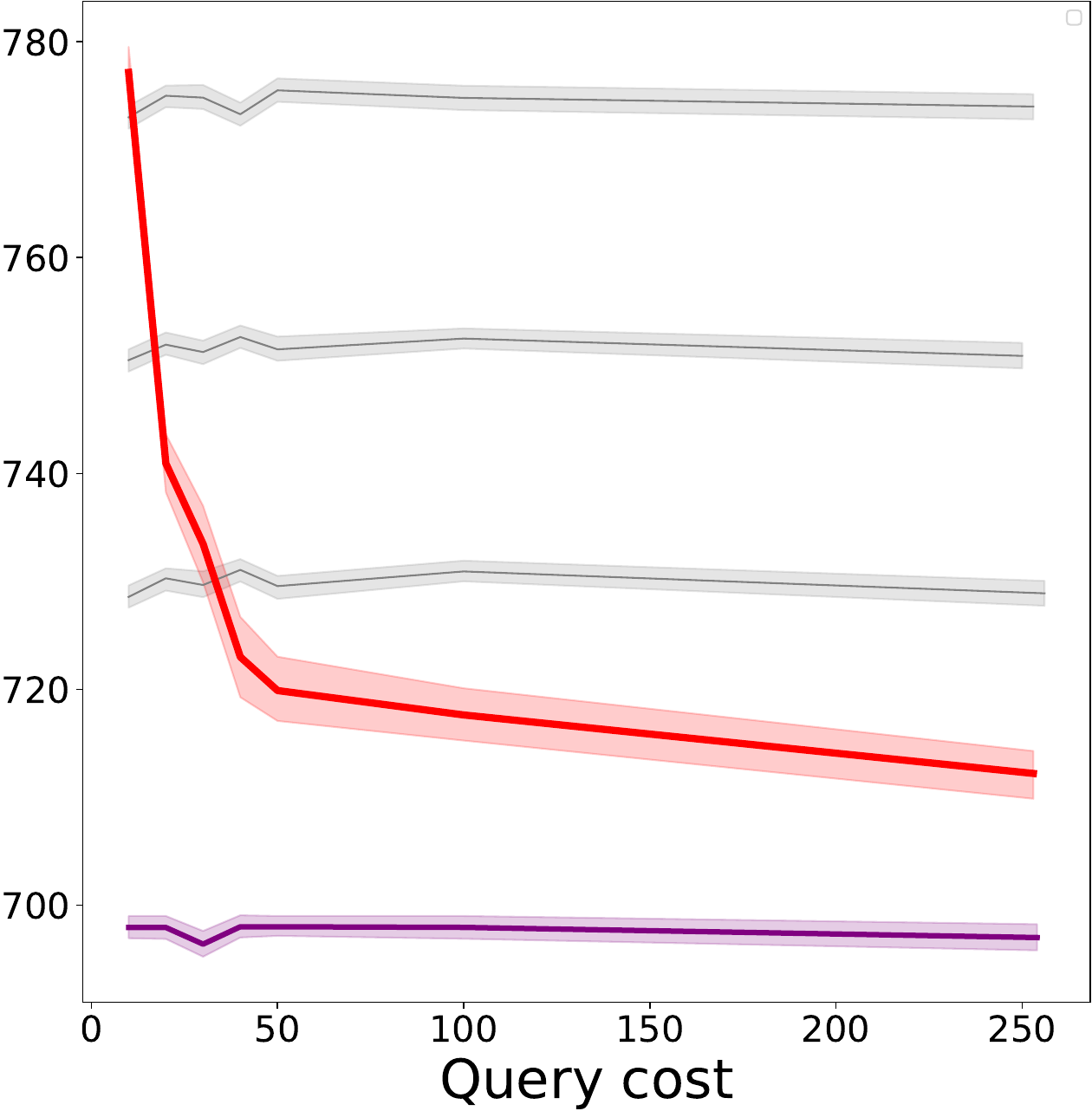}
    \caption{DRIFT}
\end{subfigure}
\begin{subfigure}{.24\textwidth}
    \centering
        \includegraphics[height=3.3cm, clip={0,0,0,0}]{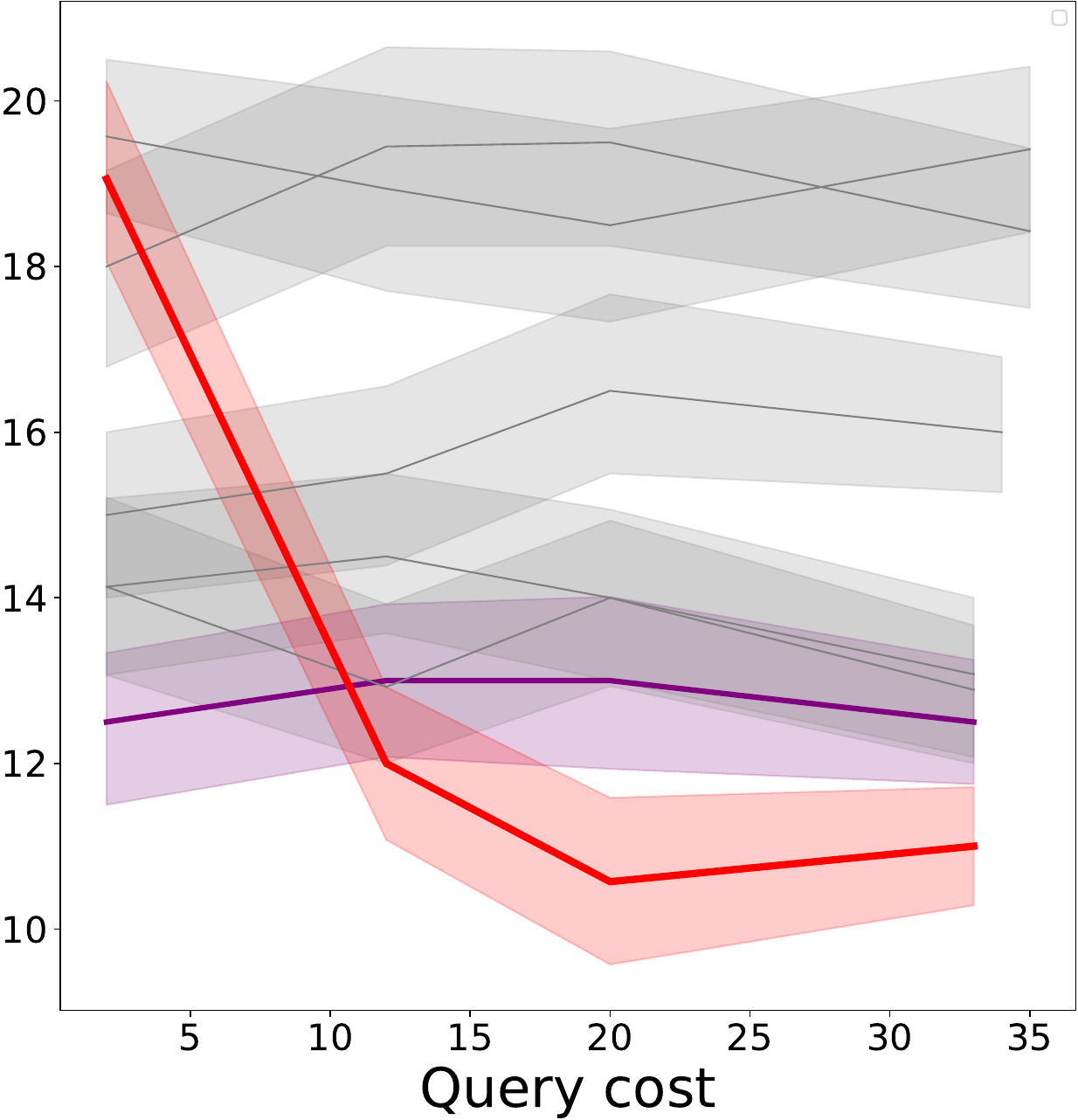}
    \caption{VERTEBRAL}
\end{subfigure}
\begin{subfigure}{.24\textwidth}
    \centering
        \includegraphics[height=3.3cm, clip={0,0,0,0}]{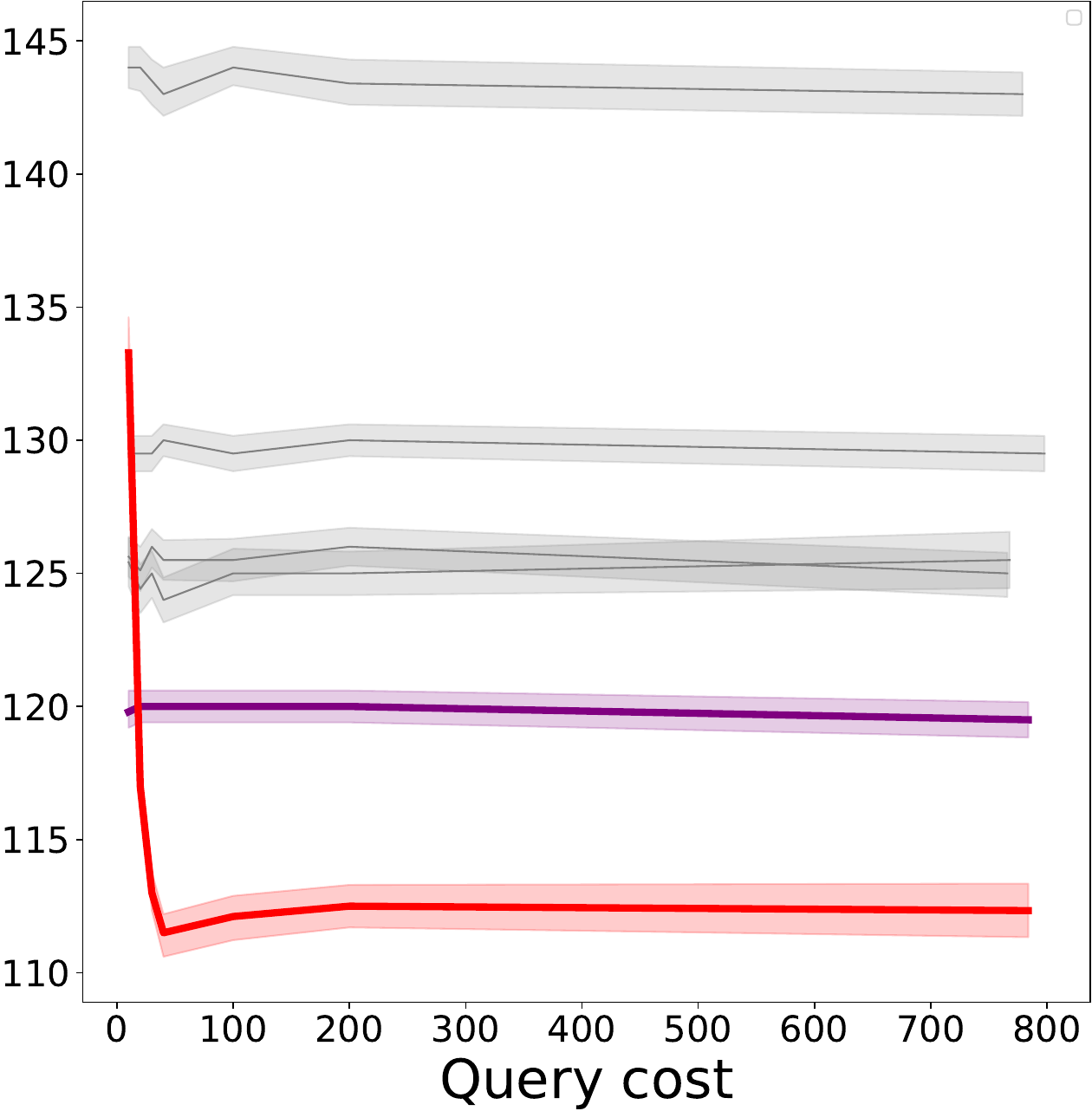}
    \caption{HIV}
\end{subfigure}
        \caption{Comparing \algname with every single policy (only plotted top performance policies in Figure). \algname could approach the best expert and exceed all others with limited queries. In particular, on VERTEBRAL and HIV Benchmarks, \algname outperforms the best expert. 90\% confident intervals are indicated in shades.}   \label{fig:policy_individual_policy}
\end{figure}

\subsection{Comparing \algname against Model Picker in a context-free environment}\label{app:compare_CAMS_with_MP}
{\algname %
outperforms Model Picker in \figref{fig:exp:results}, by leveraging the context information for adaptive model selection.}
In a context-free environment, $ \Policies = \curlybracket{\varnothing}$, so $\Policies^* := \curlybracket{\policy^{\text{const}}_1, \dots, \policy^{\text{const}}_k}$, where $\policy^{\text{const}}_j (\cdot) := \mathbi{e}_j$ represents a policy that only recommends a fixed model. In this case, selecting the best policy to \algname equals selecting the best single model. \figref{fig:mp_ap-plots} demonstrates that the mean of \algname and Model Picker lies in the shades of each other, which means \algname has approximately the same performance as model picker considering the randomness on all benchmarks.

\begin{figure}[h]
\begin{subfigure}{1\textwidth}
    \centering
    \includegraphics[height=0.3cm, clip={0,0,0,0}]{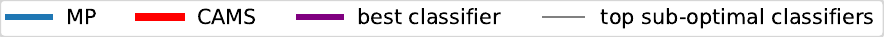}
\end{subfigure}
\rotatebox[origin=l]{90}{\quad \quad \scriptsize Cumulative loss}
\begin{subfigure}{.24\textwidth}
    \centering
        \includegraphics[height=3.3cm, clip={0,0,0,0}]{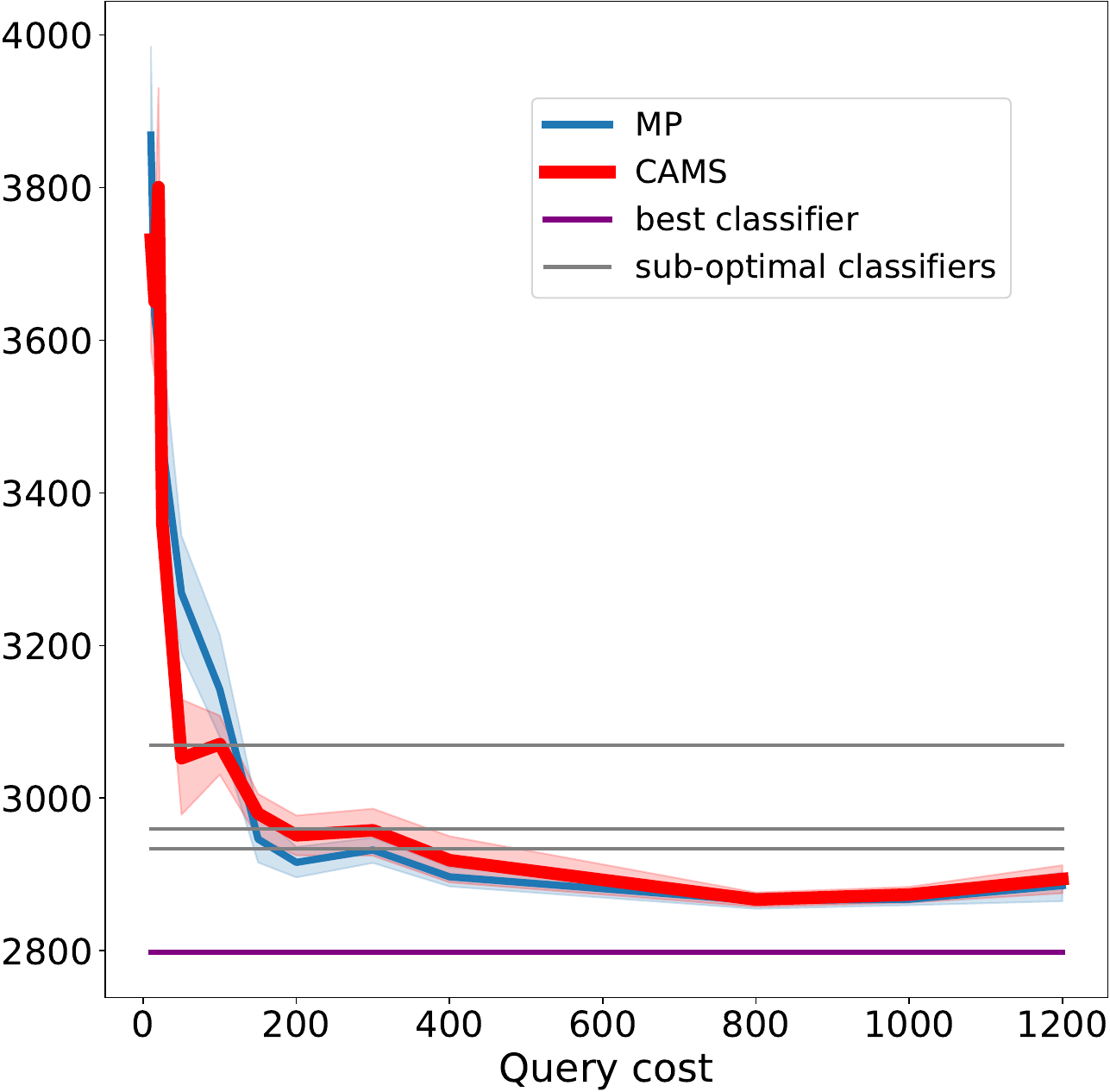}
    \caption{CIFAR10}
\end{subfigure}
\begin{subfigure}{.24\textwidth}
    \centering
        \includegraphics[height=3.3cm, clip={0,0,0,0}]{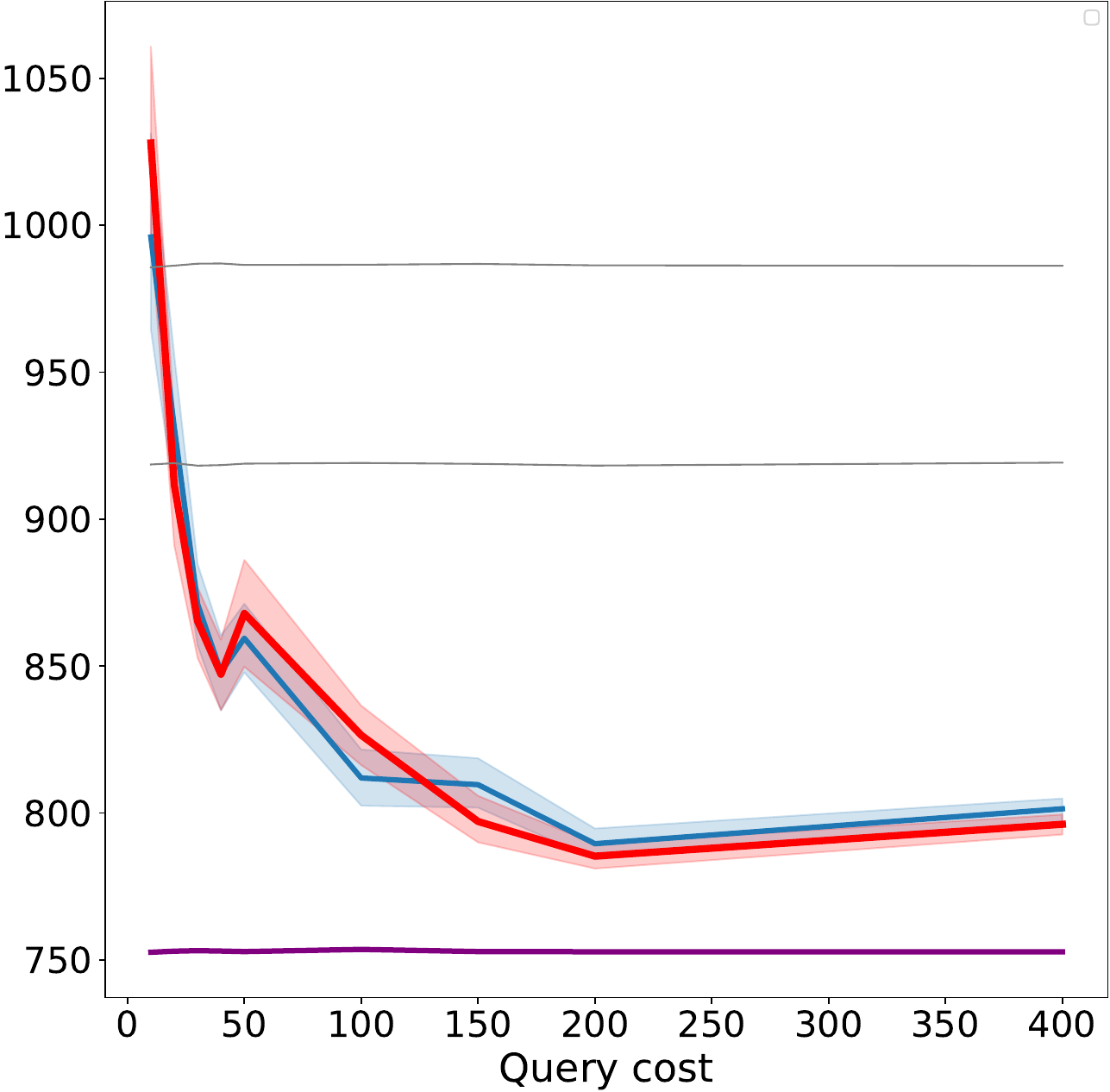}
    \caption{DRIFT}
\end{subfigure}
\begin{subfigure}{.24\textwidth}
    \centering
        \includegraphics[height=3.3cm, clip={0,0,0,0}]{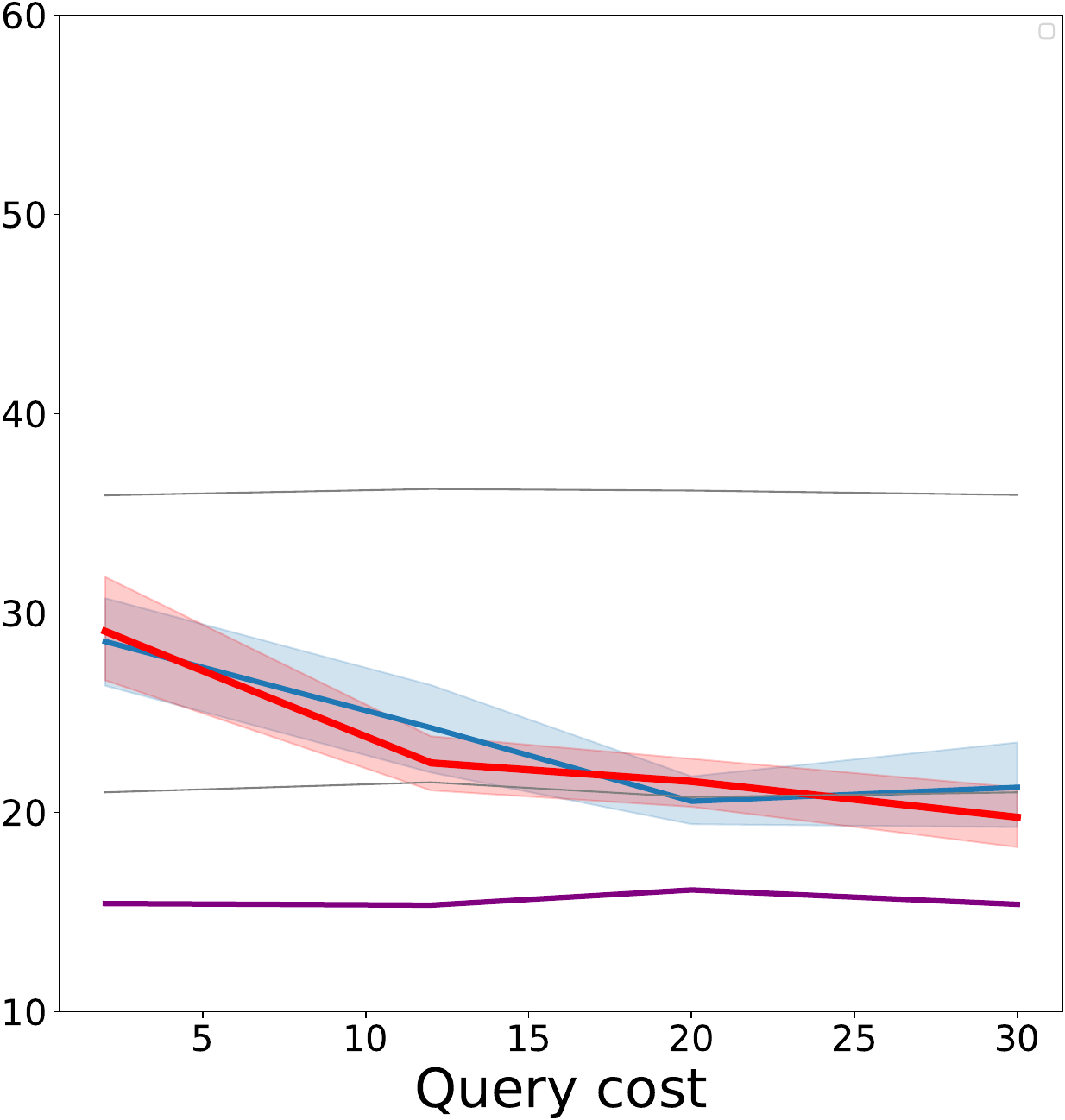}
    \caption{VERTEBRAL}
\end{subfigure}
\begin{subfigure}{.24\textwidth}
    \centering
        \includegraphics[height=3.3cm, clip={0,0,0,0}]{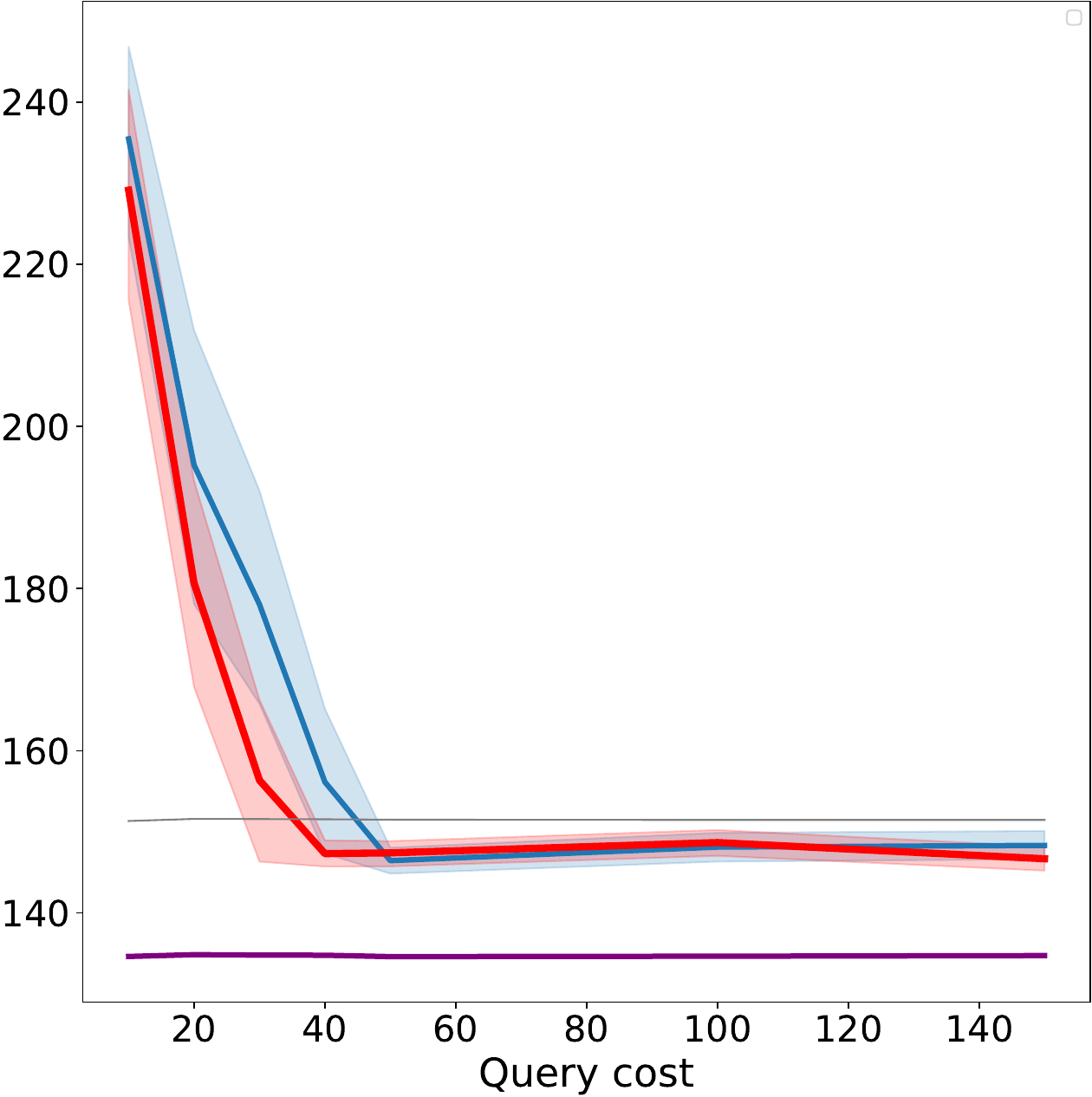}
    \caption{HIV}
\end{subfigure}
        \caption{Comparing the model selection strategy of \algname and Model Picker baseline based on the same variance-based query strategy in a context-free environment. \algname has approximately the same performance as Model Picker on all the benchmarks. 90\% confident intervals are indicated in shades.}
        \label{fig:mp_ap-plots}
\end{figure}

\subsection{Robustness against malicious experts in adversarial environments}\label{app:recover_in_complete_malicious_environment}

When given only malicious and random advice policies, the conventional contextual online learning from experts advice framework will be trapped in the malicious or random advice. In contrast, \algname could efficiently identify these policies and avoid taking advice from them. Meanwhile, it also successfully identifies the best classifier to learn to reach its best performance. 

The \emph{novelty} in \algname that enables this robustness is that we add the constant policies  \curlybracket{\policy^{\text{const}}_1, \dots, \policy^{\text{const}}_k} into the policy set $\Policies$ to form the new set as $\Policies^*$. To illustrate the performance difference, we have created a variant of \algname by adapting to the conventional approach (named \algname-conventional). \figref{fig:malicious-plots} demonstrates that \algname could outperform all the malicious and random policies and converge to the performance of the best classifier. \textbf{\emph{\algname-conventional:}} We create the \algname-conventional algorithm as the \algname using policy set $\Policies$, not $\Policies^*$.

\begin{figure}[h]
\begin{subfigure}{1\textwidth}
    \centering
    \includegraphics[height=0.3cm, clip={0,0,0,0}]{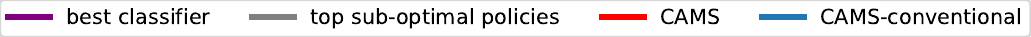}
\end{subfigure}
\rotatebox[origin=l]{90}{\quad \qquad \scriptsize Cumulative loss}
\begin{subfigure}{.24\textwidth}
    \centering
        \includegraphics[height=3.2cm, clip={0,0,0,0}]{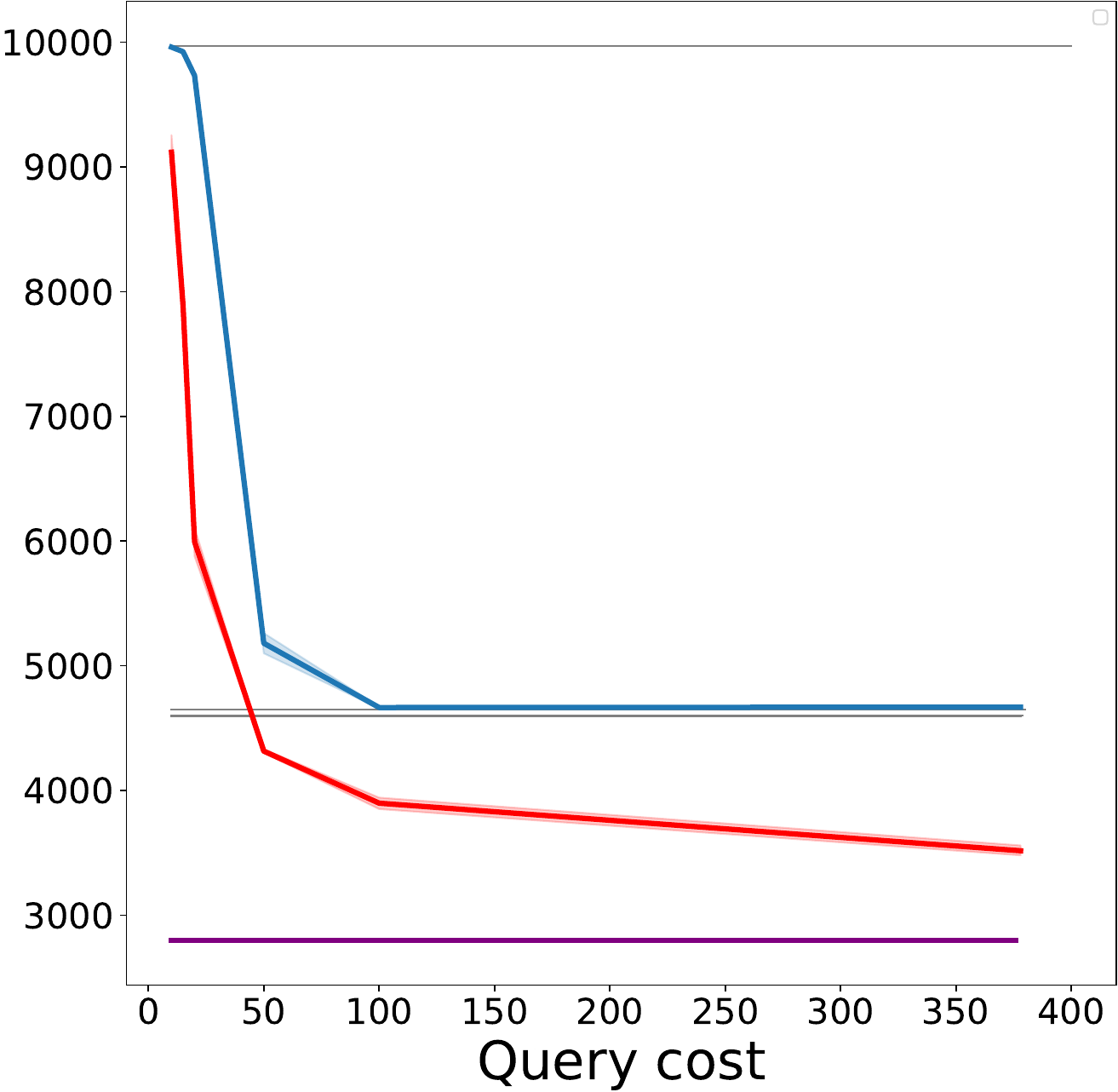}
    \caption{CIFAR10}
\end{subfigure}
\begin{subfigure}{.24\textwidth}
    \centering
        \includegraphics[height=3.3cm, clip={0,0,0,0}]{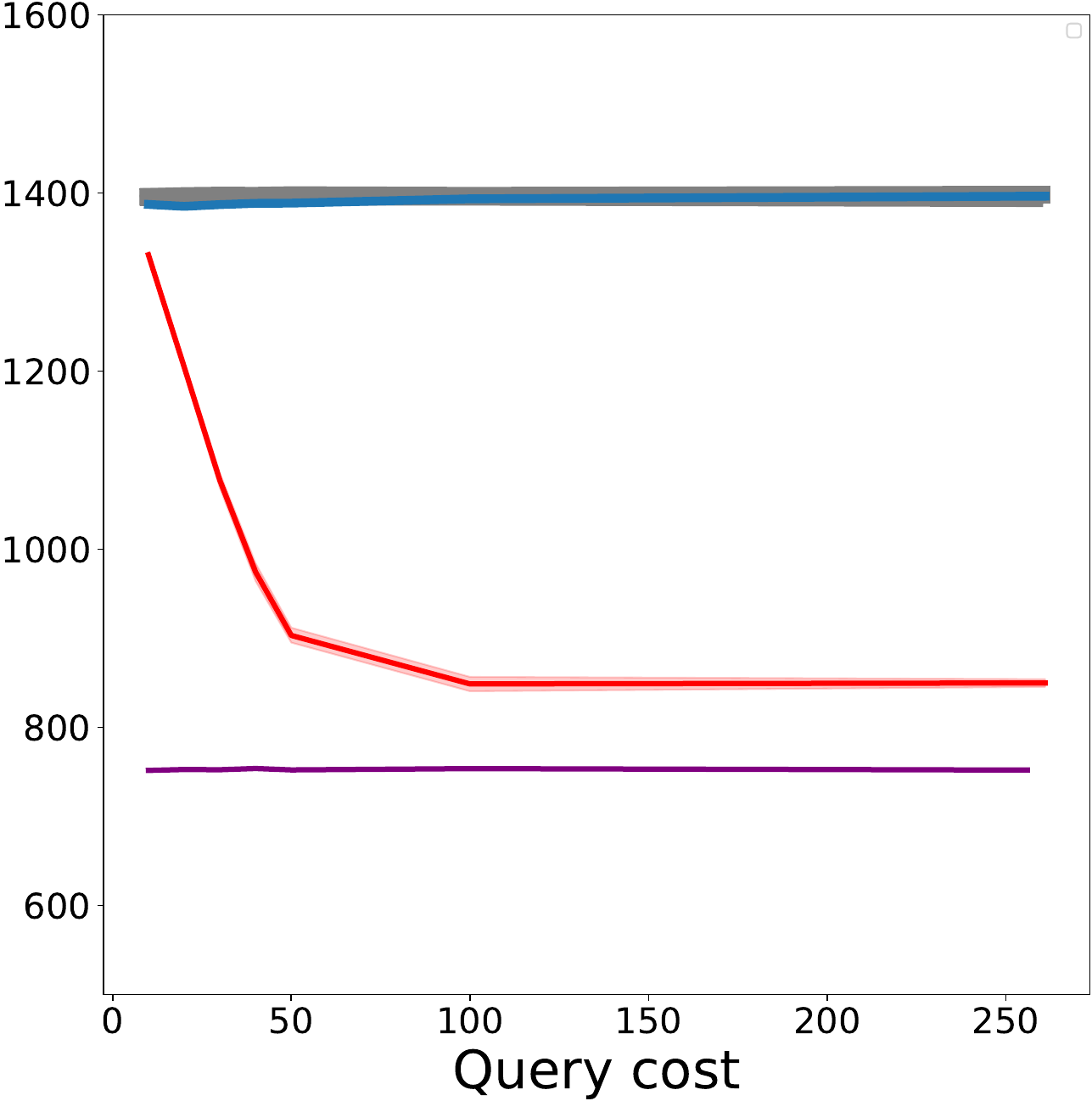}
    \caption{DRIFT}
\end{subfigure}
\begin{subfigure}{.24\textwidth}
    \centering
        \includegraphics[height=3.3cm, clip={0,0,0,0}]{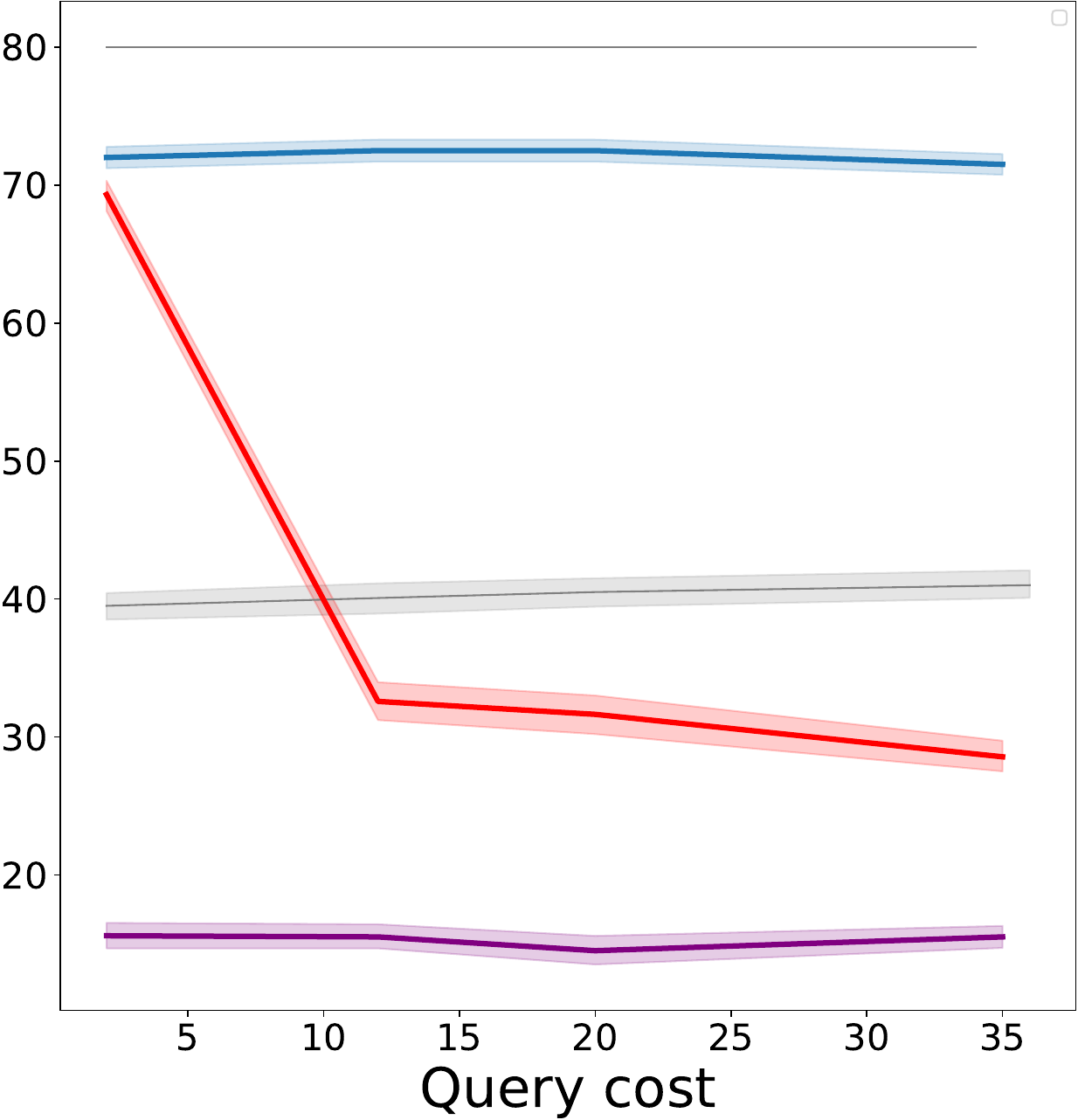}
    \caption{VERTEBRAL}
\end{subfigure}
\begin{subfigure}{.24   \textwidth}
    \centering
        \includegraphics[height=3.3cm, clip={0,0,0,0}]{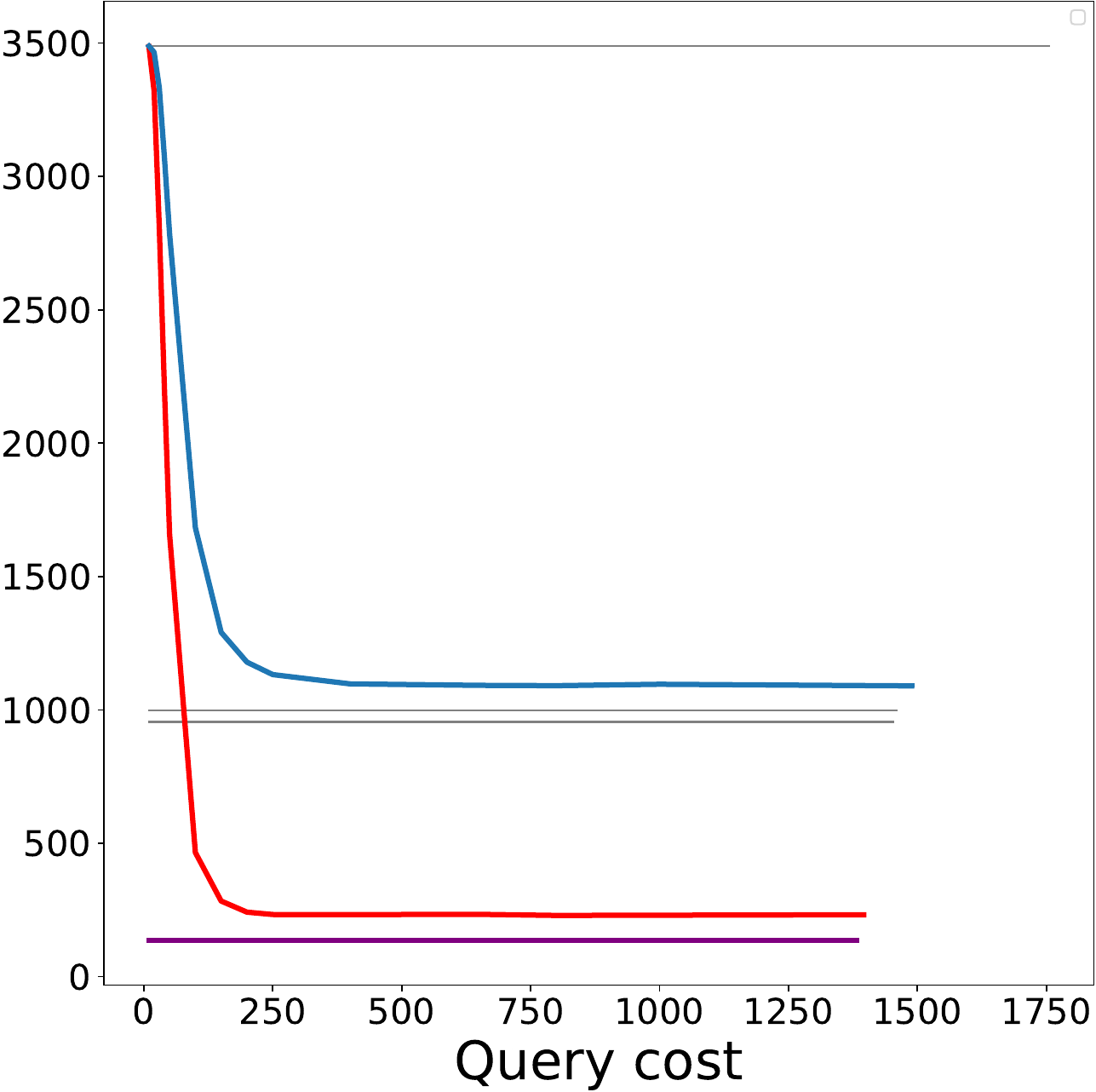}
    \caption{HIV}
\end{subfigure}
      \caption{Evaluating the robustness of \algname compared to the conventional learning from experts' advice (\algname-conventional) in a complete malicious and random policies environment. When no good policy is available, \algname could recover from malicious advice and successfully approach the performance of the best classifier. In contrast, the conventional approach will be trapped in malicious advice. 90\% confident intervals are indicated in shades.}
        \label{fig:malicious-plots}
\end{figure}

\subsection{Outperformance over the best policy/expert}\label{app:outperform_best_expert}

We also observe that \algname does not stop at approaching the best policy or classifier performance. Sometimes, it even outperforms all the policies and classifiers, and \figref{fig:outperform_all} demonstrates such a case. To demonstrate the advantage of \algname, we create two variant versions of \algname: (1) \algname-MAX (\appref{app:CAMS_MAX}), (2) \algname-Random-Policy (\appref{app:CAMS_random_policy}). \algname-MAX and \algname-Random-Policy use the same algorithm as \algname in adversarial settings but have different model selection strategies for ablation study in the stochastic settings.

We evaluate the three algorithms on VERTEBRAL and HIV benchmarks in terms of (a) \textit{normal policies} (\figref{fig:outperform_all} Left), (b) \textit{classifiers} (\figref{fig:outperform_all} Middle), and (c) \textit{malicious and random policies} (\figref{fig:outperform_all} Right). In the normal policies column, we only compare the policies with regular policies giving helpful advice. In the classifier column, we compare them with the performance of classifiers only. In the malicious and random policies column, we compare them with unreasonable policies only.

\figref{fig:outperform_all} demonstrates that all three algorithms could outperform the malicious/random policies. However, \algname-Random-Policy does not outperform the best classifier while both \algname and \algname-MAX can on both benchmarks. \algname-MAX approaches the performance of the best policy but does not outperform the best policy on both benchmarks. Finally, perhaps surprisingly, \algname outperforms the best policy (Oracle) on both benchmarks and continues to approach the hypothetical, optimal policy (with 0 cumulative loss).

This surprising factor is contributed by the adaptive weighted policy of \algname, which adaptively creates a better policy by combining the advantage of each sub-optimal policy and classifier to reach the performance of the hypothetical, optimal policy (defined as $ \sum_{t=1}^T \min_{i \in [\PoliciesNum+\ModelsNum]}\widetilde{\loss}_{t,i}$). The second reason could be that the benchmark we created, or any real-world cases, will not be strictly in a stochastic setting (in which a single policy outperforms all others or has lower $\mu$ in every round). The weight policy strategy can make a better combination of advice for this case.

\begin{figure}[h]
\begin{subfigure}{1\textwidth}
    \centering
    \includegraphics[height=0.3cm, clip={0,0,0,0}]{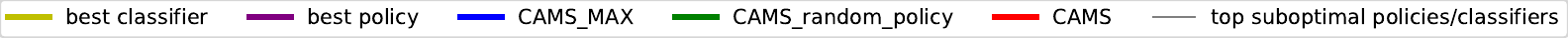}
\end{subfigure}
\rotatebox[origin=l]{90}{\quad \qquad  \scriptsize Cumulative loss}
\begin{subfigure}{.3\textwidth}
    \centering
        \includegraphics[height=3.8cm, clip={0,0,0,0}]{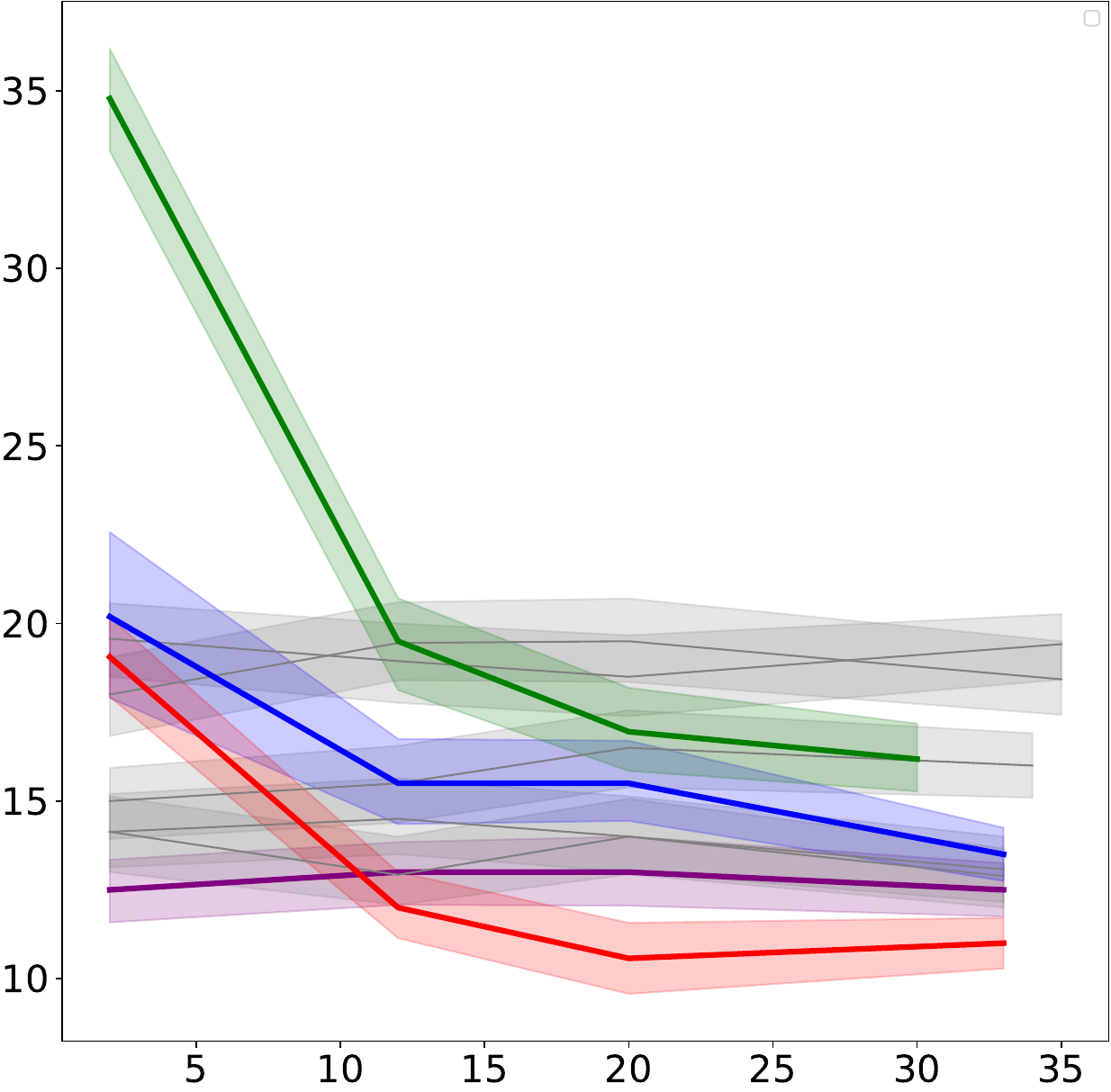}
\end{subfigure}
\begin{subfigure}{.3\textwidth}
    \centering
        \includegraphics[height=3.8cm, clip={0,0,0,0}]{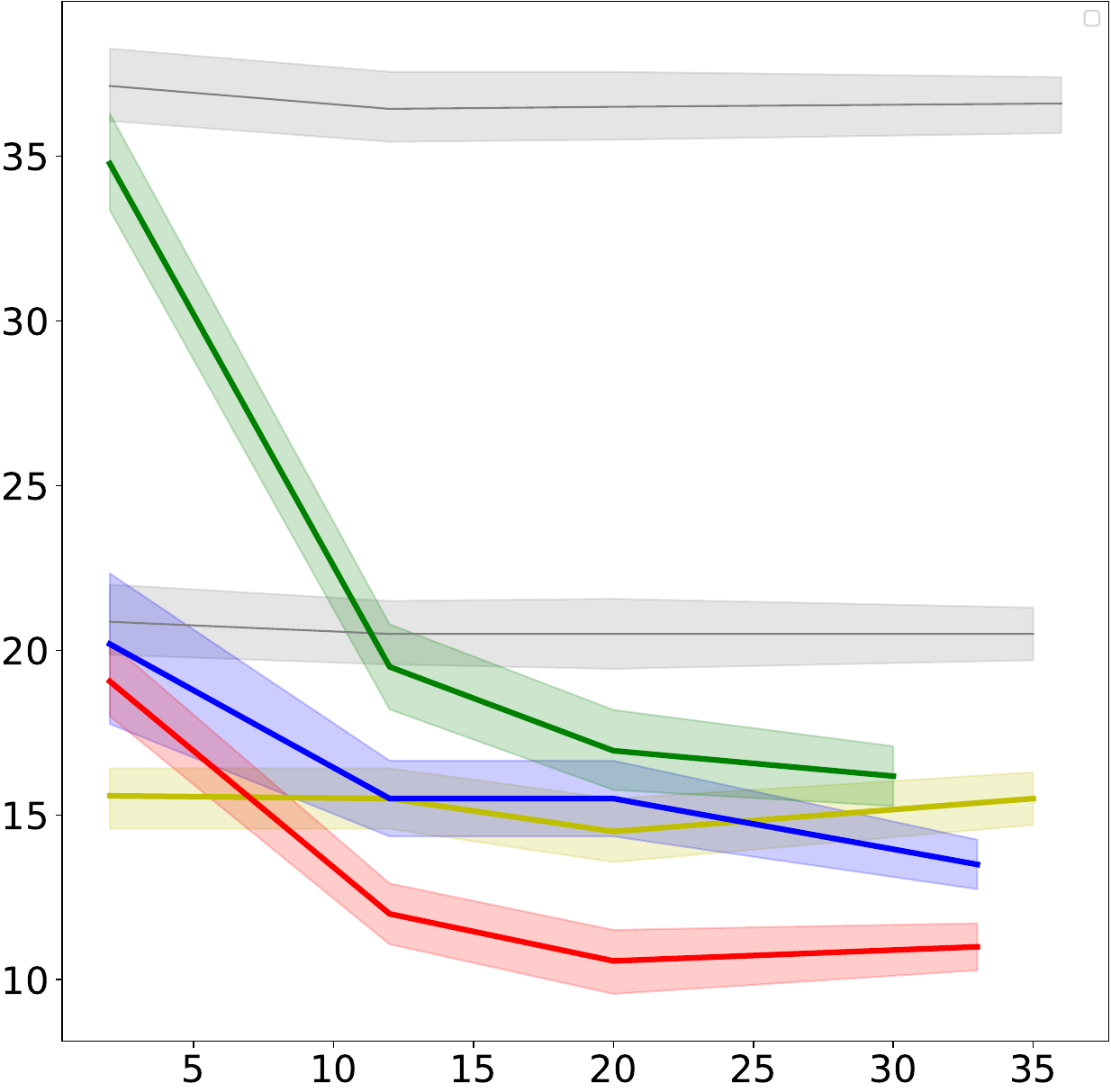}
\end{subfigure}
\begin{subfigure}{.3\textwidth}
    \centering
        \includegraphics[height=3.8cm, clip={0,0,0,0}]{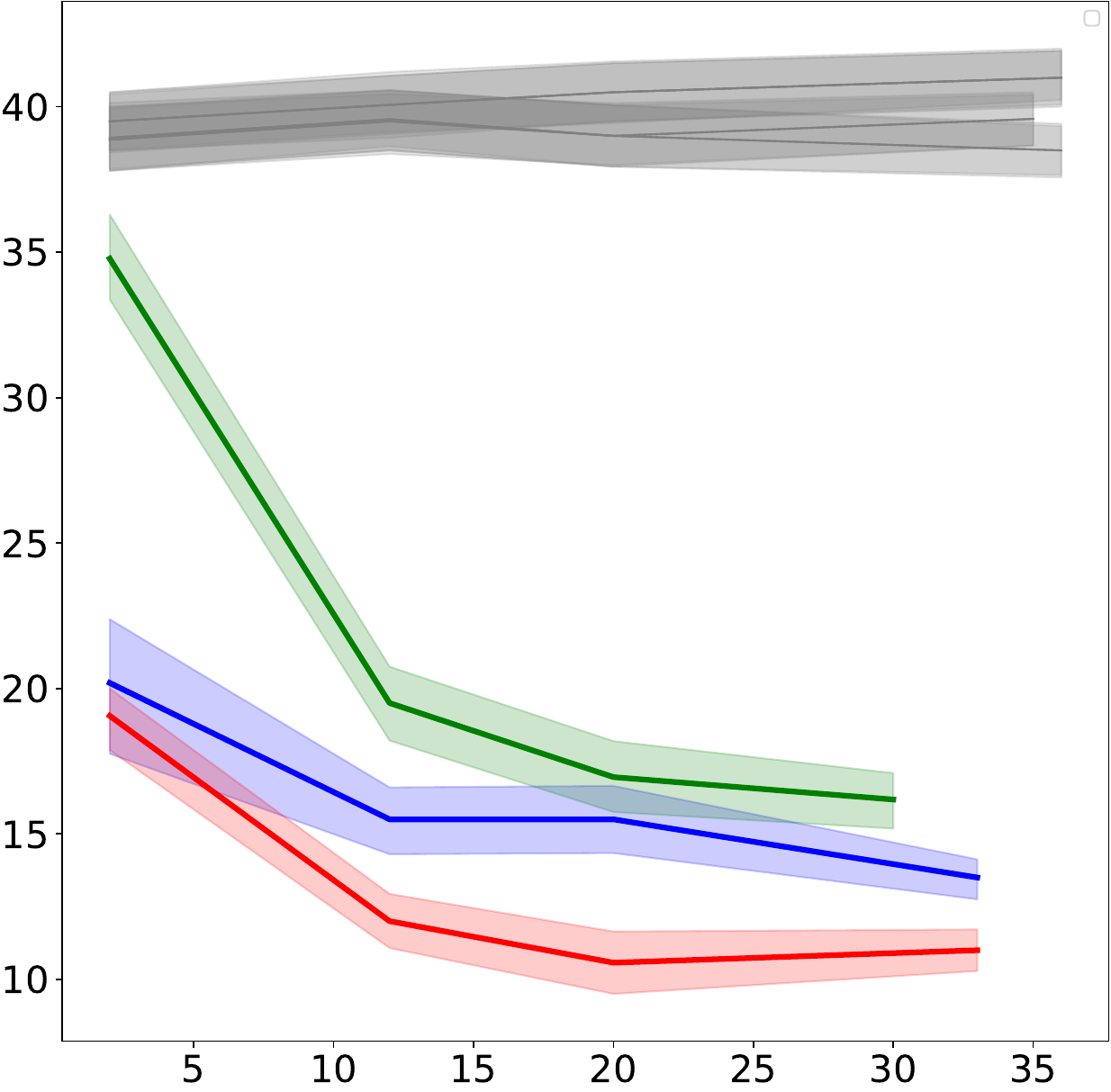}
\end{subfigure}
\rotatebox[origin=l]{90}{\small \qquad  VERTEBRAL}

\rotatebox[origin=l]{90}{\quad \quad \qquad  \scriptsize Cumulative loss}
\begin{subfigure}{.3\textwidth}
    \centering
        \includegraphics[height=4cm, clip={0,0,0,0}]{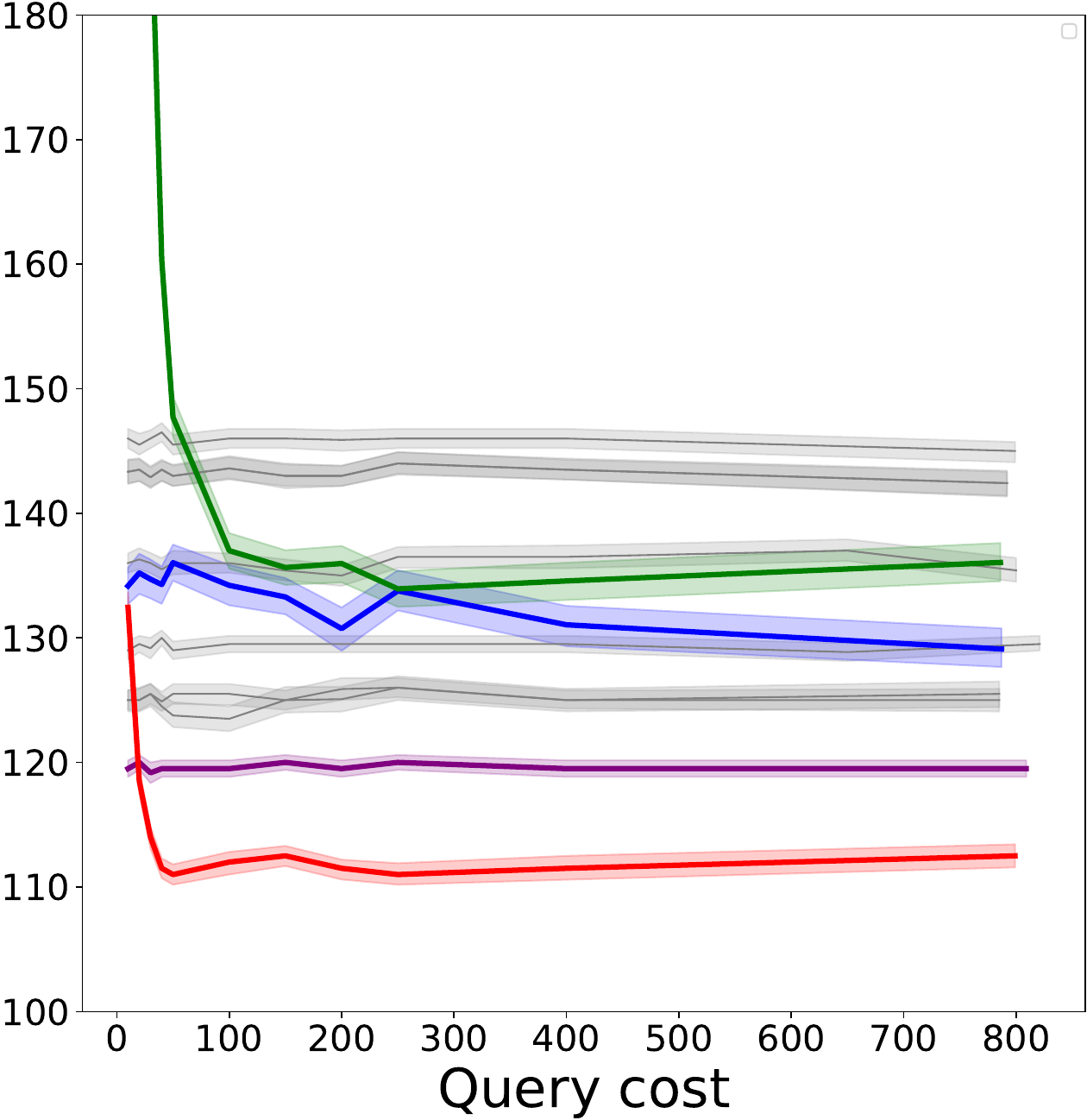}
    \caption{normal policies}
\end{subfigure}
\begin{subfigure}{.3\textwidth}
    \centering
        \includegraphics[height=4cm, clip={0,0,0,0}]{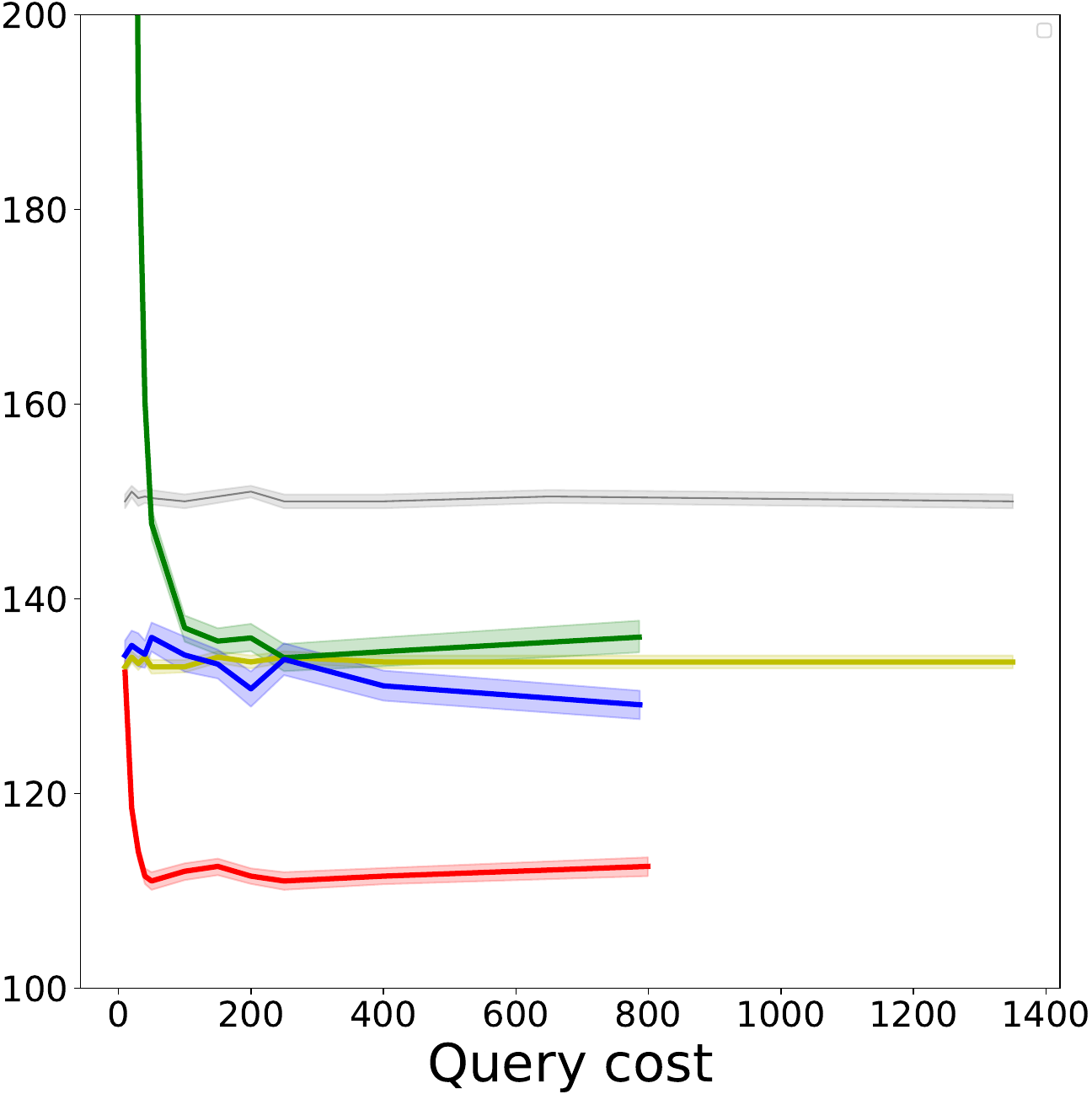}
    \caption{classifiers}
\end{subfigure}
\begin{subfigure}{.3\textwidth}
    \centering
        \includegraphics[height=4cm, clip={0,0,0,0}]{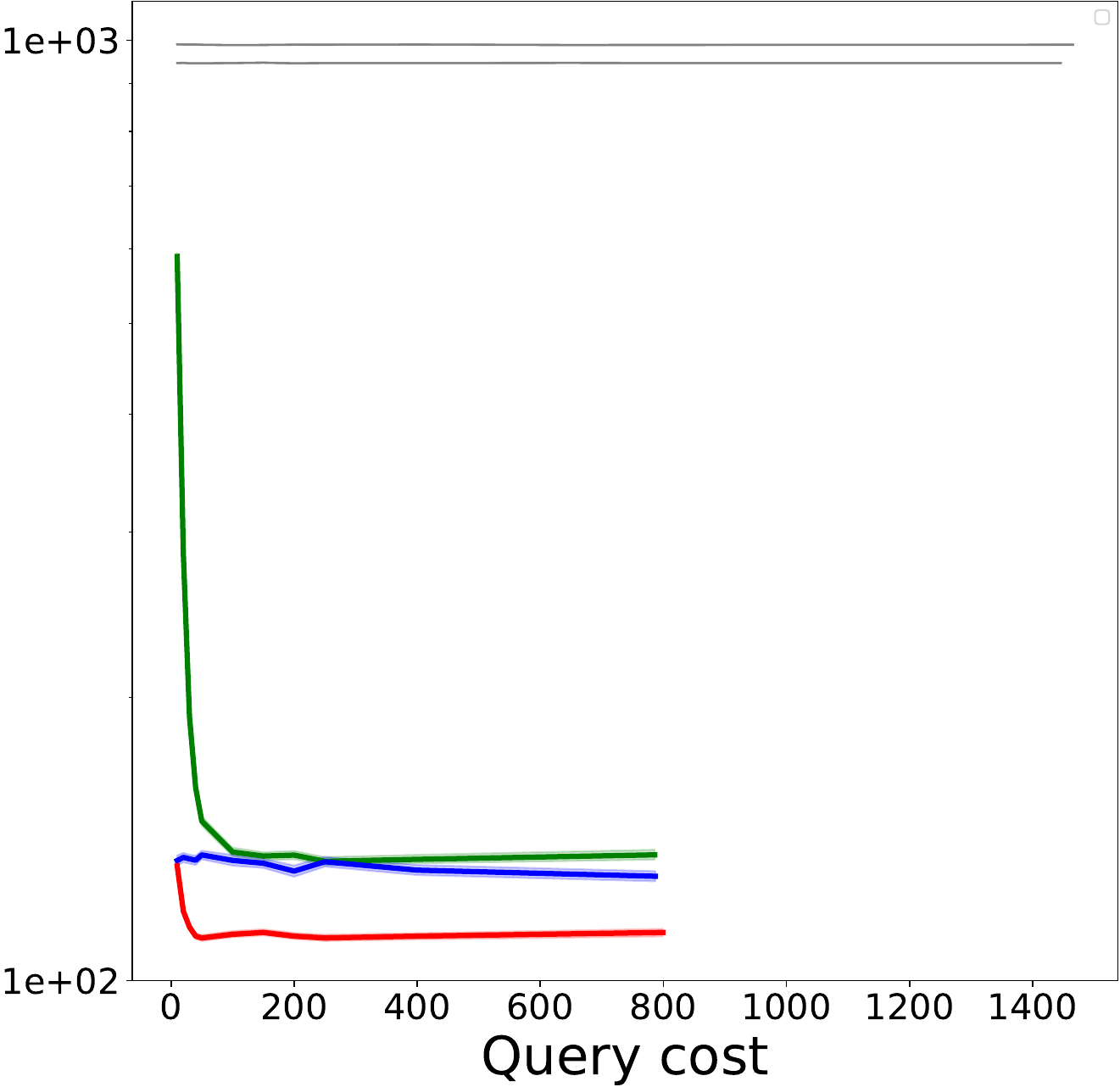}
    \caption{malicious and random policies}
\end{subfigure}
\rotatebox[origin=l]{90}{\qquad \qquad \qquad  \small HIV}
        \caption{Comparing \algname, \textsc{\algname-MAX} and \textsc{\algname-Random-Policy} with top policies and classifiers in the VERTEBRA and HIV benchmarks. They outperform all the malicious/random policies. Moreover, \algname and \textsc{CAMS-MAX} outperform the best classifier. Finally, only \algname even exceeds the best policy (Oracle) in both benchmarks and continues approaching the hypothetical, optimal policy (0 cumulative loss). 90\% confident intervals are indicated in shades. } 
        \label{fig:outperform_all}
\end{figure}

\subsection{The \algname-MAX algorithm}\label{app:CAMS_MAX}
\algname-MAX is a variant of \algname. In an adversarial setting, they share the same algorithm. However, in a stochastic setting, \algname-MAX gets the index $i^*$ of max value in the probability distribution of policy $\PoliciesDistB$, and selects the model with the max value in  $\policy_{i^*}\paren{\instance_t}$ to recommendation. The difference is marked as blue color in \figref{alg:CAMS_max_policy}.

\begin{figure}[h]
  \centering
  \scalebox{0.86}{
  \fbox{
  \begin{subfigure}[t]{0.6\textwidth}
\begin{algorithmic}[1]
  \State {\bfseries Input:} Models $\Models$, policies $\Policies^*$, \#rounds $T$, budget $\budget$
  \State Initialize loss $\tilde{\LossB}_{0}\leftarrow 0$; query cost $\queryCost_0 \leftarrow 0$ %
  \For{$t=1,2,...,T$}
  \State Receive $\instance_t$ 
  \State $\eta_t \leftarrow \textsc{SetRate}(t,\instance_t, \left|\Policies^*\right|)$
    \State Set $\PoliciesDist_{t,i} \propto \exp{\paren{-\eta_t\tilde{\Loss}_{t-1,i}}}~\forall \policyIndex \in |\Policies^*|$ 
  \State $j_t \leftarrow \algRecommend(\instance_t,\PoliciesDistB_t) $ %
  \State Output $\pd_{t,j_{t}} \sim f_{t,j_{t}}$ as the prediction for $\instance_t$ 
  \State Compute $\queryProb_t$ in \eqnref{eq:query-probability} 
  \State Sample $\mathbb{\QueryIndicator}_t \sim \textrm{Ber}\paren{\queryProb_t}$ %
  \If{$\mathbb{\QueryIndicator}_t = 1$ and $\queryCost_t\leq {\budget}$} 
  \State Query the label $\clabel_t$
  \State $\queryCost_t \leftarrow \queryCost_{t-1} + 1$ 
  \State Compute ${\lossB}_t$: $\loss_{t,j}=\mathbb{I}\curlybracket{\pd_{t,j}\neq \clabel_t}, \forall {j\in{[\left|\Models\right|]}}$ 
    \State Estimate model loss:  $\hat{\loss}_{t,j}=\frac{\loss_{t,j}}{\queryProb_t}, \forall j \in{[\left|\Models\right|]}$
    \State $\tilde{\lossB}_{t}$: %
    $\tilde{\loss}_{t,i} \leftarrow \langle \policy_i(\instance_t),  \hat{\loss}_{t,j} \rangle, \forall i\in \left[|\Policies^*|\right]$
    \State $\tilde{\LossB}_{t} = \tilde{\LossB}_{t-1} + \tilde{\lossB}_{t}$
    \Else
    \State $\tilde{\LossB}_{t} = \tilde{\LossB}_{t-1}$
    \State $\queryCost_t \leftarrow \queryCost_{t-1}$ 
    \EndIf
  \EndFor
\end{algorithmic}
    \end{subfigure}
  \begin{subfigure}[t]{0.50\textwidth}
  \begin{subfigure}[t]{1\textwidth}
\hfill
\begin{algorithmic}[1]
\setcounter{ALG@line}{20}
    \Procedure{SetRate}{$t,\instance_t,\numTotalPolicies$}
        \If{\stochastic}
        \State     $\eta_t=\sqrt{\frac{\ln{{\numTotalPolicies}}}{t}}$ 
        \EndIf
        \If{\adversarial}
        \State %
        Set $\minpgap{t}$ as in adversarial section
        \State  $\eta_t=\sqrt{\qlb + \frac{ \minpgap{t}}{c^2\ln c }}\cdot \sqrt{\frac{\ln{{\numTotalPolicies}}}{T}}$ 
        \EndIf
        \State \Return $\eta_t$
    \EndProcedure
\end{algorithmic}
    \end{subfigure}
    \\~
    \\~
 \begin{subfigure}[t]{1\textwidth}
\hfill
\begin{algorithmic}[1]
\setcounter{ALG@line}{28}
    \Procedure{\algRecommend}{$\instance_t,\PoliciesDistB_t$}
        \If{\stochastic}
        \State  {\color{blue}$i_t \leftarrow  \maxind{\PoliciesDistB_t}$}
        \State  {\color{blue}$j_t \leftarrow  \maxind{\policy_{i_t}\paren{\instance_t}}$}
        \EndIf
        \If{\adversarial}
        \State  $i_t \sim \PoliciesDistB_t$ 
        \State  $j_t \sim \policy_{i_t}\paren{\instance_t}$ 
        \EndIf
        \State \Return $j_t$
    \EndProcedure
    \end{algorithmic}
        \end{subfigure} 
    \\~
    \end{subfigure}
  }}
  \caption{The \algname-MAX Algorithm}
    \label{alg:CAMS_max_policy}
\end{figure}

\subsection{The CAMS-Random-Policy algorithm}\label{app:CAMS_random_policy}

\begin{figure}[!h]
  \centering
  \scalebox{0.86}{
  \fbox{
  \begin{subfigure}[t]{0.6\textwidth}
\begin{algorithmic}[1]
  \State {\bfseries Input:} Models $\Models$, policies $\Policies^*$, \#rounds $T$, budget $\budget$
  \State Initialize loss $\tilde{\LossB}_{0}\leftarrow 0$; query cost $\queryCost_0 \leftarrow 0$ %
  \For{$t=1,2,...,T$}
  \State Receive $\instance_t$ 
  \State $\eta_t \leftarrow \textsc{SetRate}(t,\instance_t, \left|\Policies^*\right|)$
    \State Set $\PoliciesDist_{t,i} \propto \exp{\paren{-\eta_t\tilde{\Loss}_{t-1,i}}}~\forall \policyIndex \in |\Policies^*|$ 
  \State $j_t \leftarrow \algRecommend(\instance_t,\PoliciesDistB_t) $ %
  \State Output $\pd_{t,j_{t}} \sim f_{t,j_{t}}$ as the prediction for $\instance_t$ 
  \State Compute $\queryProb_t$ in \eqnref{eq:query-probability} 
  \State Sample $\mathbb{\QueryIndicator}_t \sim \textrm{Ber}\paren{\queryProb_t}$ %
  \If{$\mathbb{\QueryIndicator}_t = 1$ and $\queryCost_t\leq {\budget}$} 
  \State Query the label $\clabel_t$
  \State $\queryCost_t \leftarrow \queryCost_{t-1} + 1$ 
  \State Compute ${\lossB}_t$: $\loss_{t,j}=\mathbb{I}\curlybracket{\pd_{t,j}\neq \clabel_t}, \forall {j\in{[\left|\Models\right|]}}$ 
    \State Estimate model loss:  $\hat{\loss}_{t,j}=\frac{\loss_{t,j}}{\queryProb_t}, \forall j \in{[\left|\Models\right|]}$
    \State $\tilde{\lossB}_{t}$: %
    $\tilde{\loss}_{t,i} \leftarrow \langle \policy_i(\instance_t),  \hat{\loss}_{t,j} \rangle, \forall i\in \left[|\Policies^*|\right]$
    \State $\tilde{\LossB}_{t} = \tilde{\LossB}_{t-1} + \tilde{\lossB}_{t}$
    \Else
    \State $\tilde{\LossB}_{t} = \tilde{\LossB}_{t-1}$
    \State $\queryCost_t \leftarrow \queryCost_{t-1}$ 
    \EndIf
  \EndFor
\end{algorithmic}
    \end{subfigure}
  \begin{subfigure}[t]{0.50\textwidth}
  \begin{subfigure}[t]{1\textwidth}
\hfill
\begin{algorithmic}[1]
\setcounter{ALG@line}{20}
    \Procedure{SetRate}{$t,\instance_t,\numTotalPolicies$}
        \If{\stochastic}
        \State     $\eta_t=\sqrt{\frac{\ln{{\numTotalPolicies}}}{t}}$ 
        \EndIf
        \If{\adversarial}
        \State %
        Set $\minpgap{t}$ as in adversarial section
        \State  $\eta_t=\sqrt{\qlb + \frac{ \minpgap{t}}{c^2\ln c }}\cdot \sqrt{\frac{\ln{{\numTotalPolicies}}}{T}}$ 
        \EndIf
        \State \Return $\eta_t$
    \EndProcedure
\end{algorithmic}
    \end{subfigure}
    \\~
    \\~
 \begin{subfigure}[t]{1\textwidth}
\hfill
\begin{algorithmic}[1]
\setcounter{ALG@line}{28}
    \Procedure{\algRecommend}{$\instance_t,\PoliciesDistB_t$}
        \If{\stochastic}
        \State  {\color{blue}$i_t \sim \PoliciesDistB_t$}
        \State  {\color{blue}$j_t \leftarrow  \maxind{\policy_{i_t}\paren{\instance_t}}$}
        \EndIf
        \If{\adversarial}
        \State  $i_t \sim \PoliciesDistB_t$ 
        \State  $j_t \sim \policy_{i_t}\paren{\instance_t}$ 
        \EndIf
        \State \Return $j_t$
    \EndProcedure
    \end{algorithmic}
        \end{subfigure} 
    \\~
    \end{subfigure}
  }}
  \caption{The \algname-Random-Policy Algorithm}
    \label{alg:CAMS_random_policy}
\end{figure}

\algname-Random-Policy is a variant of \algname. It shares the same algorithm with \algname in an adversarial environment. However, it uses a random sampling policy method in a stochastic setting. It randomly samples the policy from the probability distribution of policy $\PoliciesDistB$, and selects the model with max value in $\policy_{i^*}\paren{\instance_t}$ to recommendation. The difference is marked as blue color in \figref{alg:CAMS_random_policy}.

\subsection{Maximal queries from experiments}\label{app:exp:max_query}
\tabref{app:emperical-query-complexity} in this section summarizes the maximum query cost for a given data stream (of fixed total size), %
with its associated cumulative loss for all baselines (exclude Oracle) on all benchmarks in experiment section.
The result in this table is slightly different from the query complexity curves of \figref{fig:exp:results} (Middle). The curve in \figref{fig:exp:results} (Middle) takes the average value, while the table takes the maximal value from a fixed number of simulations. \algname wins over all baselines (other than Oracle) in terms of query cost on CIFAR10, DRIFT, and VERTEBRAL benchmarks.  \algname outperforms all baselines in terms of cumulative loss on DRIFT, VERTEBRAL, and HIV benchmarks. In particular, \algname outperforms both cumulative loss and query cost on the DRIFT and VERTEBRAL benchmarks.

\begin{table}[h]
\centering
\scalebox{1}{
    \begin{tabular}{l l l l l l}
\toprule
\textbf{Algorithm}
& \textbf{CIFAR10}
& \textbf{DRIFT} 
& \textbf{VERTEBRAL}
& \textbf{HIV}\\
\midrule
\emph{Max queries, Cumulative loss}
& \emph{\makecell[l]{1200, 10000}}
& \emph{\makecell[l]{2000, 3000}}
& \emph{\makecell[l]{80, 80}}
& \emph{\makecell[l]{2000, 4000}}\\
\hline
RS
& \makecell[l]{1200, 2916}
& \makecell[l]{2000, 766}
& \makecell[l]{80, 19} 
& \makecell[l]{2000, 143}\\
\hline
QBC
& \makecell[l]{1200, 2857}
& \makecell[l]{1904, 771}
& \makecell[l]{72, 20} 
& \makecell[l]{2000, 139}\\
\hline
IWAL
& \makecell[l]{1200, 2854}
& \makecell[l]{2000, 760}
& \makecell[l]{80, 19} 
& \makecell[l]{690, 140}\\
\hline
MP
& \makecell[l]{1200, 2885}
& \makecell[l]{493, 803}
& \makecell[l]{33, 25}
& \makecell[l]{153, 148}\\
\hline
CQBC
& \makecell[l]{1200, 2284}  
& \makecell[l]{1900, 744}
& \makecell[l]{68, 13} 
& \makecell[l]{2000, 124}\\
\hline
CIWAL
& \makecell[l]{1200, 2316}
& \makecell[l]{2000, 746}
& \makecell[l]{80, 12} 
& \makecell[l]{690, 124}\\
\hline
\textbf{\algname}
& \makecell[l]{\textbf{348}, 2348} 
& \makecell[l]{\textbf{251, 710}} 
& \makecell[l]{\textbf{32, 11}} 
& \makecell[l]{782, \textbf{112}}\\
\bottomrule
\end{tabular}
}
\caption{Maximal queries from experiments}
\label{app:emperical-query-complexity}
\end{table}

\subsection{Query complexity}\label{app:exp:query_complexity_experiment}

{To achieve the same level of prediction accuracy (measured by average cumulative loss over a fixed number of rounds), \algname incurs less than 10\% of the label cost of the best competing baselines on CIFAR10 (10K examples), and 68\% the cost on VERTEBRAL (see \figref{app:exp:query_complexity});
\figref{app:exp:query_complexity} \footnote{{We also consider variants for each algorithm (other than Random and Oracle) where we scale the query probabilities based on the early-phase performance and observe similar behavior. See \appref{app:scaling_param} for the corresponding results.}} and \tabref{app:emperical-query-complexity} also demonstrate the compelling effectiveness of \algname's query strategy outperforming all baselines in terms of query cost in VERTEBRAL, DRIFT, and CIFAR10 benchmarks, which is consistent with our query complexity bound in \thmref{thm:stochastic-query-complexity}.
}

\begin{figure*}[ht!]
\begin{subfigure}{1\textwidth}
    \centering
    \includegraphics[height=0.3cm,  clip={0,0,0,0}]{./figures/legend_horizontal.png}
\end{subfigure}
\rotatebox[origin=l]{90}{\quad \quad \scriptsize Number of queries}
\begin{subfigure}{.24\textwidth}
    \centering
    \includegraphics[height=2.7cm, width=3.4cm,  clip={0,0,0,0}]{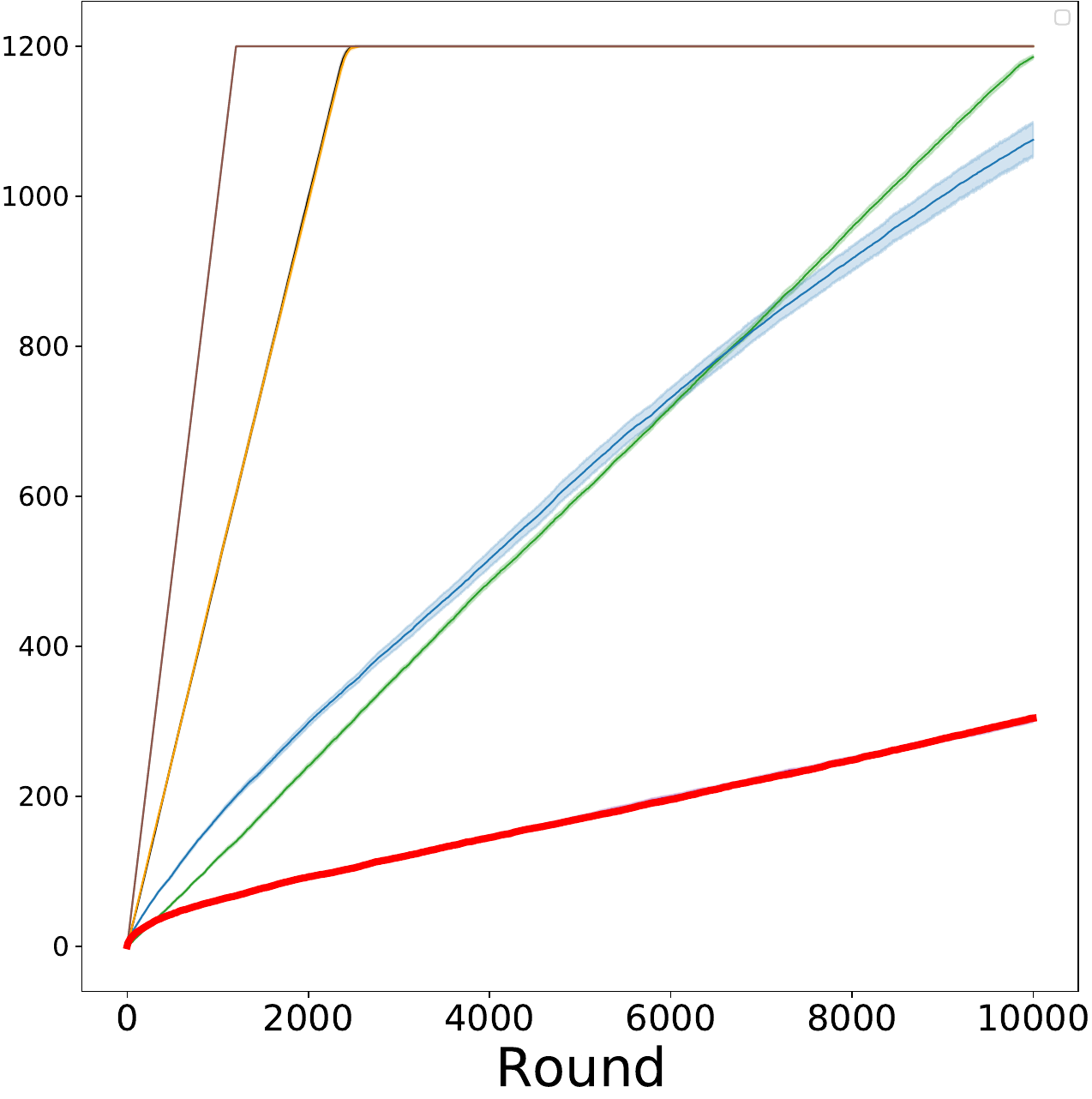}
   \caption{CIFAR10}
\end{subfigure}
\begin{subfigure}{.24\textwidth}
    \centering
    \includegraphics[height=2.7cm, width=3.4cm,  clip={0,0,0,0}]{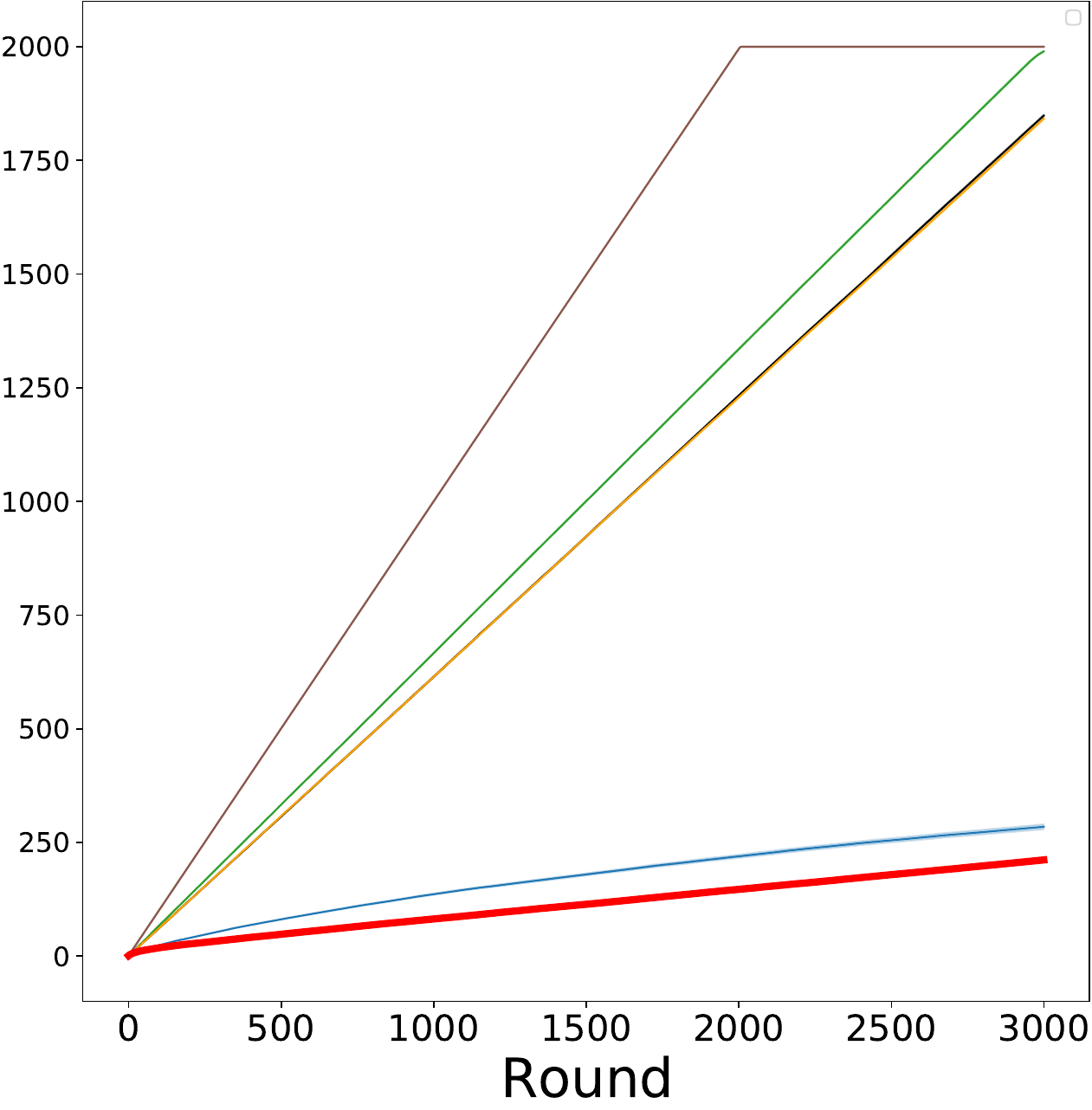}
   \caption{DRIFT}
\end{subfigure}
\begin{subfigure}{.24\textwidth}
    \centering
    \includegraphics[height=2.7cm, width=3.4cm,  clip={0,0,0,0}]{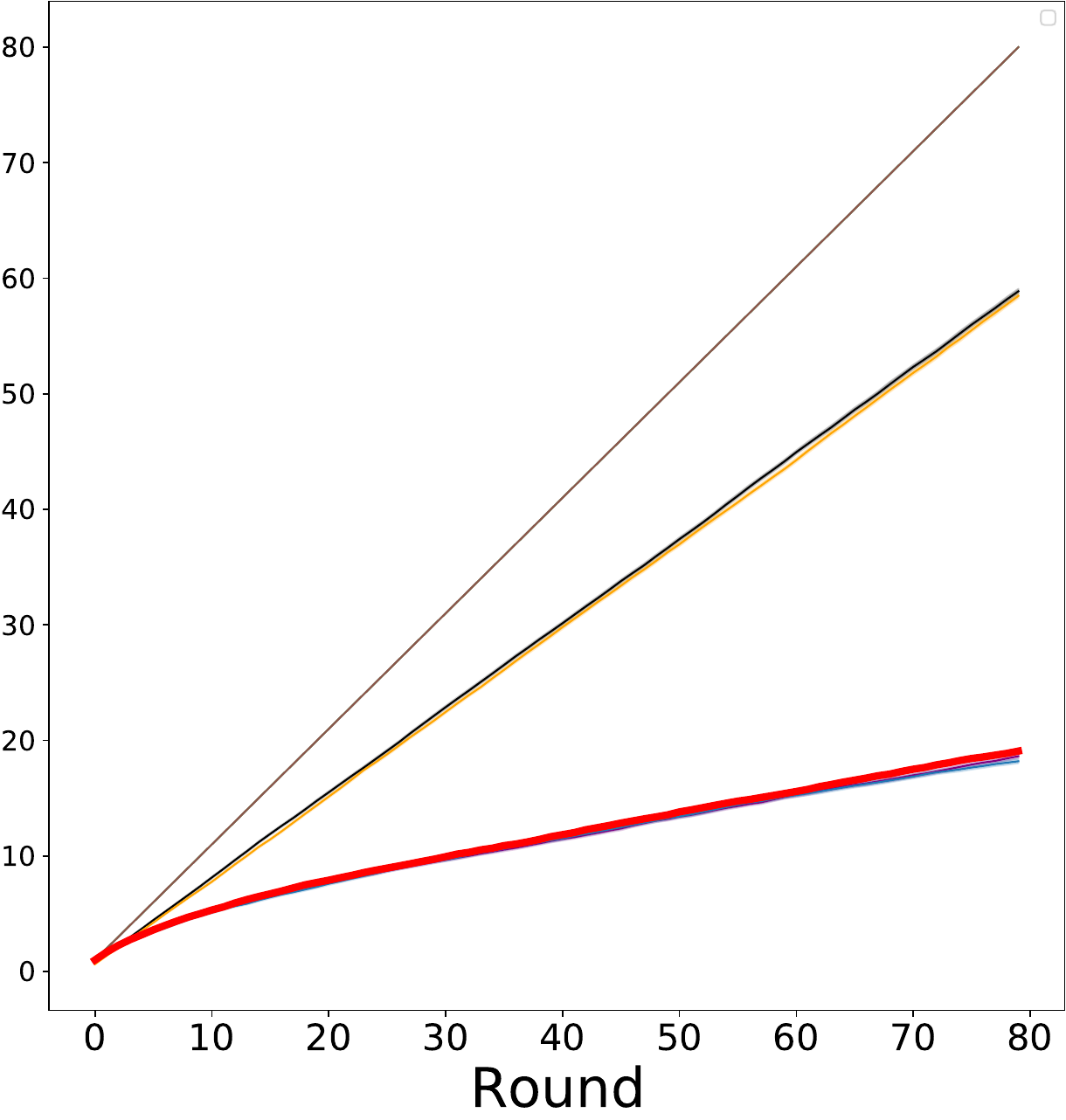}
   \caption{VERTEBRAL}
\end{subfigure}
\begin{subfigure}{.24\textwidth}
    \centering
    \includegraphics[height=2.7cm, width=3.4cm, clip={0,0,0,0}]{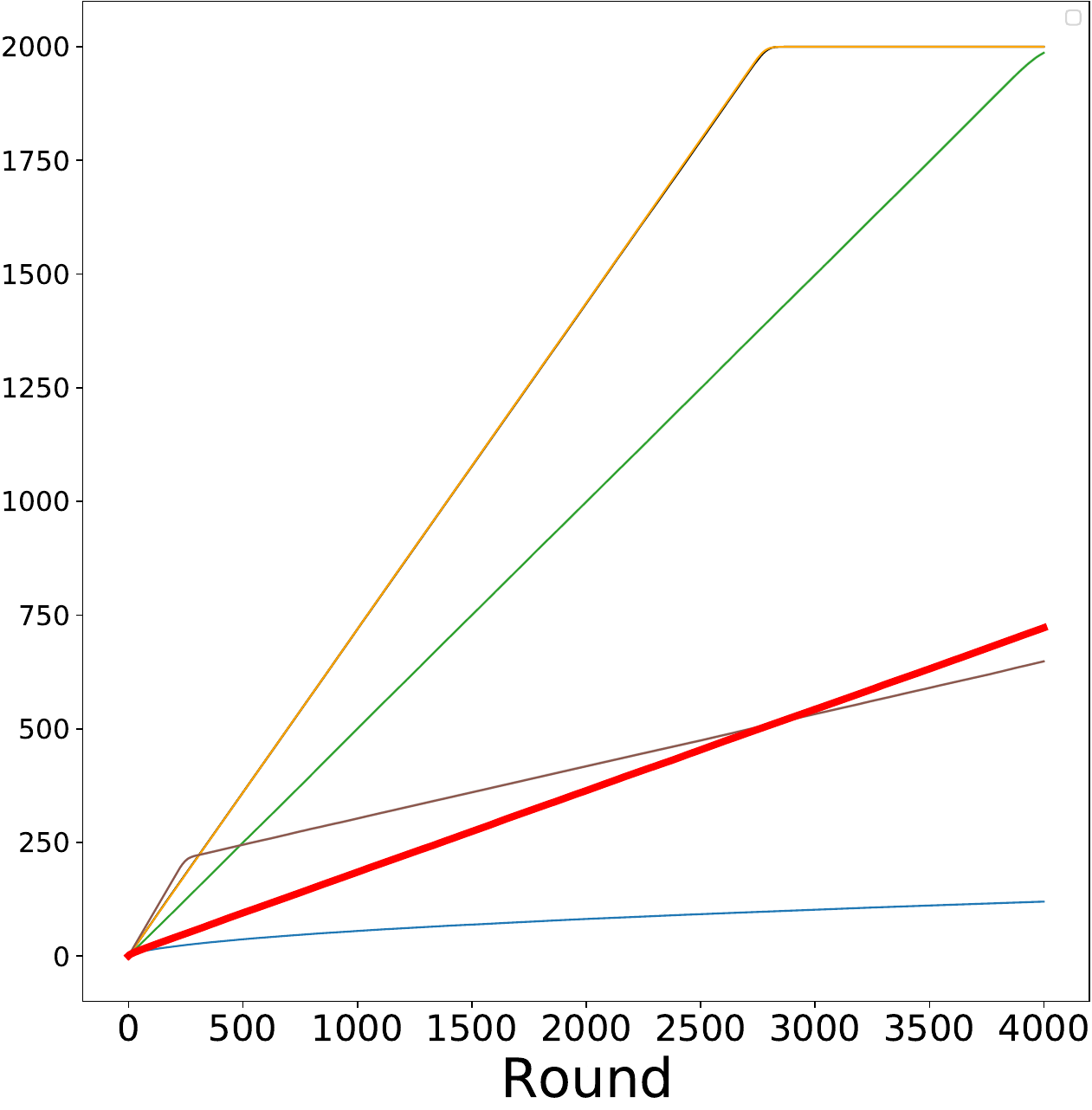}
   \caption{HIV}
\end{subfigure}
    \centering
        \caption{Comparing \algname with 7 baselines on 4 diverse benchmarks in terms of query complexity (Number of queries). 
        \algname outperforms all baselines 
        for a fixed number of rounds $T$ (where $T=10000, 3000, 80, 4000$ from left to right) and maximal query cost $B$ (where $B=1200,2000,80,2000$ from left to right). \textbf{Algorithms}: 4 contextual \{Oracle, CQBC, CIWAL, CAMS\} and 4 non-contextual baselines \{RS, QBC, IWAL, MP\} are included (see Section \pref{main:baseines}). 90\% confident interval are indicated in shades. %
        }
  \label{app:exp:query_complexity}
\end{figure*}

\subsection{{Fine-tuning the query probabilities for stochastic streams}}\label{app:scaling_param}
{

For the experimental results we reported in the main paper, we consider a streaming setting where the data arrives online in an \emph{arbitrary order} and \emph{arbitrary length}. Therefore, for both \algname and the baselines, we used the exact off-the-shelf query criteria as described in experiment setup section
\emph{without fine-tuning the query probabilities}, which could be otherwise desirable in certain scenarios (e.g. for stochastic streams, where the query probability can be further optimized).

In this section, we consider such scenarios, and conduct an additional set of experiments to further demonstrate the performance of \algname assuming stochastic data streams. 
Given the stream length $T$ and query budget $b$, we may optimize each algorithm by scaling their query probabilities, so that each algorithm allocates its query budget to the top $b$ informative labels in the entire online stream based on its own query criterion. Note that in practice, finding the exact scaling parameter is infeasible, as we do not know the online performance unless we observe the entire data stream. While it is challenging to determine the scaling factor for each algorithm under the adversarial setting, one can effectively estimate the scaling factor for stochastic streams, where the context arrives \textit{i.i.d.}. 

Concretely, we use the early budget to decide the scaling parameter in our following evaluation: Firstly, we use a small fraction (i.e. ${T}/{10}$) of the online stream and see how much queries $b_{\text{early}}$ each algorithm consumed. Then we calculate the scaling parameter $s=\frac{(b-b_{\text{early}})}{T-T/{10}}\cdot \frac{T/10}{b_{\text{early}}}$ and multiply the scaling factor with the query probability of each algorithm for the remaining $\frac{9}{10}\cdot{T}$ rounds.  %
The results in \figref{scaling-plots} demonstrate that \algname still outperforms all the baselines (excluding Oracle) when all algorithms select the top $b$ data of the whole online stream to query. 
The improvement of CAMS over the baseline approaches \emph{does not differ much} between the two versions (with or without scaling) of the experiments as shown in the bottom plots of \figref{fig:exp:results} and \figref{scaling-plots}. 

For a head-to-head comparison between the bottom plots of \figref{fig:exp:results} and \figref{scaling-plots}, note that the total number of rounds stays the same for DRIFT ($T=3000$), VERTEBRAL ($T=80$), and HIV ($T=4000$); while we used half the rounds and half the maximal budget for CIFAR10 ($T=5000$) for the version with scaling. Roughly speaking, the cumulative regret plots for the baselines were "streched out" to cover the full allocated budget after scaling, but we do not observe a significant difference in terms of the absolute gain in terms of the cumulative loss. Another way to read the difference between the two plots is to compare the cumulative losses at the budget range where all algorithms were not cut off early: e.g., for DRIFT, when Query Cost is 250, the cumulative losses for the competing algorithm stay roughly the same under the two evaluation scenarios.

} %

\begin{figure}[h]
\begin{subfigure}{1\textwidth}
    \centering
    \includegraphics[height=0.3cm, clip={0,0,0,0}]{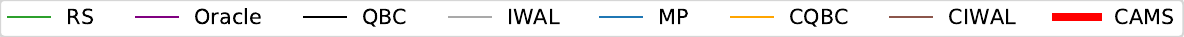}
\end{subfigure}
\rotatebox[origin=l]{90}{\qquad \qquad \scriptsize Cumulative loss}
\begin{subfigure}{.24\textwidth}
    \centering
        \includegraphics[height=3.2cm, clip={0,0,0,0}]{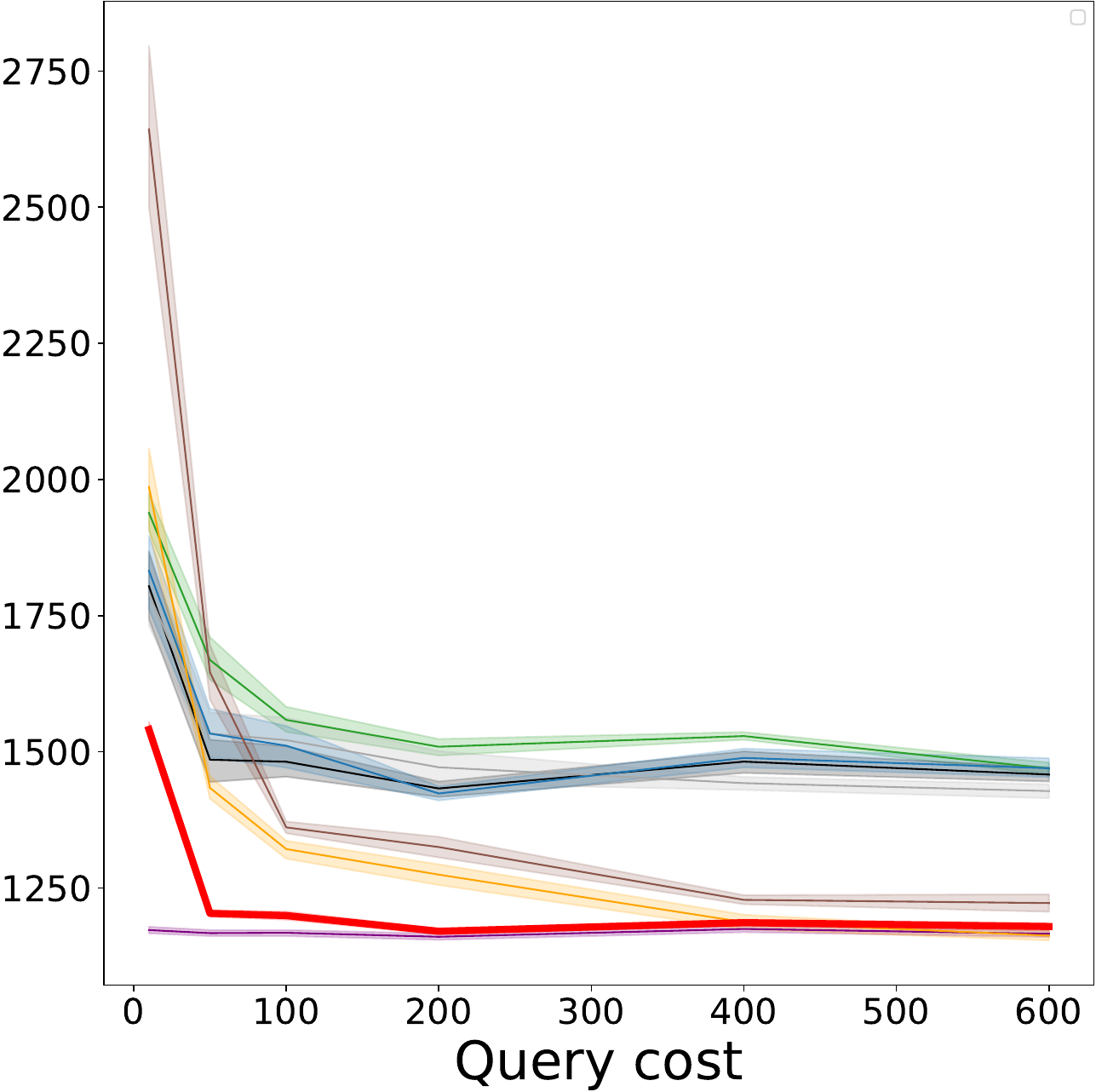}
    \caption{CIFAR10}
\end{subfigure}
\begin{subfigure}{.24\textwidth}
    \centering
        \includegraphics[height=3.3cm, clip={0,0,0,0}]{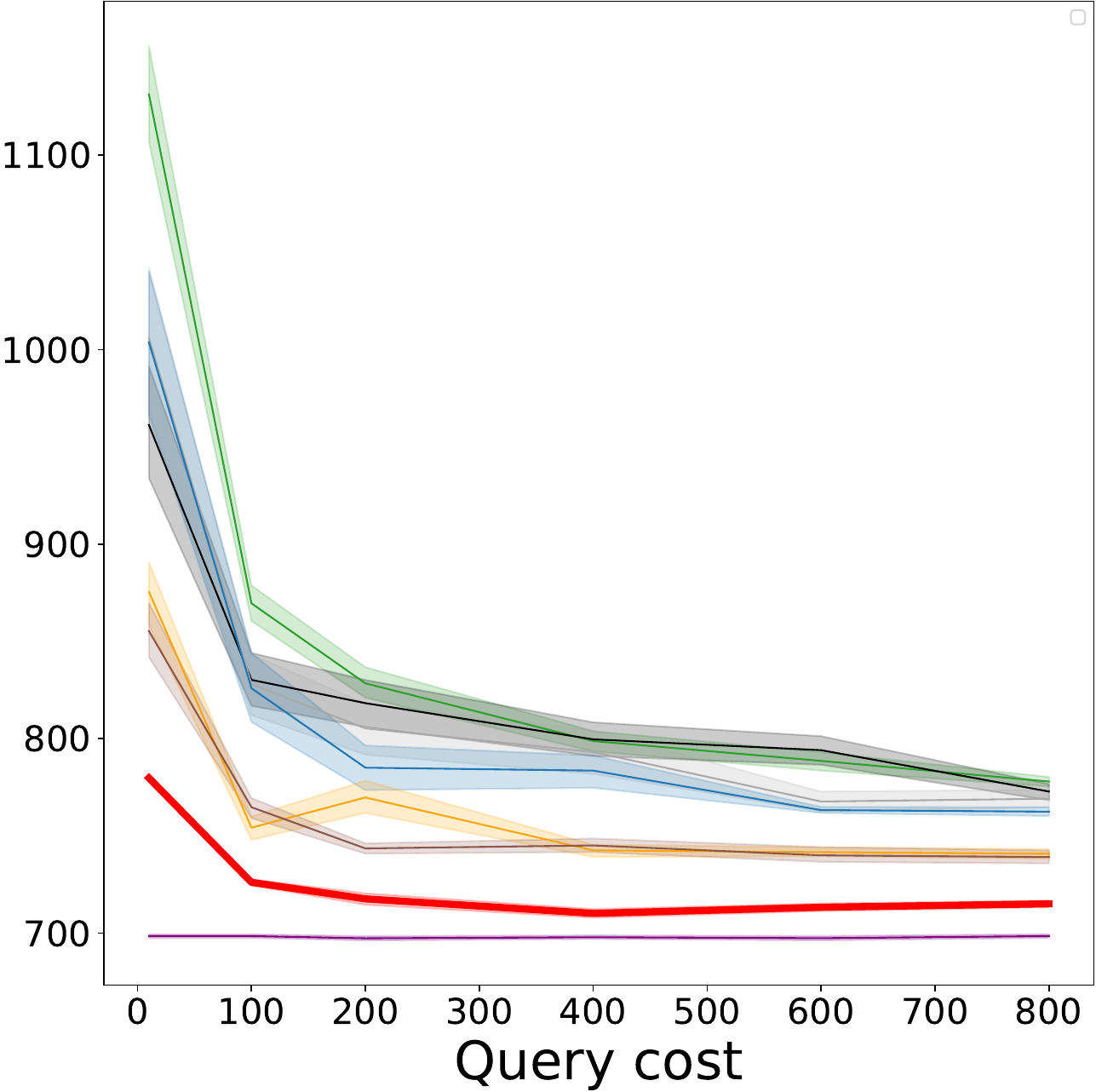}
    \caption{DRIFT}
\end{subfigure}
\begin{subfigure}{.24\textwidth}
    \centering
        \includegraphics[height=3.3cm, clip={0,0,0,0}]{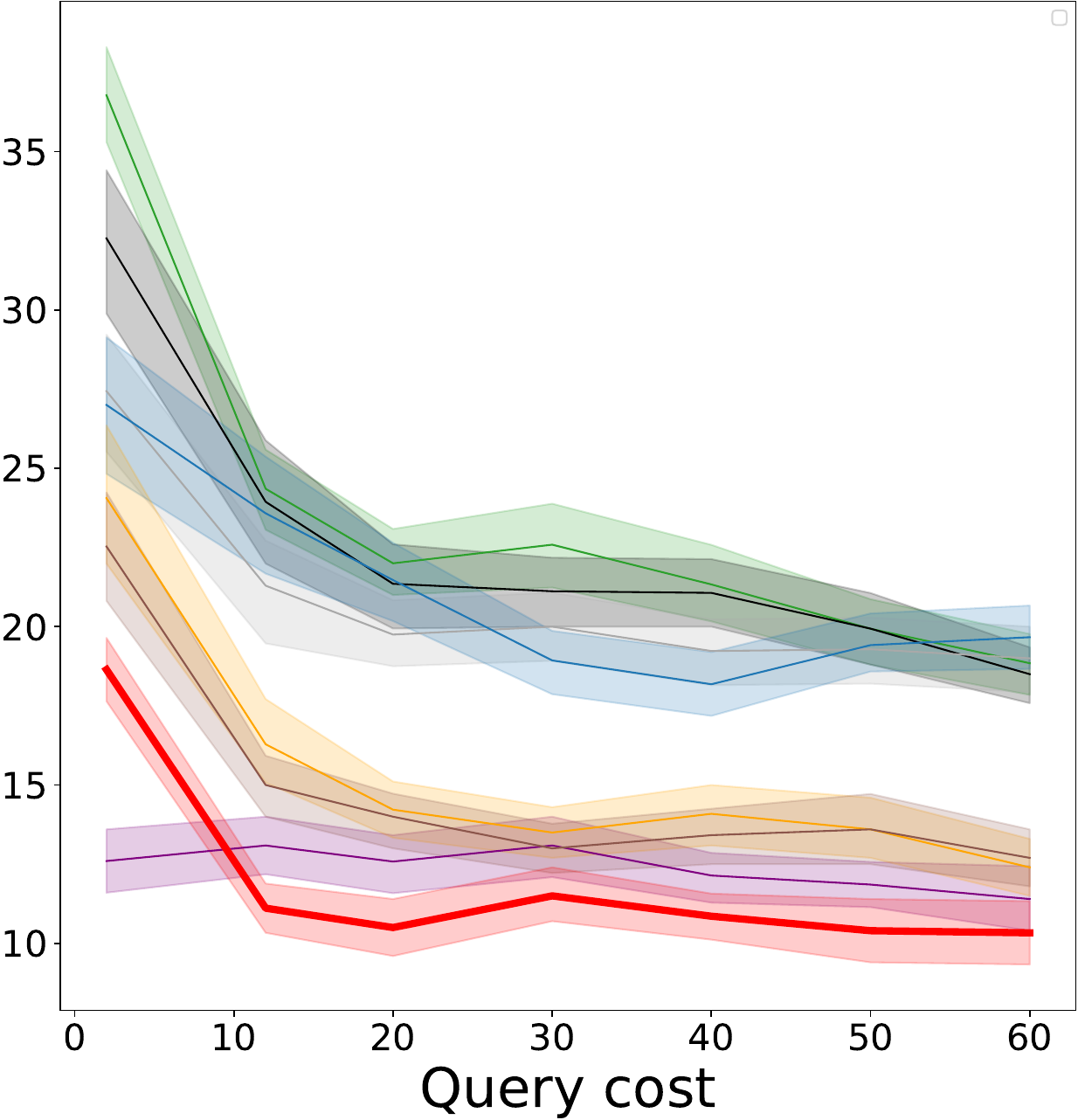}
    \caption{VERTEBRAL}
\end{subfigure}
\begin{subfigure}{.24\textwidth}
    \centering
        \includegraphics[height=3.3cm, clip={0,0,0,0}]{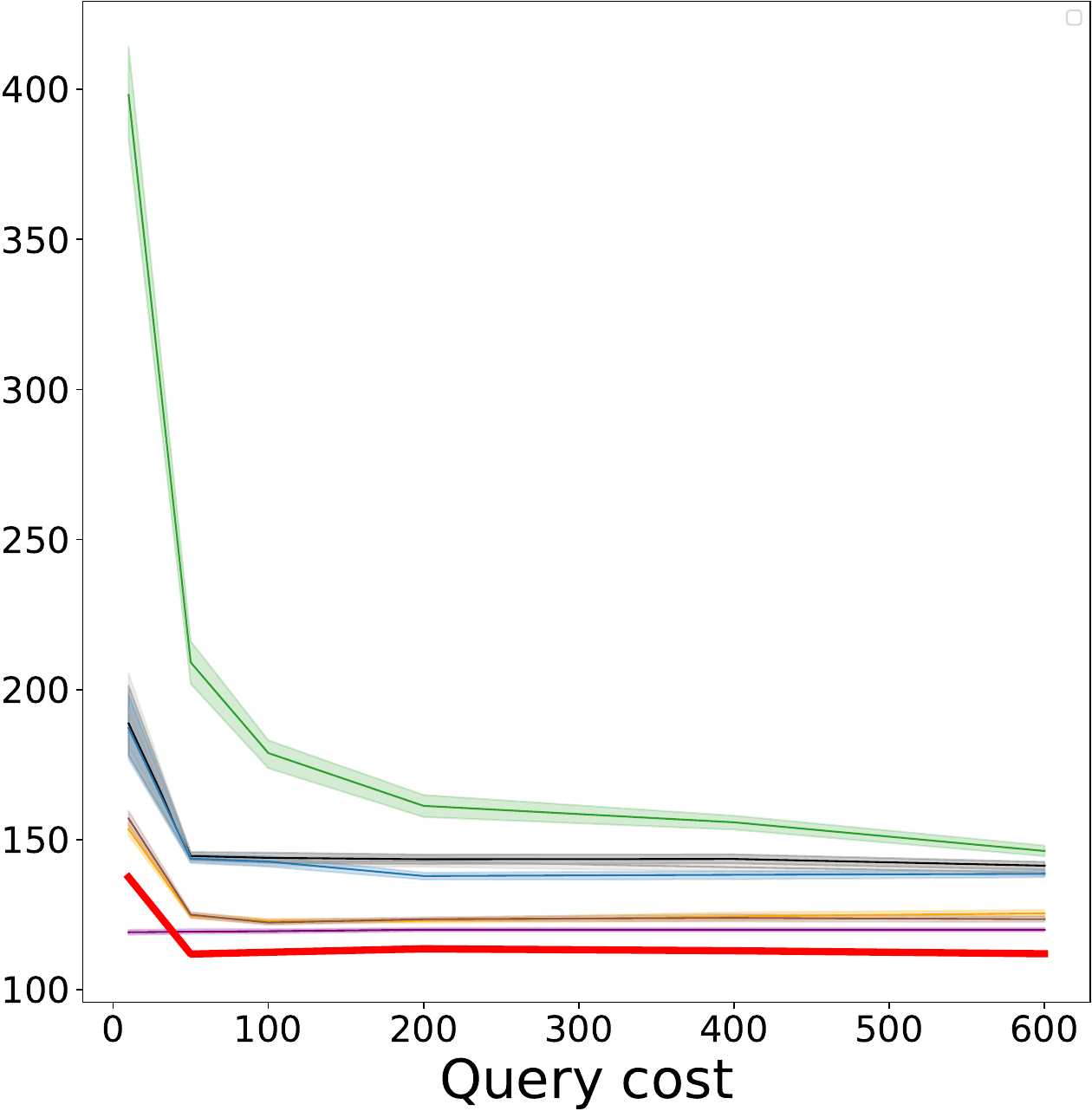}
    \caption{HIV}
\end{subfigure}
      \caption{Comparing \algname with 7 model selection baselines on 4 diverse benchmarks in terms of cost effectiveness {after applying the scaling parameter to each algorithm}. CAMS outperforms all baselines (excluding Oracle). Performance of cumulative loss by increasing the query cost, for a fixed number of rounds $T$ (where $T=5000, 3000, 80, 4000$ from left to right) and maximal query cost $B$ (where $B=600,800,60,600$ from left to right). \textbf{Algorithms}: 4 contextual \{Oracle, CQBC, CIWAL, CAMS\} and 4 non-contextual baselines \{RS, QBC, IWAL, MP\} are included (see Section \pref{main:baseines}). 90\% confident interval are indicated in shades.}
      \label{scaling-plots}
\end{figure}

\subsection{{Ablation study on the active query strategy }}\label{app:sample_efficiency}
{In this section, we compare the performance of \algname and its non-active variant (\algname-nonactive), which queries the label for each incoming data point. As shown in \figref{fig:app:sample-efficiency}, \algname performs equally well or better than \algname-nonactive, even though it queries significantly less data. Surprisingly, on the DRIFT dataset, \algname significantly outperforms \algname-nonactive, even when using less than 10 percent of the query budget (\figref{fig:app:sample-efficiency:drift}). This demonstrates that \algname selectively choose the data to query to maximal optimize policy improvement, while \algname queries all data points, regardless of their usefulness or noise, which hampers policy improvement and convergence.}

\begin{figure*}[h]
\begin{subfigure}{1\textwidth}
    \centering
    \includegraphics[height=0.3cm,  clip={0,0,0,0}]{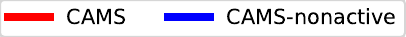}
\end{subfigure}
\rotatebox[origin=l]{90}{\quad \scriptsize Number of queries}
\begin{subfigure}{.24\textwidth}
    \centering
    \includegraphics[height=2.7cm, width=3.4cm, clip={0,0,0,0}]{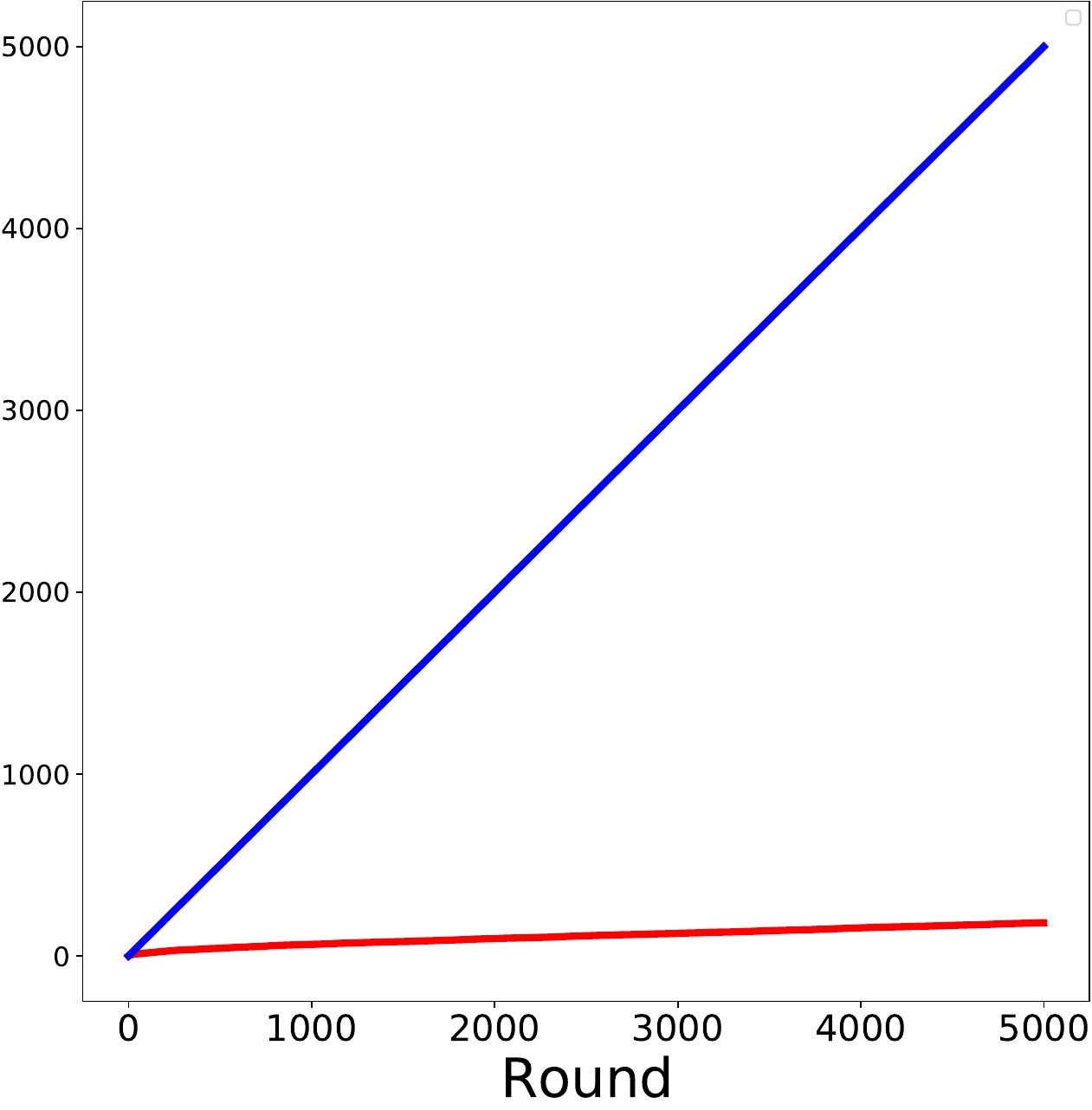}
\end{subfigure}
\begin{subfigure}{.24\textwidth}
    \centering
    \includegraphics[height=2.7cm, width=3.4cm, clip={0,0,0,0}]{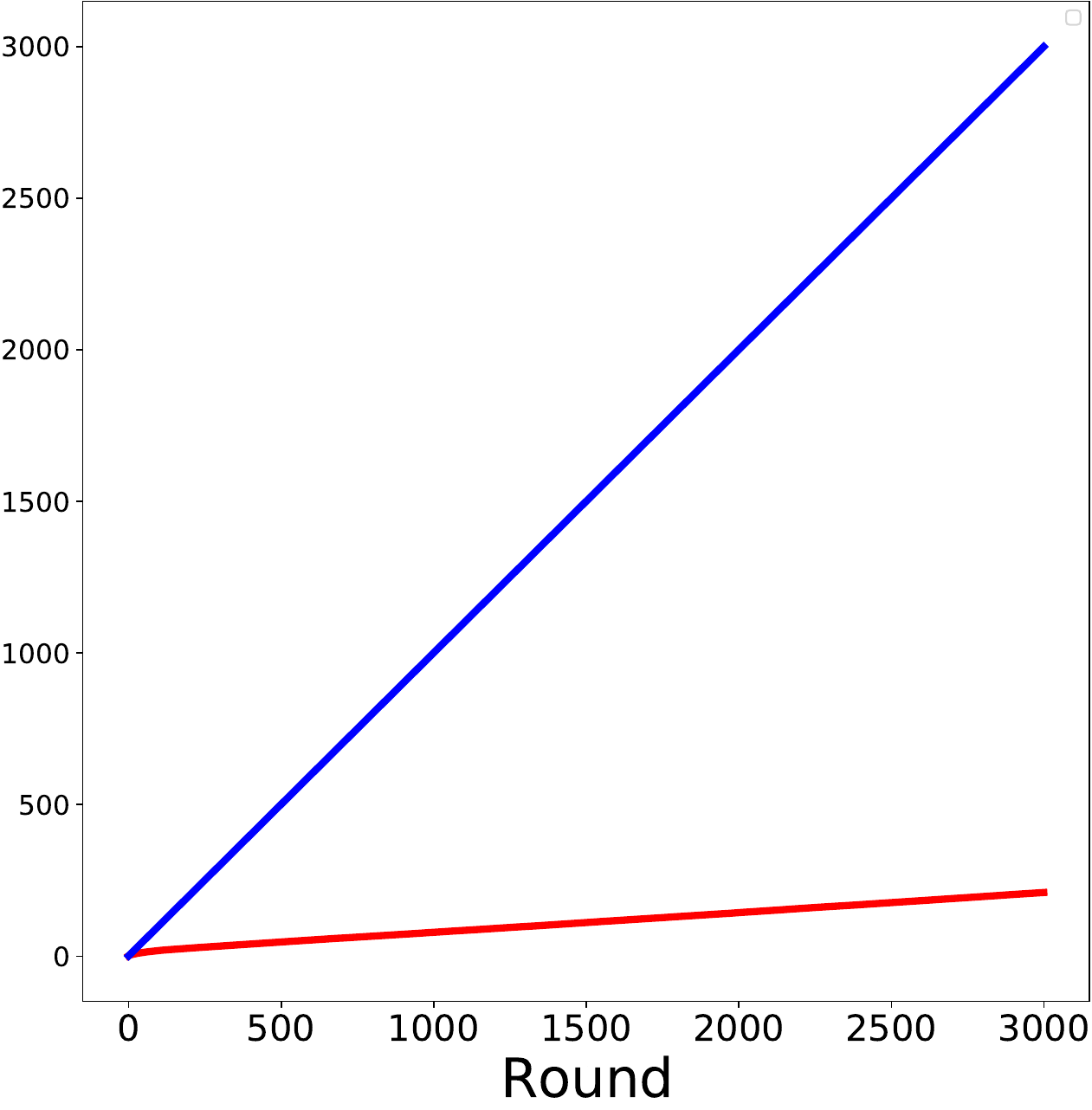}
\end{subfigure}
\begin{subfigure}{.24\textwidth}
    \centering
    \includegraphics[height=2.7cm, width=3.4cm, clip={0,0,0,0}]{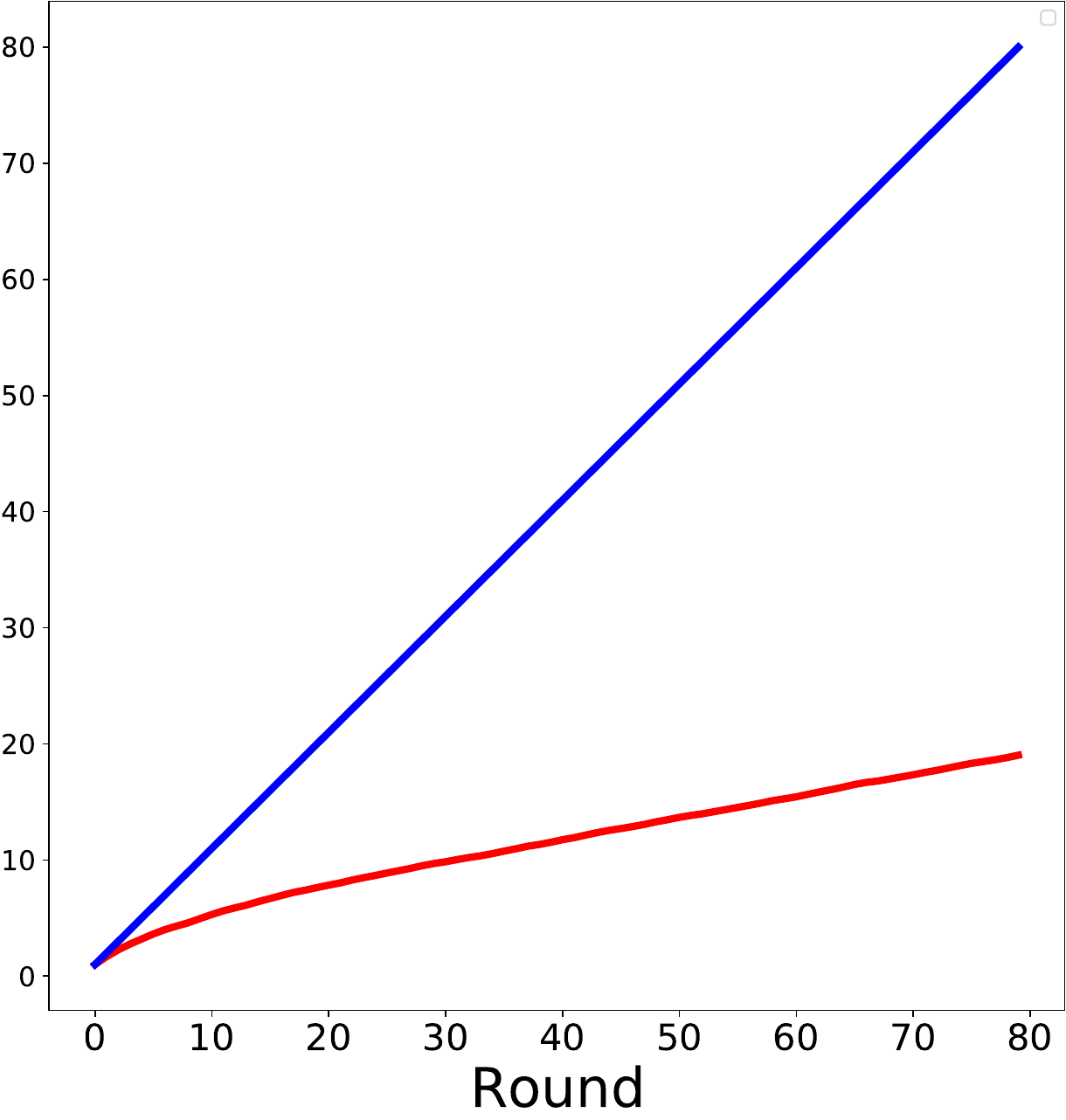}
\end{subfigure}
\begin{subfigure}{.24\textwidth}
    \centering
    \includegraphics[height=2.7cm, width=3.4cm,clip={0,0,0,0}]{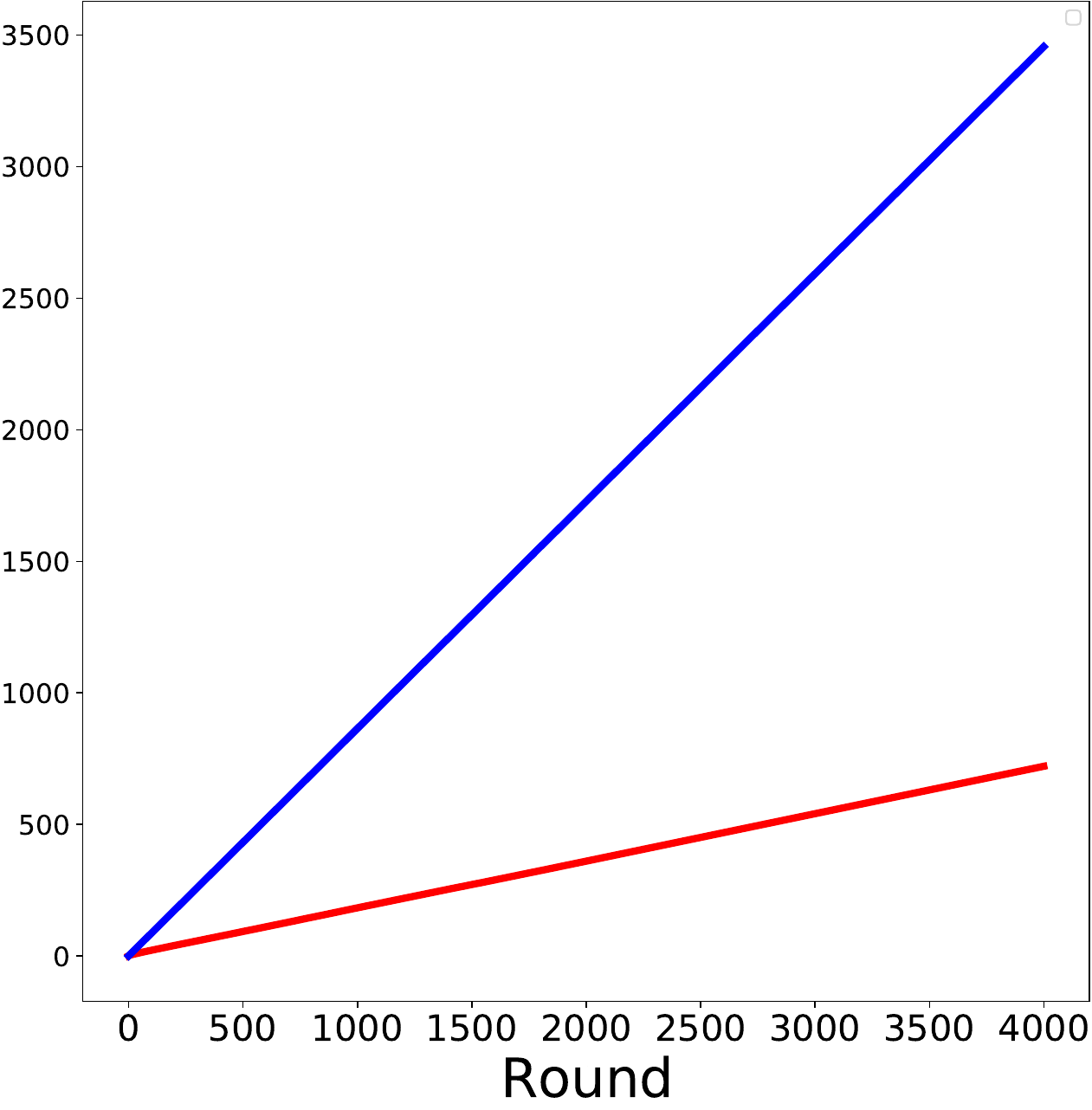}
\end{subfigure}
\rotatebox[origin=l]{90}{\qquad  \scriptsize Cumulative loss}
\begin{subfigure}{.24\textwidth}
    \centering
    \includegraphics[height=2.7cm, width=3.4cm, clip={0,0,0,0}]{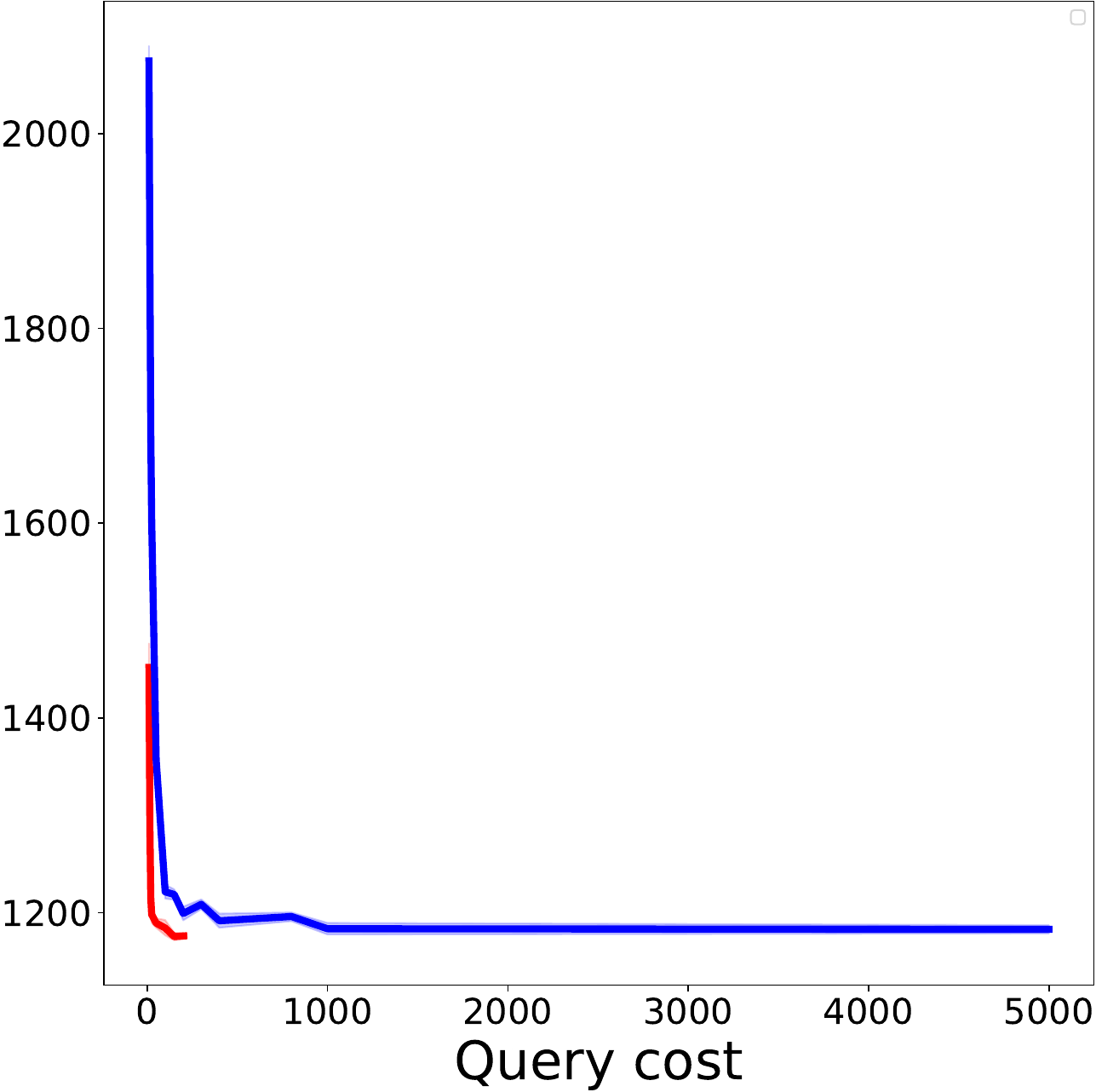}
    \caption{CIFAR10}
\end{subfigure}
\begin{subfigure}{.24\textwidth}
    \centering
    \includegraphics[height=2.7cm, width=3.4cm, clip={0,0,0,0}]{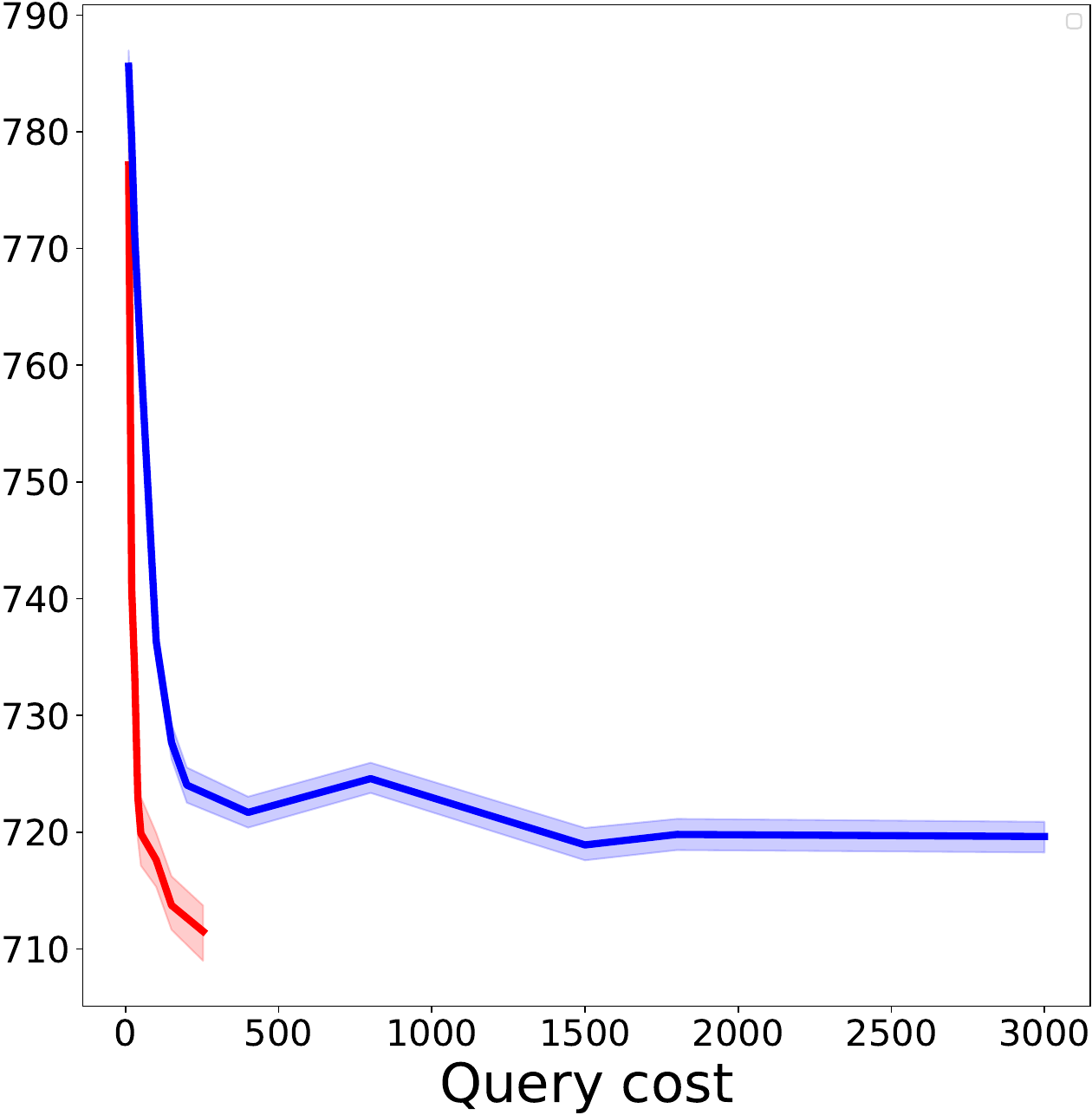}
    \caption{DRIFT}\label{fig:app:sample-efficiency:drift}
\end{subfigure}
\begin{subfigure}{.24\textwidth}
    \centering
    \includegraphics[height=2.7cm, width=3.4cm, clip={0,0,0,0}]{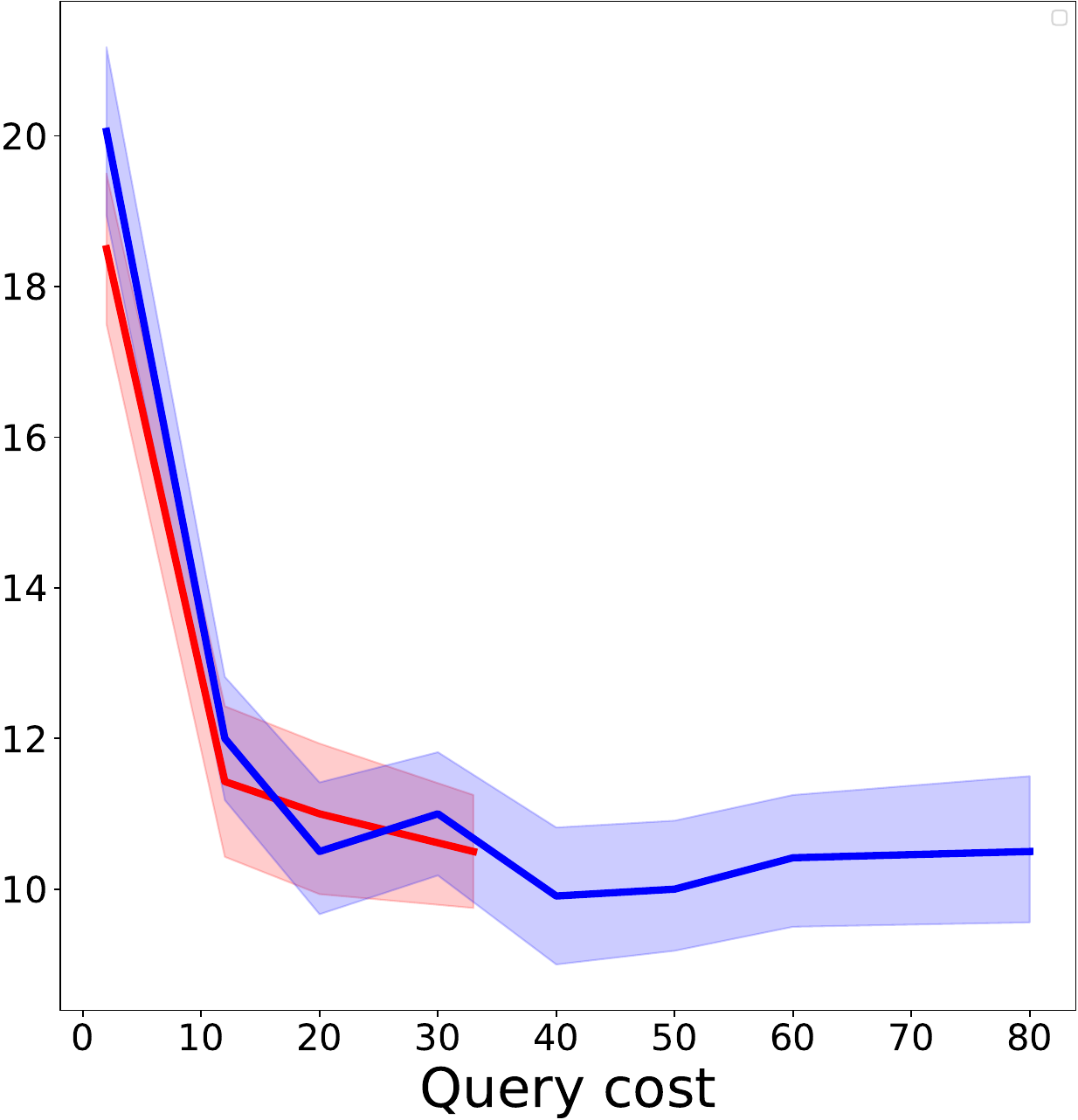}
    \caption{VERTEBRAL}
\end{subfigure}
\begin{subfigure}{.24\textwidth}
    \centering
    \includegraphics[height=2.7cm, width=3.4cm, clip={0,0,0,0}]{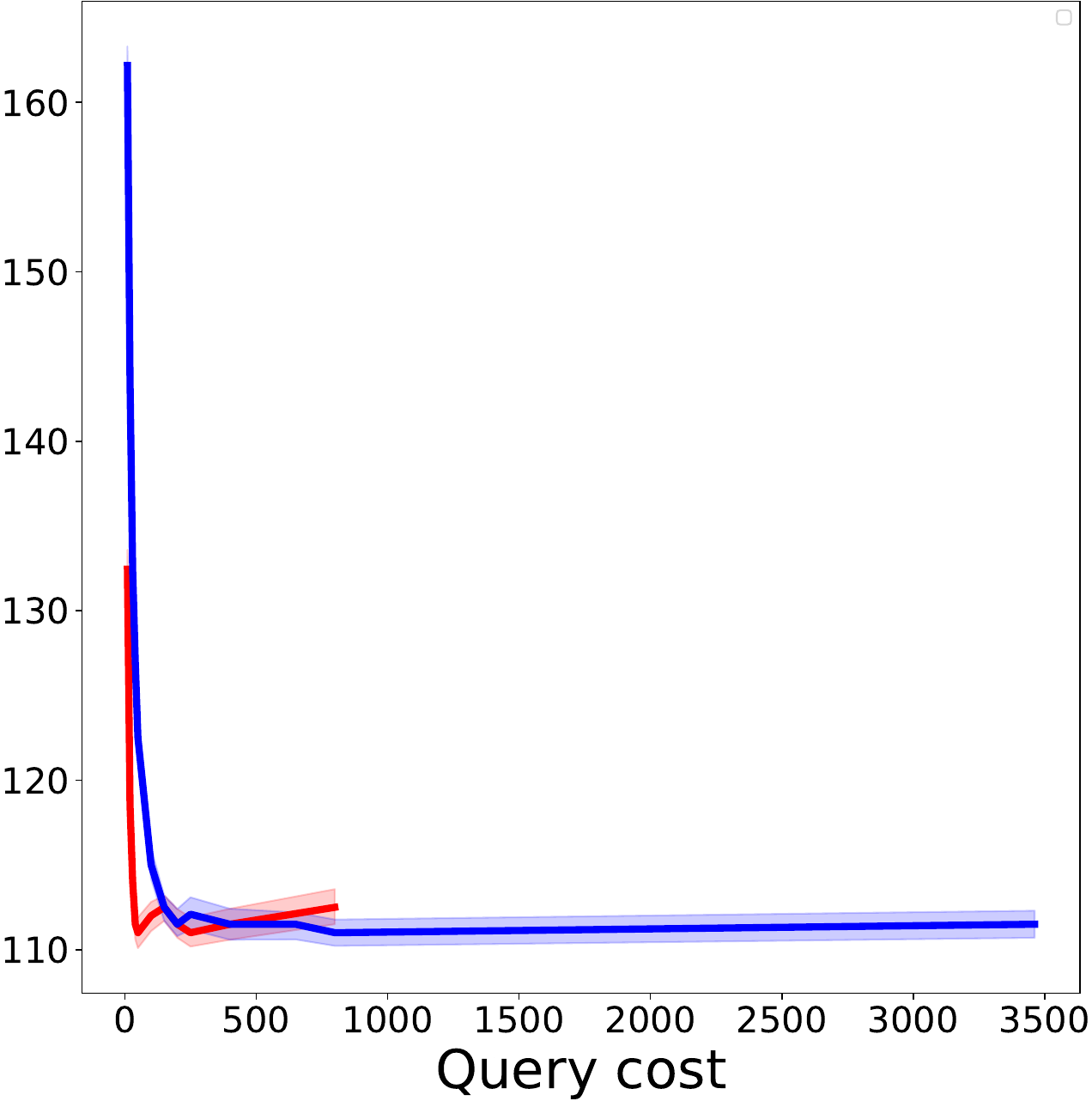}
    \caption{HIV}
\end{subfigure}
    \centering
        \caption{
        {
        Comparing \algname (in red) with \algname-nonactive (in blue) on 4 diverse benchmarks in terms of query complexity, and cost effectiveness.  \algname outperforms or performs equally well to \algname-nonactive with much less queried labels for all benchmarks. \textbf{(Top)} Number of queries and 
        \textbf{(Bottom)} Performance of cumulative loss by increasing the query cost, for a fixed number of rounds $T$ (where $T=5000, 3000, 80, 4000$ from left to right) and maximal query cost $B$ (where $B=T=5000, 3000, 80, 4000$ from left to right). 90\% confident interval are indicated in shades. 
        }
        }
  \label{fig:app:sample-efficiency}
\end{figure*}

\subsection{Relative Cumulative Loss}\label{app:exp:rcl}

{\paragraph{Relative cumulative loss (RCL).} At round $t$, we define RCL as $\Loss_{t,j_i} - \Loss_{t,j^*}$, where $\Loss_{t,j^*}$ stands for the cumulative loss (CL) of the policy always selecting the best classifier, and $\Loss_{t,j_i}$ stands for the CL of policy $i$.

\begin{figure*}[!h]
\begin{subfigure}{1\textwidth}
    \centering
    \includegraphics[height=0.3cm,  clip={0,0,0,0}]{./figures/legend_horizontal.png}
\end{subfigure}
\rotatebox[origin=l]{90}{\quad \quad \scriptsize Relative cumulative loss}
\begin{subfigure}{.24\textwidth}
    \centering
    \includegraphics[height=2.7cm, width=3.4cm,clip={0,0,0,0}]{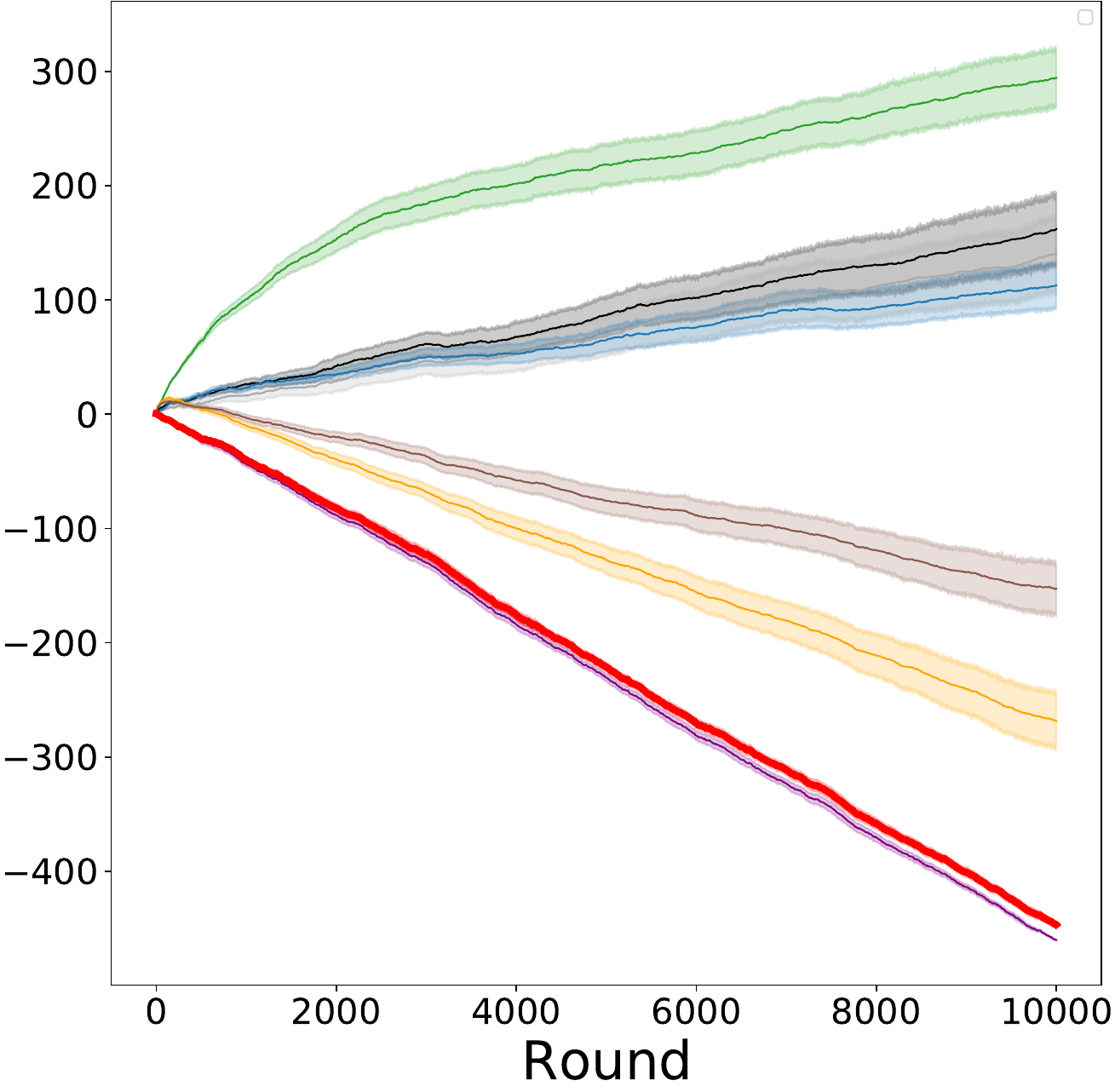}
   \caption{CIFAR10}
\end{subfigure}
\begin{subfigure}{.24\textwidth}
    \centering
    \includegraphics[height=2.7cm,width=3.4cm, clip={0,0,0,0}]{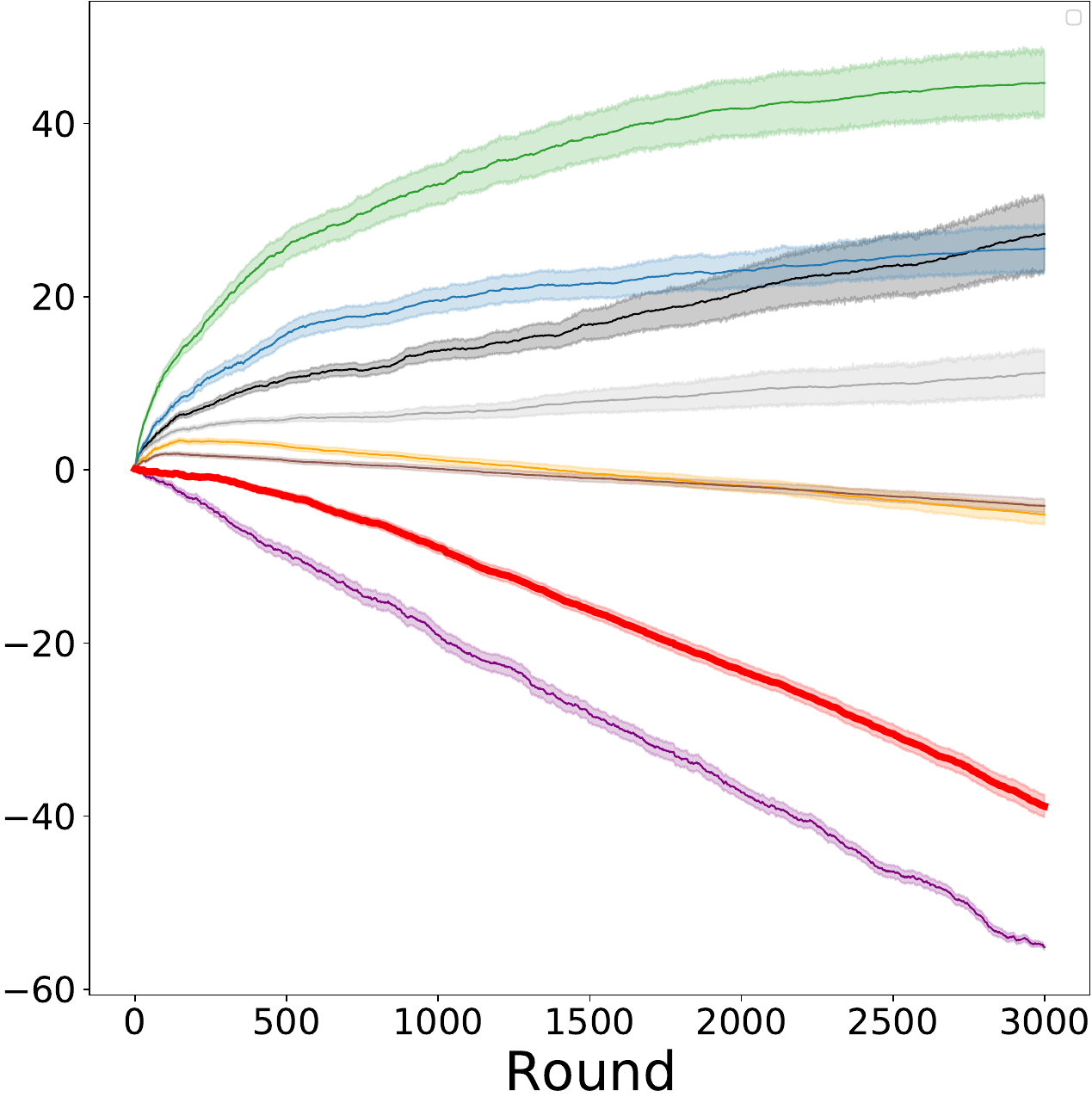}
   \caption{DRIFT}
\end{subfigure}
\begin{subfigure}{.24\textwidth}
    \centering
    \includegraphics[height=2.7cm,width=3.4cm,clip={0,0,0,0}]{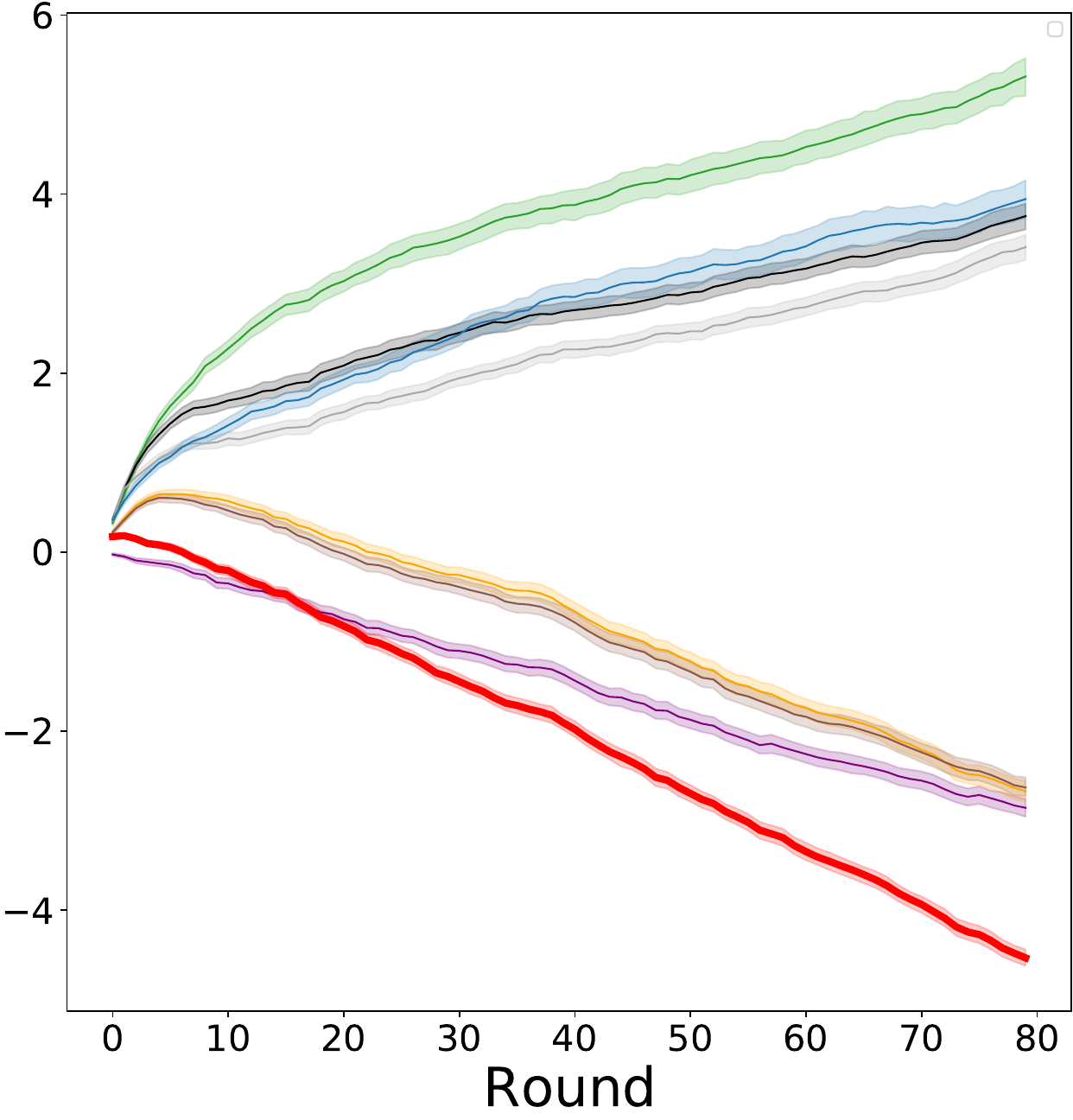}
   \caption{VERTEBRAL}
\end{subfigure}
\begin{subfigure}{.24\textwidth}
    \centering
    \includegraphics[height=2.7cm, width=3.4cm, clip={0,0,0,0}]{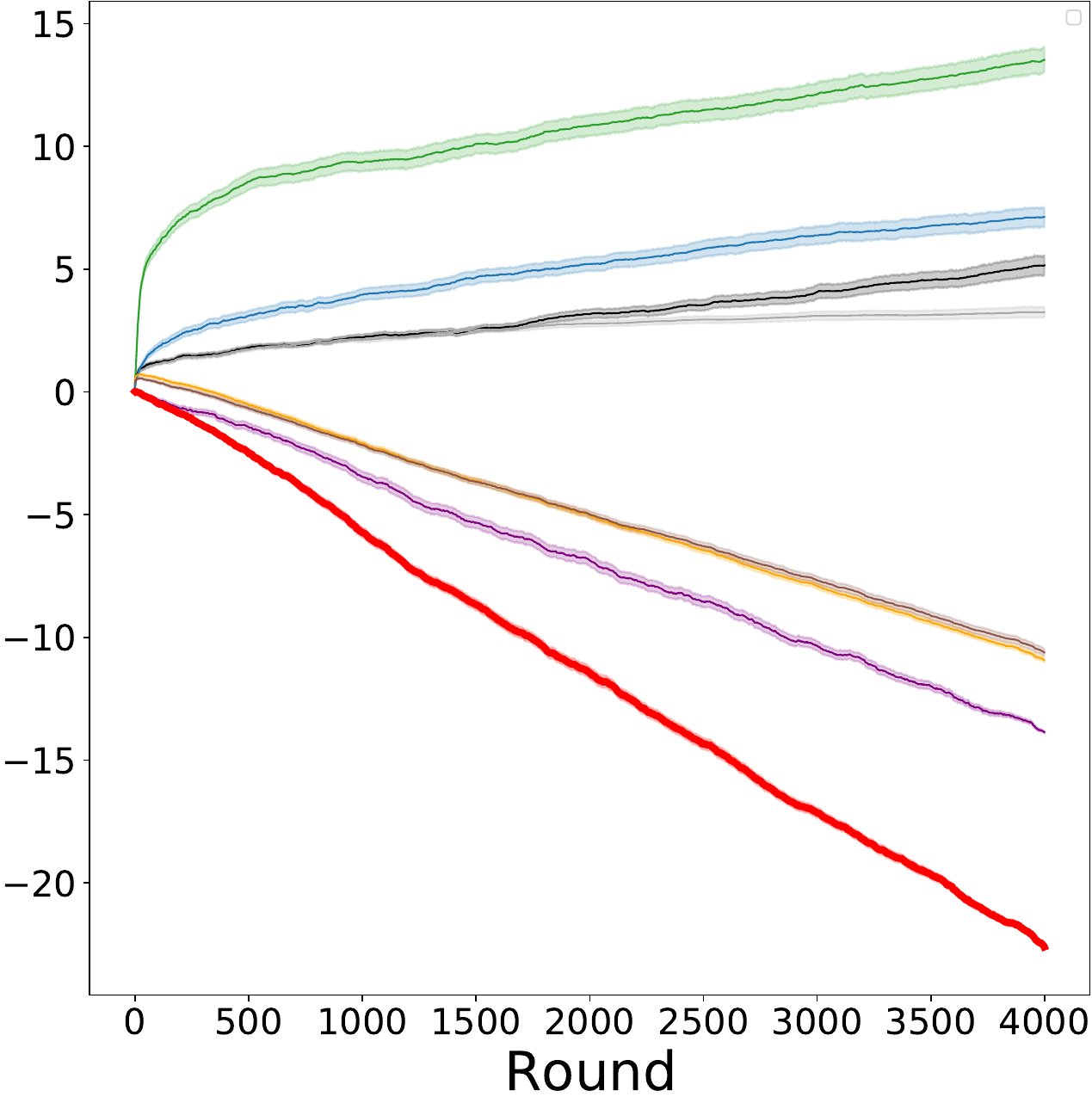}
   \caption{HIV}
\end{subfigure}
    \centering
        \caption{Comparing \algname with 7 baselines on 4 diverse benchmarks in terms of loss trajectory. \algname outperforms all baselines. 
        Performance measured by relative cumulative loss (i.e. loss against the best classifier) under a fixed query cost $B$ (where $B=200,400,30,400$ from left to right). 
        \textbf{Algorithms}: 4 contextual \{Oracle, CQBC, CIWAL, CAMS\} and 4 non-contextual baselines \{RS, QBC, IWAL, MP\} are included (see Section \pref{main:baseines}). 90\% confident interval are indicated in shades. %
        }
  \label{fig:exp:rcl}
\end{figure*}
The RCL under the same query cost for all baselines is shown in \figref{fig:exp:rcl}. 
The loss trajectory demonstrates that \algname efficiently adapts to the best policy after only a few rounds and outperforms all baselines in all benchmarks. The result also demonstrates that \algname can achieve negative RCL on all benchmarks, which means it outperforms any algorithms that chase the best classifier, as the horizontal 0 line represents the performance benchmark of best classifier. This empirical result aligns with \thmref{thm:stochasticReg} that, in the worst scenario, if the best classifier is the best policy, \algname will achieve its performance. Otherwise, \algname will reach a better policy and incurs no regret. 

{\algname could achieve such performance because when an Oracle fails to achieve 0 loss over all instances and contexts, \algname has the opportunity to outperform the Oracle \emph{in those rounds Oracle does not make the best recommendation}. For instance, the stochastic version of \algname (Line 22-23; Line 30-32 in \figref{alg:CAMS}) may achieve this by recommending a model using the weighted majority vote among all policies. Therefore, one can view \algname as adaptively constructing a new policy at each round by combining the advantages of each sub-optimal policy, which may outperform any single expert/policy. Furthermore, for the experiments we ran (or in most real-world scenarios), the data streams are not strictly in a stochastic setting (in which a single policy outperforms all others or has a lower expected loss in every round). The weighted policy strategy may find a better combination of "advices" in such cases (see \figref{fig:exp:results} and \appref{app:outperform_best_expert}).}
}

}{}

\section{Experiments on ImageNet}
\begin{figure*}[ht!]
\begin{subfigure}{1\textwidth}
    \centering
    \includegraphics[%
    width=6.5cm,  clip={0,0,0,0}]{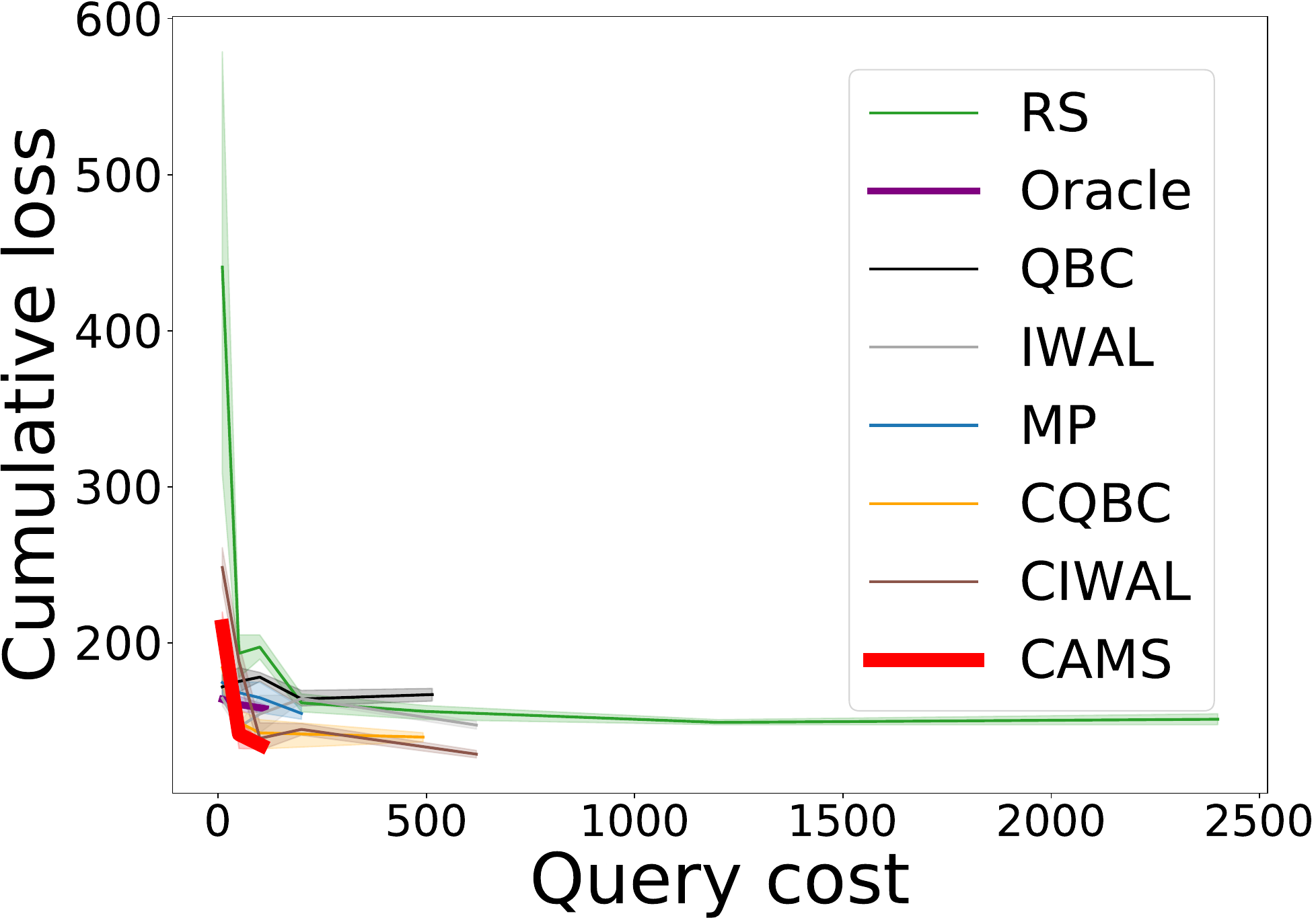}
    \caption{IMAGENET}\label{fig:results:cifar10}
\end{subfigure}
        \caption{Comparison of \algname with 7 baselines on IMAGENET benchmark in terms of cost effectiveness. We plot the cumulative loss as we increase the query cost for a fixed number of rounds $T$ and maximal query cost $B$ ($T=3000$, and $B=2500$). \algname outperforms all baselines. \textbf{Algorithms}: 4 contextual \{Oracle, CQBC, CIWAL, CAMS\} and 4 non-contextual baselines \{RS, QBC, IWAL, MP\} are included. 90\% confident interval are indicated in shades. %
        }
\end{figure*}

\section{Comparing CAMS against recent works in active learning}

\begin{table}[h]
\centering
\scalebox{0.65}{
    \begin{tabular}{l l l l l l l l}
\toprule
\textbf{\makecell[c]{Active Learning Setting \\$/$ Algorithms}}
& \textbf{Coreset}
& \textbf{Batch-BALD}
& \textbf{ \makecell[c]{BADGE; \\VAAL; \\ ClusterMargin}}
& \textbf{ \makecell[c]{BALANCE;\\ GLISTER}}
& \textbf{VeSSAL}
& \textbf{Model Picker}
& \textbf{CAMS}
\\
\midrule
\emph{Streaming, sequential}
& $\times$
& $\times$
& $\times$
& $\times$
& $\times$
& $\checkmark$
& $\checkmark$\\
\hline
\emph{Streaming, batch}
& $\times$
& $\times$
& $\times$
& $\times$
& $\checkmark$
& $\times$
& $\times$ \\
\hline
\emph{Pool-based, batch}
& $\checkmark$
& $\checkmark$
& $\checkmark$
& $\checkmark$
& $\times$
& $\times$
& $\times$\\
\bottomrule \\
\end{tabular}
}
\caption{Selective comparison against recent works in active learning. Among these algorithms, Coreset (Sener \& Savarese, 2017) is a diversity sampling strategy for deep active learning; Batch-BALD is an uncertainty sampling strategy; BADGE (Ash et al., 2019), VAAL (Sinha et al., 2019) ClusterMargin (Citovsky et al., 2021), and VeSSAL (Saran et al., 2023) represent strategies that combine both; GLISTER (Killamsetty et al., 2020) and BALANCE (Zhang et al., 2023) represent decision-theoretic approaches that directly optimize the utility of queries.}
\label{app:emperical-query-complexity}
\end{table}

\end{document}